%% file: main.tex
\documentclass[conference]{IEEEtran}
\usepackage[numbers]{natbib}
\usepackage{times}
\usepackage{graphicx}
\usepackage{amsmath,amssymb,amsfonts,amsthm,amssymb}
\usepackage{bbm}
\usepackage{thmtools}
\usepackage{mathtools}
\usepackage{lipsum}  
\usepackage{xspace}
\usepackage{diagbox}
\let\labelindent\relax
\usepackage{enumitem}
\usepackage{units}
\usepackage{booktabs} 
\usepackage{tabulary}
\usepackage{algorithm}
\usepackage{algpseudocode}
\usepackage{csquotes}

\algtext*{EndWhile}
\algtext*{EndIf}
\algtext*{EndFor}

\algblock[Input]{Input}{EndInput}
\algtext*{EndInput}
\algblock[Output]{Output}{EndOutput}
\algtext*{EndOutput}
\algblock[Variables]{Variables}{EndVariables}
\algtext*{EndVariables}
\usepackage{ifthen,version}
\usepackage{color}
\usepackage{soul}
\usepackage{moreverb,url}
\usepackage{lipsum}
\usepackage[pdftex,colorlinks,bookmarks=true]{hyperref} 

\newtheorem{theorem}{Theorem}
\newtheorem*{theorem*}{Theorem}
\newtheorem{lemma}{Lemma}
\newtheorem*{lemma*}{Lemma}
\newtheorem{problem}{Problem}
\newtheorem{definition}{Definition}
\newtheorem{observation}{O}
\makeatletter
\@addtoreset{observation}{section}
\makeatother

\input{math_operators}
\input{math_defns_ipp}
\input{math_defns_search}

\input{text_shortcuts}

\usepackage[utf8]{inputenc}

\date{}

\title{\LARGE \bf
Data-driven Planning via Imitation Learning
}

\author{\authorblockN{Sanjiban Choudhury\authorrefmark{1},
Mohak Bhardwaj\authorrefmark{1},
Sankalp Arora\authorrefmark{1},\\
Ashish Kapoor\authorrefmark{2},
Gireeja Ranade\authorrefmark{2}, 
Sebastian Scherer\authorrefmark{1} and
Debadeepta Dey\authorrefmark{2}}
\authorblockA{\authorrefmark{1}The Robotics Institute\\
Carnegie Mellon University,
Pittsburgh, PA 15232\\ Email: \{sanjibac,asankalp, mbhardwa, basti\}@andrew.cmu.edu}
\authorblockA{\authorrefmark{2}Microsoft Research, Redmond USA\\
Email: \{akapoor,giranade,dedey\}@microsoft.com}}

\begin{document}

\maketitle
\thispagestyle{empty}
\pagestyle{empty}

\thispagestyle{empty}
\pagestyle{empty}

\makeatletter
\begin{abstract}

Robot planning is the process of selecting a sequence of actions that optimize for a task specific objective. For instance, the objective for a navigation task would be to find collision free paths, while the objective for an exploration task would be to map unknown areas. The optimal solutions to such tasks are heavily influenced by the implicit structure in the environment, i.e. the configuration of objects in the world. State-of-the-art planning approaches, however, do not exploit this structure, thereby expending valuable effort searching the action space instead of focusing on potentially good actions. In this paper, we address the problem of enabling planners to adapt their search strategies by inferring such good actions in an efficient manner using only the information uncovered by the search up until that time.

We formulate this as a problem of sequential decision making under uncertainty where at a given iteration a planning policy must map the state of the search to a planning action. Unfortunately, the training process for such partial information based policies is slow to converge and susceptible to poor local minima. Our key insight is that if we could fully observe the underlying world map, we would easily be able to disambiguate between good and bad actions. We hence present a novel data-driven imitation learning framework to efficiently train planning policies by imitating a \emph{clairvoyant oracle} - an oracle that at train time has full knowledge about the world map and can compute optimal decisions. We leverage the fact that for planning problems, such oracles can be efficiently computed and derive performance guarantees for the learnt policy. We examine two important domains that rely on partial information based policies - informative path planning and search based motion planning. We validate the approach on a spectrum of environments for both problem domains, including experiments on a real UAV, and show that the learnt policy consistently outperforms state-of-the-art algorithms. Our framework is able to train policies that achieve upto $39\%$ more reward than state-of-the art information gathering heuristics and a $70$x speedup as compared to A* on search based planning problems. 
Our approach paves the way forward for applying data-driven techniques to other such problem domains under the umbrella of robot planning.

\end{abstract}
\IEEEpeerreviewmaketitle

\input{introduction}

\section{Background}
\label{sec:background}
\input{background_ipp}
\input{background_search}

\input{background_pomdp}
\input{background_il}

\input{problem_formulation}

\input{imitation_learning}

\input{approach}

\input{results_informative_path_planning}
\input{results_search_based_planning}

\input{discussion}

\section{Acknowledgement}
We would like to thank Silvio Maeta, Vishal Dugar and Brian McAllister for help with flight test results on the UAV.
We would like to thank Shushman Choudhury for insightful discussions and feedback on the paper. 
We  would like to acknowledge the support from ONR grant N000141310821 and NASA contract NNX17CS56C.

\bibliography{all}
\bibliographystyle{plainnat}

\begin{appendices}
\input{proof_lemma1}
\input{proof_theorem1}
\input{proof_theorem2}
\input{proof_theorem3}
\input{ml_baseline_search}
\input{representations_search}
\input{time_complexity_sail}
\end{appendices}

\end{document}

%% file: math_operators.tex
\DeclareMathOperator*{\argmin}{arg\,min}
\DeclareMathOperator*{\argmax}{arg\,max}

\newcommand{\argmaxprob}[1]{\argmax\limits_{#1}}
\newcommand{\argminprob}[1]{\argmin\limits_{#1}}
\newcommand{\norm}[2]{\left|\left| #1 \right|\right|_{#2}}
\newcommand{\abs}[1]{\left|#1 \right|}
\newcommand{\card}[1]{\left|#1\right|}

\DeclareMathOperator*{\suchthat}{\;\; \mbox{s.t.} \;\;}
\newcommand{\expect}[2]{\mathbb{E}_{#1}\left[#2\right]}
\newcommand{\expectS}[1]{\mathbb{E}_{#1}}

\newcommand{\real}[0]{\mathbb{R}}

\newcommand{\bbm}{\begin{bmatrix}}
\newcommand{\ebm}{\end{bmatrix}}

\newcommand{\pair}[2]{\left( #1, #2\right)}
\newcommand{\set}[1]{\left\lbrace #1\right\rbrace}
\newcommand{\seq}[2]{\left(#1_{1}, #1_{2}, \ldots, #1_{#2}\right)}
\newcommand{\seqset}[2]{\left\lbrace#1_{1}, #1_{2}, \ldots, #1_{#2}\right\rbrace}
\newcommand{\setst}[2]{\left\lbrace #1\;\;\middle|\;\;#2\right\rbrace}

\newcommand{\bigo}[1]{\mathcal{O}\left(#1\right)}

%% file: math_defns_ipp.tex
\newcommand{\graph}[0]{\mathcal{G}}
\newcommand{\edgeSet}[0]{\mathcal{E}}
\newcommand{\vertexSet}[0]{\mathcal{V}}
\newcommand{\vertex}[0]{v}

\newcommand{\vertexStart}[0]{\vertex_s}
\newcommand{\Path}[0]{\xi}
\newcommand{\PathSet}[0]{\Xi}

\newcommand{\world}[0]{\phi}

\newcommand{\worldSet}[0]{\mathcal{M}}
\newcommand{\meas}[0]{y}
\newcommand{\measSet}[0]{\mathcal{Y}}
\newcommand{\measFnDef}[0]{\mathcal{H}}
\newcommand{\measFn}[2]{\measFnDef\left(#1, #2\right)}

\newcommand{\utilityFnDef}[0]{\mathcal{F}}
\newcommand{\utilityFn}[2]{\utilityFnDef\left(#1, #2\right)}
\newcommand{\marginalGain}[2]{\Delta_\utilityFnDef\left(#1, #2\right)}

\newcommand{\costFnDef}[0]{\mathcal{T}}
\newcommand{\costFn}[2]{\costFnDef\left(#1, #2\right)}
\newcommand{\costBudget}[0]{B}

\newcommand{\state}[0]{s}
\newcommand{\stateSet}[0]{\mathcal{S}}
\newcommand{\action}[0]{a}

\newcommand{\actionSet}[0]{\mathcal{A}}
\newcommand{\actionSetFeas}[1]{\actionSet_{\mathrm{feas}}\left(#1\right)}
\newcommand{\transFnDef}[0]{\Omega}
\newcommand{\transFn}[3]{\transFnDef{}\left(#1, #2, #3\right)}
\newcommand{\rewardFnDef}[0]{R}
\newcommand{\rewardFn}[2]{\rewardFnDef{}\left(#1, #2\right)}

\newcommand{\policy}[0]{\pi}
\newcommand{\policySet}[0]{\Pi}

\newcommand{\valueFn}[2]{V^{#1}_{#2}}

\newcommand{\QFn}[2]{Q^{#1}_{#2}}

\newcommand{\valuePol}[1]{J\left(#1\right)}

\newcommand{\obs}[0]{o}
\newcommand{\obsSet}[0]{\mathcal{O}}
\newcommand{\obsFnDef}[0]{Z}
\newcommand{\obsFn}[3]{\obsFnDef{}\left(#1, #2, #3\right)}

\newcommand{\belief}[0]{\psi}
\newcommand{\beliefSpace}[0]{\Psi}

\newcommand{\VBel}[0]{\tilde{V}}
\newcommand{\QBel}[0]{\tilde{Q}}

\newcommand{\valueFnBel}[2]{\tilde{V}^{#1}_{#2}}
\newcommand{\QFnBel}[2]{\tilde{Q}^{#1}_{#2}}

\newcommand{\dataset}[0]{\mathcal{D}}
\newcommand{\policyLEARN}[0]{\hat{\pi}}
\newcommand{\policyLEARNAvg}[0]{\policyLEARN_{\mathrm{avg}}}

\newcommand{\numLearnIter}[0]{N}
\newcommand{\mixfrac}[1]{\beta_{#1}}
\newcommand{\numDatapoints}[0]{m}

\newcommand{\policyOR}[0]{\pi_{\mathrm{OR}}}
\newcommand{\policyMix}[1]{\pi_{\mathrm{mix},#1}}

\newcommand{\policyORBel}[0]{\tilde{\pi}_{\mathrm{OR}}}

\newcommand{\QVal}[0]{Q}
\newcommand{\QOR}[0]{\QVal^\mathrm{OR}}

\newcommand{\lossFnPolicyDef}[0]{\mathcal{L}}
\newcommand{\lossFnPolicy}[2]{\lossFnPolicyDef{}\left( #1, #2\right)}

\newcommand{\errcs}[0]{\varepsilon_{\mathrm{cs}}}
\newcommand{\errhor}[0]{\varepsilon_{\mathrm{or}}}

\newcommand{\errclass}[0]{\varepsilon_{\mathrm{class}}}

\newcommand{\errreg}[0]{\varepsilon_{\mathrm{reg}}}

\newcommand{\policyGR}[0]{\pi_{\mathrm{GR}}}

\newcommand{\feature}[0]{f}
\newcommand{\featureIG}[0]{\feature{}_{\mathrm{IG}}}
\newcommand{\featureMot}[0]{\feature{}_{\mathrm{mot}}}

\newcommand{\ray}[0]{r}
\newcommand{\raySet}[0]{\mathcal{R}}
\newcommand{\occGridDef}[0]{\mathcal{X}}
\newcommand{\occGrid}[1]{\occGridDef{}\left(#1\right)}
\newcommand{\gridCell}[0]{x}
\newcommand{\probOcc}[1]{P_o\left(#1\right)}
\newcommand{\probFree}[1]{\bar{P}_o\left(#1\right)}
\newcommand{\infoGainFnDef}[0]{\mathcal{I}}
\newcommand{\infoGainFn}[2]{\infoGainFnDef_{#1}\left(#2\right)}
\newcommand{\infoGainVertFn}[2]{\mathcal{I}\mathcal{G}_{#1}\left(#2\right)}

%% file: math_defns_search.tex
\newcommand{\Graph}[0]{\graph}

\newcommand{\edge}[0]{e}
\newcommand{\start}[0]{\vertex_s}
\newcommand{\goal}[0]{\vertex_g}
\newcommand{\succFnDef}[0]{\mathtt{Succ}}
\newcommand{\succFn}[1]{\succFnDef(#1)}
\newcommand{\evalFn}[0]{\mathtt{Eval}}

\newcommand{\plannerTuple}[6]{\langle#1, #2, #3, #4, #5, #6\rangle}

\newcommand{\vertexSucc}[0]{\vertexSet_\mathrm{succ}}
\newcommand{\invalidEdge}[0]{\edgeSet_\mathrm{inv}}

\newcommand{\openList}[0]{\mathcal{O}}
\newcommand{\closedList}[0]{\mathcal{C}}
\newcommand{\closedObsList}[0]{\mathcal{I}}

\newcommand{\planTime}[0]{T}

\newcommand{\costToGoOracle}{Q^{\textsc{OR}}}

\newcommand{\featureVec}{f}

\newcommand{\featureSpace}{\mathcal{F}}

\newcommand{\greedyEuc}{h_{\textsc{EUC}}}
\newcommand{\greedyMan}{h_{\textsc{MAN}}}
\newcommand{\distObs}{d_{OBS}}

\newcommand{\gVal}{g_{\vertex}}
\newcommand{\treeDepth}{d_{TREE}}

\newcommand{\mixParam}[0]{\beta}
\newcommand{\dataPointsParam}{k}
\newcommand{\policyOracle}[0]{\policy_\mathrm{OR}}

%% file: text_shortcuts.tex
\newcommand{\aggrevate}[0]{\textsc{AggreVaTe}\xspace}
\newcommand{\Dagger}[0]{\textsc{DAgger}\xspace}
\newcommand{\FT}[0]{\textsc{ForwardTraining}\xspace}
\newcommand{\RearSideVoxel}[0]{Rear Side Voxel\xspace}
\newcommand{\AverageEntropy}[0]{Average Entropy\xspace}
\newcommand{\OcclusionAware}[0]{Occlusion Aware\xspace}

\newcommand{\knownunc}[0]{\textsc{Known-Unc}\xspace}
\newcommand{\knowncon}[0]{\textsc{Known-Con}\xspace}
\newcommand{\hiddenunc}[0]{\textsc{Hidden-Unc}\xspace}
\newcommand{\hiddencon}[0]{\textsc{Hidden-Con}\xspace}
\newcommand{\algRewFT}[0]{\textsc{RewardFT}\xspace}
\newcommand{\algQvalFT}[0]{\textsc{QvalFT}\xspace}
\newcommand{\algRewAgg}[0]{\textsc{RewardAgg}\xspace}
\newcommand{\algQvalAgg}[0]{\textsc{QvalAgg}\xspace}

\newcommand{\search}[0]{\mathtt{Search}}
\newcommand{\select}[0]{\mathtt{Select}}
\newcommand{\expand}[0]{\mathtt{Expand}}

\newcommand{\algName}[0]{\textsc{NONAME}}

\newcommand{\optHeur}[0]{\textsc{Opt-Heur}\xspace}
\newcommand{\algSail}[0]{\textsc{SaIL}\xspace}

%% file: introduction.tex

\section{Introduction}

Motion planning, the task of computing a sequence of collision-free motions for a robotic system from a start to a goal configuration, has a rich and varied history~\citep{Lav06}. Up until now, the bulk of the prominent research has focused on the development of tractable planning algorithms with provable \emph{worst-case performance guarantees} such as computational complexity~\citep{canny1988complexity}, probabilistic completeness~\citep{lavalle2001randomized} or asymptotic optimality~\citep{karaman2011sampling}. 
In contrast, analysis of the \emph{expected performance} of these algorithms on real world planning problems a robot encounters has received considerably less attention, primarily due to the lack of standardized datasets or robotic platforms.

Informative path planning, the task of computing an optimal sequence of sensing locations to visit so as to maximize information gain,  has also had an extensive amount of prior work on algorithms with provable worst-case performance guarantees such as computational complexities ~\citep{singh2007efficient} and the probabilistic completeness~\citep{hollinger2013sampling} of information theoretic planning. While these algorithms use heuristics to approximate information gain using variants of Shannon’s entropy, their expected performance on real world planning problems is heavily influenced by the geometric distribution of objects encountered in the world.

\begin{figure*}[!t]
    \centering
    \includegraphics[width=\textwidth]{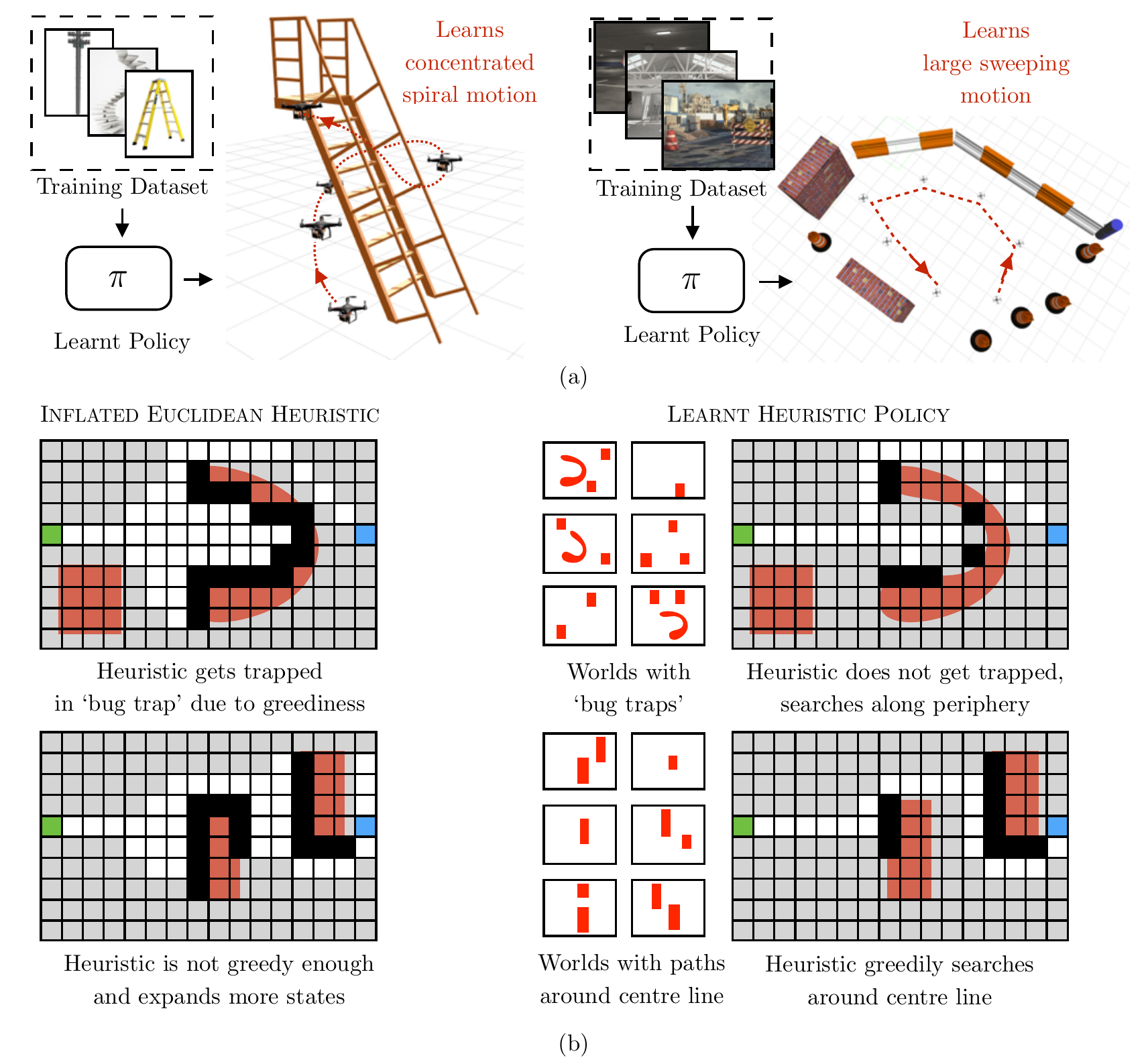}
    \caption{%
    Sequential decision making in informative path planning and search based planning. The implicit structure of the environment affects the performance of policies in both tasks. 
    (a) The effectiveness of a policy to gather information depends on the distribution of worlds.
    (left) When the distribution corresponds to a scene containing ladders, the learnt policy executes a helical motion around parts of the ladder already observed as it is unlikely that there is information elsewhere.
    (right) When the distribution corresponds to a scene from a construction site, the learnt policy executes a large sweeping motion as information is likely to be dispersed. 
    (b) A learnt heuristic policy adapts to different obstacle configurations to minimize search effort. All schematics show the evolution of a search algorithm as the expansion of a search wavefront (expanded(white), invalid(black), unexpanded(grey)) from start (green) to goal (blue). 
    A commonly used inflated Euclidean heuristic cannot adapt to different environments, e.g it gets stuck in bugtraps. 
    On the other hand, the learnt policy is able to infer the presence of a bug trap when trained on such a distribution and switch to greedy behaviour when trained on other distributions.
        \label{fig:marquee}}
\end{figure*}%

A unifying theme for both these problem domains is that as robots break out of contrived laboratory settings and operate in the real world, the scenarios encountered by them vary widely and have a significant impact on performance. 
Hence, a key requirement for autonomous systems is a \emph{robust planning module} that maintains \emph{consistent performance} across the diverse range of scenarios it is likely to encounter.
To do so, planning modules must possess the ability to leverage information about the implicit structure of the world in which the robot operates and adapt the planning strategy accordingly.
Moreover, this must occur in a pure \emph{data-driven fashion} without the need for human intervention.
Fortunately, recent advances in affordable sensors and actuators have enabled mass deployment of robots that navigate, interact and collect real data. This motivates us to examine the following question: 

\begin{displayquote}
How can we design planning algorithms that, subject to on-board computation and sensing constraints, maximize their expected performance on the actual distribution of problems that a robot encounters?
\end{displayquote}

\subsection{Motivation}

We look at two domains - informative path planning and search based planning.
We briefly delve into these motivations and make the case for data-driven approaches in both.
\subsubsection {Informative Path Planning}
We consider the following information gathering problem - given a hidden world map, sampled from a prior distribution, the goal is to successively visit sensing locations such that the amount of relevant information uncovered is maximized while not exceeding a specified fuel budget. This problem fundamentally recurs in mobile robot applications such as autonomous mapping of environments using ground and aerial robots~\citep{Charrow-RSS-15,heng2015efficient}, monitoring of water bodies~\citep{hollinger2013sampling} 
and inspecting models for 3D reconstruction~\citep{isler2016information,hollinger2011active}.

The nature of ``interesting'' objects in an environment and their spatial distribution influence the optimal trajectory a robot might take to explore the environment. As a result, it is important that a robot learns about the type of environment it is exploring as it acquires more information and adapts its exploration trajectories accordingly. 

To illustrate our point, we sketch out two extreme examples of environments for a particular mapping problem, shown in  Fig.~\ref{fig:marquee}(a). Consider a robot equipped with a sensor (RGBD camera) that needs to generate a map of an unknown environment. It is given a prior distribution about the geometry of the world, but has no other information. This geometry could include very diverse settings. First it can include a world where there is only one ladder, but the form of the ladder must be explored, which is a very dense setting. Second, it could include a sparse setting with spatially distributed objects, such as a construction site.

The important task for the robot is to now try to infer which type of environment it is in based on the history of measurements, and thus plan an efficient trajectory. At every time step, the robot visits a sensing location and receives a sensor measurement (e.g. depth image) that has some amount of information utility (e.g. surface coverage of objects with point cloud). As opposed to naive lawnmower-coverage patterns, it will be more efficient if the robot could use a policy that maps the history of locations visited and measurements received to decide which location to visit next such that it maximizes the amount of information gathered in the finite amount of battery time available.

The ability of such a learnt policy to gather information efficiently depends on the prior distribution of worlds in which the robot has been shown how to navigate optimally. Fig.~\ref{fig:marquee}(a) (left) shows an efficient learnt policy for inspecting a ladder, which executes a helical motion around parts of the ladder already observed to efficiently uncover new parts without searching naively. This is efficient because given the prior distribution the robot learns that information is likely to be geometrically concentrated in a particular volume given its initial observations of parts of the ladder. Similarly Fig.~\ref{fig:marquee}(a) (right) shows an effective policy for exploring construction sites by executing large sweeping motions. Here again the robot learns from prior experience that wide, sweeping motions are efficient since it has learnt that information is likely to be dispersed in such scenarios. We wish to arrive at an efficient procedure for training such a policy.

\subsubsection {Search Based Planning}

Search based motion planning offers a comprehensive framework for reasoning about a vast number of motion planning algorithms~\citep{Lav06}. In this framework, an algorithm grows a \emph{search tree} of feasible robot motions from a start configuration towards a goal~\citep{pearl1984heuristics}. 
This is done in an incremental fashion by first selecting a leaf node of the tree, \emph{expanding} this node by computing outgoing edges, checking each edge for validity and finally updating the tree with potentially new leaf nodes.
It is useful to visualize this search process as a \emph{wavefront of expanded nodes} that grows from the start outwards till it finds the goal as illustrated in Fig.~\ref{fig:marquee}(b). 

This paper addresses a class of robotic motion planning problems where edge evaluation dominates the search effort, such as for robots with complex geometries like robot arms~\citep{dellin2016guided} or for robots with limited onboard computation like UAVs~\citep{cover2013sparse}.
In order to ensure real-time performance, algorithms must prioritize minimizing the search effort, i.e. keeping the volume of the search wavefront as small as possible while it grows towards the goal. 
This is typically achieved by heuristics, which guide the search towards promising areas by selecting which nodes to expand.
As shown in Fig.~\ref{fig:marquee}, this acts as a force stretching the search wavefront towards the goal.

A good heuristic must balance the bi-objective criteria of finding a good solution and minimizing the search effort. 
The bulk of the prior work has focused on the former objective of guaranteeing that the search returns a near-optimal solution~\citep{pearl1984heuristics}. 
These approaches define a heuristic function as a \emph{distance metric} that estimates the cost-to-go value of a node~\citep{pohl1970first}.
However, estimation of this distance metric is difficult as it is a complex function of robot geometry, dynamics and obstacle configuration. 
Commonly used heuristics such as the euclidean distance do not adapt to different robot configurations or different environments.
On the other hand, by trying to compute a more accurate distance the heuristic should not end up doing more computation than the original search. 
While state-of-the-art methods propose different relaxation-based~\citep{likhachev2009planning, dolgov2008practical} and learning-based approaches~\citep{paden2017verification} to estimate the distance metric they run into a much more fundamental limitation - \emph{a small estimation error can lead to a large search wavefront}. Minimizing the estimation error does not necessarily minimize search effort.  

Instead, we focus on the latter objective of designing heuristics that explicitly reduce search effort in the interest of real-time performance. 
Our key insight is that \emph{heuristics should adapt during search} - as the search progresses, they should actively infer the structure of the valid configuration space, and focus the search on potentially good areas. Moreover, we want to learn this behaviour from data - changing the data distribution should change the heuristic automatically. 
Consider the example shown in Fig.~\ref{fig:marquee}(b). When a heuristic is trained on a world with `bug traps', it learns to recognize when the search is trapped and circumvent it. On the other hand, when it is trained on a world with narrow gaps, it learns a greedy behaviour that drives the search to the goal.

\subsection{Key Idea}

It is natural to think of both these problems as a Partially Observable Markov Decision Process (POMDP). However the POMDP is defined on a belief over possible world maps, which is very large in size rendering even the most efficient of online POMDP solvers impractical. 

Our key insight is that if the policies could fully observe and process the world map during decision making, they could quite easily disambiguate good actions from bad ones. This motivates us to frame the problem of learning a planning policy as a novel data-driven imitation~\citep{ross2014reinforcement} of a \emph{clairvoyant oracle}. During the training process, the oracle has full knowledge about the world map (hence clairvoyant) and selects actions that maximize cumulative rewards. The policy is then trained to imitate these actions as best as it can using partial knowledge from the current history of actions and observations. As a result of our novel formulation, we are able to sidestep a number of challenging issues in POMDPs like explicitly computing posterior distribution over worlds and planning in belief space. 

We empirically show that training such policies using imitation learning of clairvoyant oracles leads to much faster convergence and robustness to poor local minima than training policies via model free policy improvement. We leverage the fact that such oracles can be efficiently computed for our domains once the source of uncertainty is removed. We show in our analysis that imitation of such clairvoyant oracles during training is equivalent to being competitive with a \emph{hallucinating oracle} at test time, i.e. an oracle that implicitly maintains a posterior over world maps and selects the best action at every time step. This offers some valuable insight behind the success of this approach as well as instances where such an approach would lead to a near-optimal policy.

\subsection{Contributions}

Our contributions are as follows:
\begin{enumerate}
\item We motivate the need to learn a planning policy that adapts to the environment in which the robot operates. We examine two domains - informative path planning and search based planning. We examine both problems through the lens of sequential decision making under uncertainty (Section~\ref{sec:background}).
\item We present a novel mapping of both these problems to a common POMDP framework (Section~\ref{sec:problem_formulation}).
\item We propose a novel framework for training such POMDP policies via imitation learning of a clairvoyant oracle. We analyze the implications of imitating such an oracle (Section~\ref{sec:imitation_learning}).
\item We present training procedures that deal with the non i.i.d distribution of states induced by the policy itself along with performance guarantees. We present concrete instances of the algorithm for both problem domains. We also show that for a certain class of informative path planning problems, policies trained in this fashion possess near-optimality properties (Section~\ref{sec:approach}).
\item We extensively evaluate the approach on both problem domains. In each domain, we evaluate on a spectrum of environments and show that policies outperform state-of-the-art approaches by exhibiting adaptive behaviours. We also demonstrate the impact of this framework on real world problems by presenting flight test results from a UAV (Section~\ref{sec:res_ipp} and Section~\ref{sec:res_search}).
\end{enumerate}

This paper is an unification of previous works on adaptive information gathering~\citep{choudhury2017adaptive,choudhury2016learning} and learning heuristic search~\citep{bhardwaj2017heuristic}. We present a unified framework for reasoning about both problems. We compare and contrast training procedures due to both domains. We present new results in learning heuristics on 4D planning problems and present flight test results from a UAV. We present new results on comparing the imitation learning with policy search and comparing sample efficiency of \aggrevate and \FT. We present more details on implementation and analysis of results. We provide comprehensive discussions on shortcomings of this approach and directions for future work in Section~\ref{sec:discussion}.

%% file: background_ipp.tex

\subsection{Informative Path Planning}
\label{sec:background:ipp}

We now present a framework for informative path planning where the objective is to visit maximally informative sensing locations subjected to time and travel constraints. We use this framework to pose the problem of computing a information gathering policy for a given distribution over worlds and briefly discuss prior work on this topic.

\subsubsection{Framework} 
\label{sec:background:ipp:framework}

We now introduce a framework and set of notations to express the IPP problems of interest. The specific implementation details of the problem are described in detail in Section~\ref{sec:res_ipp:problem}.

We have a robot that is constrained to move on a graph $\graph = (\vertexSet, \edgeSet)$ where $\vertexSet$ is the set of nodes corresponding to all sensing locations. The start node is $\vertexStart$.
Let $\Path = \seq{\vertex}{p}$ be a sequence of connected nodes (a path) such that $\vertex_1 = \vertexStart$. Let $\PathSet$ be the set of all such paths.

Let $\world \in \worldSet$ be the world map in which the robot operates. The world map is usually represented in practice as a binary grid map where grid cells are either occupied or free. We assume that the world map is fixed during an episode. 

Let $\meas \in \measSet$ be a measurement received by the robot. Let $\measFnDef{}: \vertexSet \times \worldSet \to \measSet$ be a measurement function. When the robot is at node $\vertex$ in a world map $\world$, the measurement $\meas$ received by the robot is $\meas = \measFn{\vertex}{\world}$. The measurement function is defined by a sensor model, e.g. a range limited sensor. A measurement is obtained by projecting the sensor model on the sensing node $\vertex$ and ray-casting to determine the surfaces of the underlying world $\world$ that intersect with the sensor rays.

The objective of the robot is to move on the graph and maximize utility. Let $\utilityFnDef{}: 2^\vertexSet \times \worldSet \to \real_{\geq 0}$ be a utility function. For a path $\Path$ and a world map $\world$, $\utilityFn{\Path}{\world}$ assigns a utility to executing the path on the world. The utility of a measurement from a node is usually the amount of surface of the world covered by it. In such an instance, the function does not depend on the sequence of vertices in the path, i.e. is a set function. 
For simplicity, we assume that the measurement and utility function is deterministic. However, this assumption can easily be relaxed in our approach and is discussed in Section.~\ref{sec:discussion:noisy}. 

As the robot moves on the graph, the travel cost is captured by the cost function $\costFnDef{}: \PathSet \times \worldSet \to \real_{\geq 0}$. For a path $\Path$ and a world map $\world$, $\costFn{\Path}{\world}$ assigns a travel cost for executing the path on the world. In a practical setting, the total number of timesteps is bounded by $T$ and the travel cost is bounded by $\costBudget$. 
Fig.~\ref{fig:problem} shows an illustration of the framework.

\begin{figure}[t!]
    \centering
    \includegraphics[width=\columnwidth]{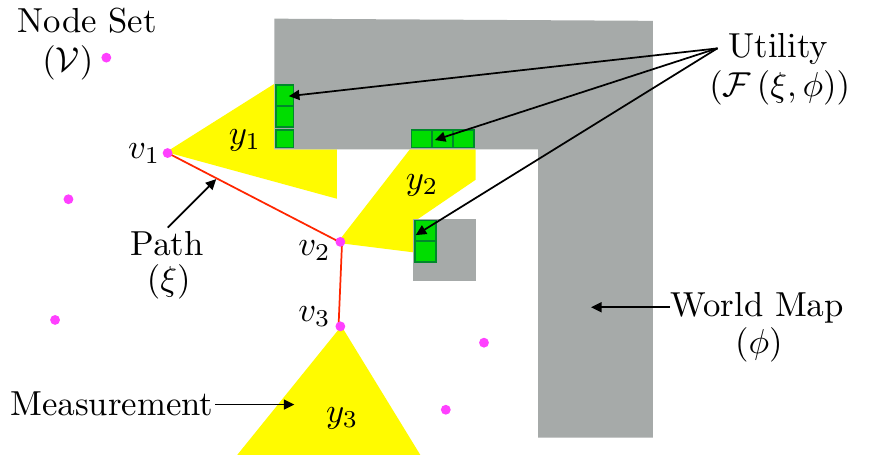}
    \caption{%
    The informative path planning problem. Given a world map $\world$, the robot plans a path $\Path$ which visits a node $\vertex_i \in \vertexSet$ and receives measurement $y_i$, such that utility (information gathered) $\utilityFn{\Path}{\world}$ is maximized. Here the utility is the cardinality of all the cells uncovered (green), which is a union of the cells uncovered at each location (and hence a set cover function)
    \label{fig:problem}
}
\end{figure}

We are now ready to define the informative path planning problems. There are two axes of variations
\begin{enumerate}
\item Constraint on the motion of the robot
\item Observability of the world map
\end{enumerate}

The first axis arises from whether the robot is subject to any travel constraints. For problems such as sensor placement, the agent is free to select any sequence of nodes and the travel cost between nodes is $0$. For such situations, the graph is also fully connected to permit any sequence. For problems involving physical movements, the agent is constrained by a budget on the travel cost. Additionally the graph may also not be fully connected. 

The second axis arises from different task specifications which result in the world map being observable or being hidden. We categorize the problems on this axis to aid future discussions on imitating clairvoyant oracles in Section~\ref{sec:approach}.

\subsubsection{Problems with Known World Maps}
\label{sec:background:ipp:problem_known}
For the first two variants, the world map $\world$ is known and can be evaluated while computing a path $\Path$.

\begin{problem}[\knownunc: Known World Map; Unconstrained Travel Cost] \label{prob:known:unc}
Given a world map $\world$, a fully connected graph $\graph$ and a time horizon $T$, find a path $\Path$ that maximizes utility
\begin{equation}
\begin{aligned}
\argmaxprob{\Path \in \PathSet} \quad & \utilityFn{\Path}{\world} \\
\suchthat                             & \card{\Path} \leq T+1
\end{aligned}
\end{equation}
\end{problem}

In the case where the utility function is a set function, Problem~\ref{prob:known:unc} is a set function maximization problem which in general can be NP-Hard~\citep{krause2012submodular}). Such problems occur commonly in the sensor placement problem \cite{krause2008efficient}. However, in many instances the utility function can be shown to posses the powerful property of \emph{monotone submodularity}. This property implies the following
\begin{enumerate}
\item \emph{Monotonic improvement}: The value of the utility can only increase on adding nodes, i.e. 
\begin{equation*}
	\utilityFn{ \vertexSet_1 \cup \vertexSet_2 }{\world} \geq \utilityFn{ \vertexSet_1 }{\world} 
\end{equation*}
for all $\vertexSet_1, \vertexSet_2 \subseteq \vertexSet$
\item \emph{Diminishing returns}: The gain in adding a set of nodes diminshes
\begin{equation*}
\begin{aligned}
	\utilityFn{ \vertexSet_1 \cup \vertexSet_3 }{\world} - \utilityFn{\vertexSet_3 }{\world} \leq & \utilityFn{ \vertexSet_1 \cup \vertexSet_2 }{\world} \\ 
	& - \utilityFn{\vertexSet_2 }{\world}
\end{aligned}
\end{equation*}
for all $\vertexSet_1, \vertexSet_2, \vertexSet_3 \subseteq \vertexSet$ where $\vertexSet_2 \subseteq \vertexSet_3$.
\end{enumerate}
 For such functions, it has been shown that a greedy algorithm achieves near-optimality~\citep{krause2008efficient,Krause:2007:NOS:1619797.1619913}.

\begin{problem}[\knowncon: Known World Map; Constrained Travel Cost] \label{prob:known:cons}
Given a world map $\world$, a time horizon $T$ and a travel cost budget $B$, find a path $\Path$ that maximizes utility
\begin{equation}
\begin{aligned}
\argmaxprob{\Path \in \PathSet} \quad & \utilityFn{\Path}{\world} \\
\suchthat                             & \costFn{\Path}{\world} \leq \costBudget \\
                                      &  \card{\Path} \leq T+1
\end{aligned}
\end{equation}
\end{problem}

Problem~\ref{prob:known:cons} introduces a routing constraint (due to $\costFnDef$) for which greedy approaches can perform arbitrarily poorly. Such problems occur when a physical system has to travel between nodes. 
\citet{chekuri2005recursive,singh2007efficient} propose a quasi-polynomial time recursive greedy approach to solving this problem. \citet{iyer2013submodular} solve a related problem (submodular knapsack constraints) using an iterative greedy approach which is generalized by \citet{zhang2016submodular}. \citet{yu2014correlated} propose a mixed integer approach to solve a related correlated orienteering problem. \citet{hollinger2013sampling} propose a sampling based approach. \citet{arora2017rapidly} use an efficient TSP with a random sampling approach.

\subsubsection{Problems with Hidden World Maps}
\label{sec:background:ipp:problem_hidden}

We now consider the setting where the world map $\world$ is hidden. Given a prior distribution $P(\world)$, it can be inferred only via the measurements $\meas_i$ received as the robot visits nodes $\vertex_i$. Hence, instead of solving for a fixed path, we compute a policy that maps history of measurements received and nodes visited to decide which node to visit. 

\begin{problem}[\hiddenunc: Hidden World Map; Unconstrained Travel Cost] \label{prob:hidden:unc}
Given a distribution of world maps, $P(\world)$, a fully connected graph $\graph$, a time horizon $T$, find a policy that at time $t$, maps the history of nodes visited $\{ \vertex_i \}_{i=1}^{t}$ and measurements received $\{ \meas_i \}_{i=1}^{t}$ to compute the next node $\vertex_{t+1}$ to visit at time $t+1$, such that the expected utility is maximized. 
\end{problem}

Such a problem occurs for sensor placement where sensors can optionally fail~\cite{golovin2011adaptive}.
Due to the hidden world map $\world$, it is not straight forward to apply the approaches of Problem \knownunc - we have to reason both about $P(\world \; | \; \{ \vertex_i \}_{i=1}^{t} , \{ \meas_i \}_{i=1}^{t})$ and how the function will evolve. 
 However, in some instances the utility function $\utilityFnDef$  has an additional property of \emph{adaptive submodularity}~\citep{golovin2011adaptive}. This is an extension of the submodularity property where the gain of the function is measured in expectation over the conditional distribution over world maps $P(\world \; | \; \{ \vertex_i \}_{i=1}^{t} , \{ \meas_i \}_{i=1}^{t})$. Under such situations, applying greedy strategies to Problem~\ref{prob:hidden:unc} has near-optimality guarantees~\citep{golovin2010near, Javdani_2013_7419,Javdani_2014_7555, AAAI159841,DBLP:journals/corr/ChenHK16a} ). However, these strategies require explicitly sampling from the posterior distribution over $\world$ which make it intractable to apply for our setting.

\begin{problem}[\hiddencon: Hidden World Map; Constrained Travel Cost] \label{prob:hidden:cons}
Given a distribution of world maps, $P(\world)$, a time horizon $T$, and a travel cost budget $B$, find a policy that at time $t$, maps the history of nodes visited $\{ \vertex_i \}_{i=1}^{t}$ and measurements received $\{ \meas_i \}_{i=1}^{t}$ to compute the next node $\vertex_{t+1}$ to visit at time $t+1$, such that the expected utility is maximized. \end{problem}

Such problems crop up in a wide number of areas such as sensor planning for 3D surface reconstruction~\cite{isler2016information} and indoor mapping with UAVs~\cite{Charrow-RSS-15,nelson2015information}.
Problem~\ref{prob:hidden:cons} does not enjoy the adaptive submodularity property due to the introduction of travel constraints. \citet{hollinger2011active,hollinger2012active} propose a heuristic based approach to select a subset of informative nodes and perform minimum cost tours. \citet{Singh:2009:NAI:1661445.1661741} replan every step using a non-adaptive information path planning algorithm. Inspired by adaptive TSP approaches by \citet{gupta2010approximation}, \citet{lim2016adaptive,NIPS2015_6005} propose recursive coverage algorithms to learn policy trees. However such methods cannot scale well to large state and observation spaces. \citet{heng2015efficient} make a modular approximation of the objective function. \citet{isler2016information} survey a broad number of myopic information gain based heuristics that work well in practice but have no formal guarantees.

%% file: background_search.tex

\subsection{Search Based Planning}
\label{sec:background:search}

We now present a framework for search based planning where the objective is to find a feasible path from start to goal while minimizing search effort. We use this framework to pose the problem of learning the optimal heuristic for a given distribution over worlds and briefly discuss prior work on this topic.

\subsubsection{Framework} 
\label{sec:background:search:framework}

We consider the problem of search on a graph, $\Graph = \pair{\vertexSet}{\edgeSet}$, where vertices $\vertexSet$ represent robot configurations and edges $\edgeSet$ represent potentially valid movements of the robot between these configurations. 
Given a pair of start and goal vertices, $\pair{\start}{\goal} \in \vertexSet$, the objective is to compute a path $\Path \subseteq \edgeSet$ - a connected sequence of valid edges.
The implicit graph $\Graph$ can be compactly represented by $\pair{\start}{\goal} $ and a successor function $\succFn{\vertex}$ which returns a list of outgoing edges and child vertices for a vertex $\vertex \in \vertexSet$. Hence a graph $\Graph$ is constructed during search by repeatedly \emph{expanding} vertices using $\succFn{\vertex}$.
Let $\world \in \worldSet$ be a representation of the world that is used to ascertain the validity of an edge. 
An edge $\edge \in \edgeSet$ is checked for validity by invoking an evaluation function $\evalFn(\edge, \world)$ which is an expensive operation and may require complex geometric intersection operations \citep{dellin2016unifying}.

Alg.~\ref{alg:search} defines a general search based planning algorithm $\search$ which takes as input the tuple $\plannerTuple{\start}{\goal}{\succFnDef}{\evalFn}{\world}{\select}$ and returns a valid path $\Path$. To ensure systematic search, the algorithm maintains the following lists - an open list $\openList \subset \vertexSet$ of candidate vertices to be expanded and a closed list $\closedList \subset \vertexSet$ of vertices which have already been expanded. It also retains an additional invalid list $\closedObsList \subset \edgeSet$ of edges found to be in collision. These $3$ lists together represent the complete information available to the algorithm at any given point of time. At a given iteration, the algorithm uses this information to select a vertex $\vertex \in \openList$ to expand by invoking $\select(\openList)$. It then expands $\vertex$ by invoking $\succFn{\vertex}$ and checking validity of edges using $\evalFn(\edge, \world)$ to get a set of valid successor vertices $\vertexSucc$ as well as invalid edges $\invalidEdge$. The lists are then updated and the process repeated till the goal vertex $\goal$ is uncovered. Fig.~\ref{fig:search_problem} illustrates this framework.

\begin{figure}[!t]
    \centering
    \includegraphics[width=\columnwidth]{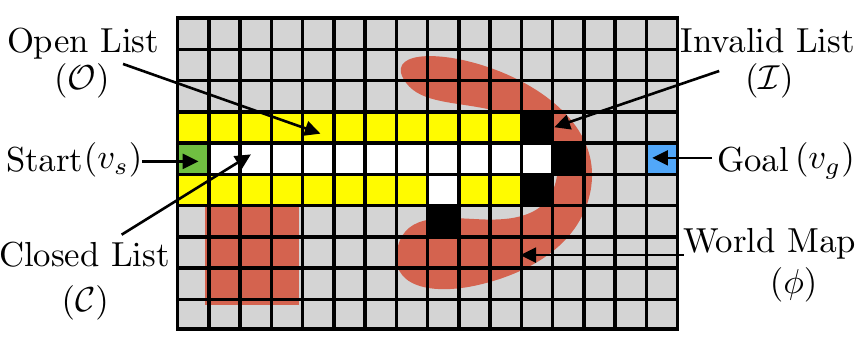}
    \caption{%
    The search based planning problem. Given a world map $\world$, the agent has to guide a search tree from start $\start$ to goal $\goal$ by expanding vertices. At any given iteration, the open list $\openList$ represents the set of candidate vertices that can be expanded. The closed list $\closedList$ represents the set of vertices already expanded. The invalid list represents the set of edges that were found to be in collision with the world. The status of every other vertex is unknown. The search continues till the goal belongs to the open list, i.e. a feasible path to goal has been found. 
    \label{fig:search_problem}
}
\end{figure}

\subsubsection{The Optimal Heuristic Problem} 
In this work, we focus on the \emph{feasible path problem} and ignore the optimality of the path. Although this is a restrictive setting, quickly finding the feasible path is a very important problem in robotics. Efficient feasible path planners such as RRT-Connect~\citep{kuffner2000rrt} has proven highly effective in high dimensional motion planning applications such as robotic arm planning~\citep{Lav06} and mobile robot planning~\citep{laumond1998guidelines}. Hence we ignore the traversal cost of an edge and deal with unweighted graphs. We defer discussions on how to relax this restriction to Section \ref{sec:discussion:anytime_search}.

We view a heuristic policy as a \emph{selection function} (Alg.~\ref{alg:search}, Line~\ref{alg:search:select}) that selects a vertex $\vertex$ from the open list $\openList$. The objective of the policy is to minimize the number of expansions until the search terminates. Note that the evolution of the open list $\openList$ depends on the underlying world map $\world$ which is hidden. Given a prior distribution over world maps $P(\world)$, it can be inferred only via the outcome of the expansion operation $\pair{\vertexSucc}{\invalidEdge}$. The history of outcomes is captured by the state of the search, i.e. the combination of the 3 lists $\{ \openList, \closedList, \closedObsList\}$. 

\begin{problem}[\optHeur] \label{prob:opt_heur}
Given a distribution of world maps, $P(\world)$, find a heuristic policy that at time $t$, maps the state of the search $\{ \openList_t, \closedList_t, \closedObsList_t\}$ to select a vertex $\vertex_t \in \openList_t$ to expand, such that the expected number of expansions till termination is minimized.
\end{problem}

\begin{algorithm}[!t]
    \caption{$\search \plannerTuple{\start}{\goal}{\succFnDef}{\evalFn}{\world}{\select}$ }\label{alg:search}
    \begin{algorithmic}[1]
    \State $\openList \gets \start,\; \closedList \gets \emptyset,\; \closedObsList \gets \emptyset$
    \While{$\goal \notin \openList$}
      \State $\vertex \gets \select(\openList)$ \label{alg:search:select} 
      \State $\pair{\vertexSucc}{\invalidEdge} \gets \expand(\vertex, \succFnDef, \evalFn, \world)$  
      \State $\openList \gets \openList \cup \vertexSucc, \; \closedList \gets \closedList \cup \vertex, \; \closedObsList \gets \closedObsList \cup \invalidEdge $
    \EndWhile
    \State \textbf{Return} $\mathtt{Path}\pair{\start}{\goal}$
    \end{algorithmic}
\end{algorithm}

The problem of heuristic design has a lot of historical significance. A common theme is ``Optimism Under Uncertainty''. A spectrum of techniques exist to manually design good heuristics by relaxing the problem to obtain guarantees with respect to optimality and search effort~\citep{pearl1984heuristics}. To get practical performance, these heuristics are inflated, as has been the case in the applications in mobile robot planning~\citep{likhachev2009planning}. However, being optimistic under uncertainty is not a foolproof approach and could be disastrous in terms of search efforts depending on the environment (See Fig 2.5, \citet{Lav06}).  

Learning heuristics falls under machine learning for general purpose planning~\citep{jimenez2012review}. \citet{yoon2006learning} propose using regression to learn residuals over FF-Heuristic~\citep{hoffmann2001ff}. \citet{xu2007discriminative,xu2010iterative, xu2009learning} improve upon this in a beam-search framework. \citet{arfaee2011learning} iteratively improve heuristics. \citet{us2013learning} learn combination of heuristic to estimate cost-to-go. Kendall rank coefficient is used to learn open list ranking~\citep{wilt2015building,garrettlearning}. \citet{thayer2011learning} learn heuristics online during search. \citet{paden2017verification} learn admissible heuristics as S.O.S problems. However, these methods do not address minimization of search effort and also ignore the non i.i.d nature of the problem.

%% file: background_pomdp.tex

\subsection{Partially Observable Markov Decision Process}

POMDPs~\cite{kaelbling1998planning} provide a rich framework for sequential decision making under uncertainty. However, solving a POMDP is often intractable - finite horizon POMDPs are PSPACE-complete~\citep{papadimitriou1987complexity} and infinite horizon POMDPs are undecidable~\citep{madani2003undecidability}. Despite this challenge, the field has forged on and produced a vast amount of work by investigating effective approximations and analyzing the structure of the optimal solution. We refer the reader to \cite{ross2008online} for a concise survey of modern approaches. 

There are two main approaches to POMDP planning: offline policy computation and online search. In offline planning, the agent computes before hand a policy by considering all possible scenarios and executes the policy based on the observation received. Athough offline methods have shown success in planning near-optimal policies in several domains~\citep{smith2012point,kurniawati2008sarsop,spaan2005perseus}, they are difficult to scale up due to the exponential number of future scenarios that must be considered. 

Online methods interleave planning and execution. The agent plans with the current belief, executes the action and updates the belief. Monte-carlo sampling methods explicitly maintain probability over states and plan via monte carlo roll-outs~\citep{mcallester1999approximate,asmuth2011approaching}. This limits scalability since belief update can take time. In contrast, POMCP~\citep{silver2010monte} maintains a set of particles to represent belief and employ UCT methods to plan with these particles. This allows the method to scale up for larger state spaces. 

However, the disadvantage of purely online methods is that they require a lot of search effort online and can lead to poor performance due to evaluation on a small number of particles. \cite{somani2013despot} present a state-of-the-art algorithm DESPOT that combines the best aspects of many algorithms. First it uses determinized sampling techniques to ensure that the branching factor of the tree is bounded~\cite{ng2000pegasus,kearns2000approximate}. Secondly, it uses offline precomputed policies to roll-out from a vertex, thus lower bounding its value. Finally, it tries to regularize the search by weighing the utility of a node to be robust against the fact that a finite number of samples is being used. 

The methods we have talked about explicitly models the belief. For large scale POMDPs, this might be an issue. Model free approaches and representation learning offer attractive alternatives. Model free policy improvement has been successfully used to solve POMDPs~\cite{liu2013online,li2009multi}. Predictive state representations ~\cite{littman2002predictive,boots2011closing} that minimize prediction loss of future observations offer more compact representations than maintaining belief. There also has been a lot of success in employing deep learning to learn powerful representations~\cite{hausknecht2015deep,karkus2017qmdp}.

%% file: background_il.tex

\subsection{Reinforcement Learning and Imitation Learning}
\label{sec:background:il}
Reinforcement Learning (RL)~\citep{sutton1998reinforcement} especially deep RL has dramatically advanced the capabilities of sequential decision making in high dimensional spaces such as controls~\cite{duan2016benchmarking}, video games~\cite{silver2016mastering} and strategy games~\citep{silver2016mastering}. Several conventional supervised learning tasks are now being solved using deep RL to achieve higher performance~\citep{ranzato2015sequence, li2016deep}. In sequential decision making, the prediction of a learner is dependent on the history of previous outcomes. Deep RL algorithms are able to train such predictors by reasoning about the future accumulated cost in a principle manner. 

We refer the reader to \cite{kober2013reinforcement} for a concise survey on RL and to \cite{arulkumaran2017brief} for a survey on deep RL. Training such policies can be classified into two approaches - either \emph{value function-based approach}, where a value function for an action is learnt, or \emph{policy search}, where a policy is directly learnt. The value function methods can themselves be categorized in two categories - \emph{model-free} algorithms and \emph{model-based} algorithms. 

Model-free methods are computationally cheap but ignore the dynamics of the world thus requiring a lot of samples. 
Q-learning~\citep{watkins1992q} is a representative algorithm for estimating the long-term expected return for executing an action from a given state. When the number of state action pairs are too large in number to track each uniquely, a function approximator is required to estimate the value. Deep Q-learning~\citep{mnih-dqn-2015,wang2016dueling} addresses such a need by employing a neural-network as a function approximator and learning these network weights. However, the process of using the same network to generate both target values and update Q-values results in oscillations. Hence a number of remedies are required to maintain stability such as having a buffer of experience, a separate target network and an adaptive learning rate. These are indicative of the underlying sample inefficiency problem of a model-free approach.

Model-based methods such as R-Max~\citep{brafman2002r} learn a model of the world which is then used to plan for actions. While such methods are sample efficient, they require a lot of exploration to learn the model. Even in the case when the model of the environment is known, solving for the optimal policy might be computationally expensive for large spaces. Policy search approaches are commonly used where its easier to parameterize a policy than learn a value function~\citep{peters2006policy}, however such approaches are sensitive to initialization and can lead to poor local minima.

In contrast with RL methods, imitation learning (IL) algorithms~\citep{daume2009search,venkatraman2014data,chang2015learning,ross2014reinforcement} reduce the sequential prediction problem to supervised learning by leveraging the fact that, for many tasks, at training time we usually have a (near) optimal cost-to-go oracle. This oracle can either come from a human expert guiding the robot~\cite{abbeel2004apprenticeship} or from ground truth data as in natural language processing~\cite{chang2015learning}. 
The existence of such oracles can be exploited to alleviate learning by trial and error - imitation of an oracle can significantly speed up learning. A traditional approach to using such oracles is to learn a policy or value function from a pre-collected dataset of oracle demonstrations~\citep{ratliff2009learning,ziebart2008maximum,finn2016guided}. A problem with these methods is that they require training and test data to be sampled from the same distribtution which is difficult in practice. In contrast, interactive approaches to data collection and training has been shown to overcome stability issues and works well empirically~\citep{ross2011reduction,ross2014reinforcement,sun2017deeply}. Furthermore, these approaches lead to strong performance through a reduction to no-regret online learning.

Recent approaches have also employed imitation of clairvoyant oracles, that has access to more information than the learner during training, to improve reinforcement learning as they offer better sample efficiency and safety. \citet{zhang2016mpcgps,kahn2016plato} train policies that map current observation to action by extending guided policy search~\citep{levine2013guided} for imitation of model predictive control oracles. \citet{tamar2016hindsight} consider a cost-shaping approach for short horizon MPC by offline imitation of long horizon MPC which is closest to our work. \citet{gupta2017cmp} develop a holistic mapping and planner framework trained using feedback from optimal plans on a graph. 

\cite{sun2017deeply} also theoretically analyze the question of why imitation learning aids in reinforcement learning. They develop a comprehensive theoretical study of IL on discrete MDPs and construct scenarios to show that IL acheives better sample efficiency than any RL algorithm. Concretely, they conclude that one can expect atleast a polynomial gap ad a possible exponential gap in regret between IL and RL when one has access to unbiased estimates of the optimal policy during training. 

%% file: problem_formulation.tex

\section{Problem Formulation}
\label{sec:problem_formulation}

\subsection{POMDPs}
A discrete-time finite horizon POMDP is defined by the tuple $(\stateSet, \actionSet, \transFnDef, \rewardFnDef, \obsSet, \obsFnDef, T)$ where
\begin{itemize}
\item $\stateSet$ is a set of states
\item $\actionSet$ is a set of actions
\item $\transFnDef$ is a set of state transition probabilities
\item $\rewardFnDef: \stateSet \times \actionSet$ is the reward function
\item $\obsSet$ is the set of observations
\item $\obsFnDef$ is a set of conditional observation probabilities
\item $T$ is the time horizon
\end{itemize}

At each time period, the environment is in some state $\state \in \stateSet$ which cannot be directly observed. The initial state is sampled from a distribution $P(\state)$. The agent takes an action $\action \in \actionSet$ which causes the environment to transition to state $\state' \in \stateSet$ with probability $\transFn{\state}{\action}{\state'} = P(\state_{t+1} = \state' | \state_t = \state, \action_t = \action)$. The agent receives a reward $\rewardFn{\state}{\action}$. On reaching the new state $\state'$, it receives an observation $\obs \in \obsSet$ according to the probability $\obsFn{\state'}{\action}{\obs} = P(\obs_{t+1} = \obs | \state_{t+1} = \state', \action_t = \action)$. 

A \emph{history} $\belief \in \beliefSpace$ is a sequence of actions and observations $\belief_t = \{ <\obs_1>, <\action_1, \obs_2>, \dots, <\action_{t-1}, \obs_t> \}$. Note that the initial history $\belief_t = <\obs_1>$ is simply the observation at the initial timestep. The history $\psi_t$ captures all information required to express the belief over state. The belief $P(\state_{t+1} | \belief_{t+1})$ can be computed recursively applying Bayes' rule
\begin{equation*}
\eta \;\obsFn{\state_{t+1}}{\action_t}{\obs_{t+1}} \sum\limits_{\state_t \in \stateSet} \transFn{\state_t}{\action_t}{\state_{t+1}} P(\state_t | \belief_t)
\end{equation*}
where $\eta$ is a normalization constant.

The history can then also be used to compute an update $P(\belief_{t+1} | \belief_t, \action_t)$:
\begin{equation*}
 \sum\limits_{\state_t \in \stateSet} \sum\limits_{\state_{t+1} \in \stateSet}  P(\state_t | \belief_t) \transFn{\state_t}{\action_t}{\state_{t+1}} \obsFn{\state_{t+1}}{\action_t}{\obs_{t+1}}
\end{equation*}

The agent's action selection behaviour can be explained by a policy $\policy(\belief_t) \in \policySet$ that maps history $\belief_t$ to action $\action_t$.

Let the state and history distribution induced by a policy $\policy$ after $t$ timesteps be $P(\state, \belief | \policy, t)$. 
The value of a policy $\policy$ is the expected cumulative reward for executing $\policy$ for $T$ timesteps on the induced state and history distribution
\begin{equation}
	\valuePol{\policy} = \sum\limits_{t=1}^{T} \expect{\state_t, \belief_t \sim P(\state, \belief | \policy, t)}{\rewardFn{\state_t}{\policy(\belief_t)}}
\end{equation}

The optimal policy maximizes the expected cumulative reward, i.e $\policy^* \in \argmaxprob{\policy \in \policySet} \valuePol{\policy}$.

Given a starting history $\belief$, let  $P(\state', \belief' | \belief, \policy, i)$ be the induced state history distribution after $i$ timesteps. 
The value of executing a policy $\policy$ for $t$ time steps from a history $\belief$ is the expected cumulative reward:
\begin{equation}
\valueFnBel{\policy}{t}(\belief) = \sum\limits_{i=1}^{t} 
\expect{\state_i, \belief_i \sim P(\state', \belief' | \belief, \policy, i)}
{\rewardFn{\state_i}{\policy(\belief_i)}}
\end{equation}

The state-action value function $\QFnBel{\policy}{t}(\belief_t, \action_t)$ is defined as the expected sum of one-step-reward and value-to-go:
\begin{equation}
\begin{aligned}
\label{eq:pomdp:q}
\QFnBel{\policy}{t}(\belief, \action) =& \expect{\state \sim P(\state | \belief)}{\rewardFn{\state}{\action}} + \\
                        &\expect{\belief' \sim P(\belief' | \belief, \action)}{\valueFnBel{\policy}{t-1}(\belief')}
\end{aligned}
\end{equation}

\subsection{Mapping Informative Path Planning to POMDPs}
\label{sec:problem_formulation:ipp_mapping}

We now map IPP problems \hiddenunc and \hiddencon to a POMDP. 
The state is defined to contain all information that is required to define the reward, observation and transition functions. 
Let the state be the set of nodes visited and the underlying world, $\state_t = \{ \vertex_1, \dots, \vertex_t, \world\}$.
At the start of an episode, a world is sampled from a prior distribution $\world \sim P(\world)$ along with a graph $\graph \sim P(\graph)$. 
The initial state is assigned by setting $\state_1 = \{ \vertex_1, \world\}$.
Note that the state $\state_t$ is partially observable due to the hidden world map $\world$.

We define the action $\action_t = \vertex_{t+1}$ to be the next node to visit. 
We are now ready to map the utility and travel cost to the reward function definition.
Given the agent is in state $\state_t$ and has executed $\action_t$, we can extract the path $\Path = \seq{\vertex}{t+1}$ and the underlying world $\world$.
Hence we can compute the utility function $\utilityFn{\Path}{\world}$. 
We can also compute the travel cost function $\costFn{\Path}{\world}$. 

Before we define the reward function, we note that for Problem \hiddencon not all actions are feasible at all times due to connectivity of the graph and constraints due to travel cost.
Hence we can define a feasible set of actions $\actionSetFeas{\state} \subset \actionSet$ for a state as follows
\begin{equation}
   \actionSetFeas{\state} = \setst{\action}{\action \in \actionSet, (\vertex_t, \vertex_{t+1}) \in \edgeSet, \costFn{\Path}{\world} \leq \costBudget}
\end{equation} 
For Problem \hiddenunc, let $\actionSetFeas{\state} = \actionSet$.

Since the objective is to maximize the cumulative reward function, we define the reward to be proportional to the marginal utility of visiting a node. 
Given a node $\vertex \in \vertexSet$, a path $\Path$ and world $\world$, the marginal gain of the utility function $\utilityFnDef$ is $\marginalGain{\vertex \mid \Path}{\world} = \utilityFn{\Path \cup \{ \vertex \}}{\world} - \utilityFn{\Path}{\world}$. 
The one-step-reward function, $\rewardFn{\state}{\action}$, is defined as the marginal gain of the utility function.
Additionally, the reward is set to $-\infty$ whenever an infeasible action is selected. Hence:
\begin{equation}
	\rewardFn{\state}{\action} = .
	\begin{cases}
	\marginalGain{\action \mid \Path}{\world} &\text{if $\action \in \actionSetFeas{\state}$} \\
	-\infty & \text{otherwise}
	\end{cases}
\end{equation}

The state transition function, $\transFn{\state}{\action}{\state'}$, is defined as the deterministic function which sets $\vertex_{t+1} = \action_t$.
We define the observation to be the measurement $\obs_t = \meas_t$ and the observation model $\obsFnDef$ to be a deterministic function $\obs_{t} = \measFn{\vertex_t}{\world}$.

Note that the history $\belief_t$, the sequence of actions and observations, is captured in the sequence of nodes visited $\{\vertex_i\}_{i=1}^t$ and measurements received $\{\meas_i\}_{i=1}^t$. In our implementation, we encode this information in an occupancy map as described later in Section~\ref{sec:res_ipp:problem}. The information gathering policy $\policy(\belief_t)$ maps this history to an action $\action_t$, the sensing location to visit.

\subsection{Mapping Search Based Planning to POMDPs}

We now map the problem of computing a heuristic policy to a POMDP setting. 
Let the state be the open list and the underlying world, $\state_t = \{ \openList_t, \world\}$.
At the start of an episode, a world is sampled from a prior distribution $\world \sim P(\world)$ along with a start state $\start$. 
The initial state is assigned by setting $\state_1 = \{ \start, \world\}$.
Note that the state $\state_t$ is partially observable due to the hidden world map $\world$.

We define the action $\action_t$ as the vertex $\vertex \in \openList_t$ that is to be expanded by the search. 
The state transition function, $\transFn{\state}{\action}{\state'}$, is defined as the deterministic function which sets $\openList_{t+1}$ by querying $\expand(\vertex, \succFnDef, \evalFn, \world)$.
The one-step-reward function, $\rewardFn{\state}{\action}$, is defined as $-1$ for every $\pair{\state_t}{\action_t}$ until the goal is added to the open list.
Additionally, the reward is set to $-\infty$ whenever an infeasible action is selected. Hence:
\begin{equation}
	\rewardFn{\state}{\action} = .
	\begin{cases}
	-\infty & \text{if $\action \notin \openList$} \\
	0 &\text{if $\goal \in \openList$} \\
	-1 & \text{otherwise}
	\end{cases}
\end{equation}

We define the observation to be the successor nodes and invalid edges, i.e. $\obs_t = \{ \vertexSucc, \invalidEdge \}$ and the observation model $\obsFnDef$ to be a deterministic function $\pair{\vertexSucc}{\invalidEdge} = \expand(\vertex, \succFnDef, \evalFn, \world)$.

Note that the history, the sequence of actions and observations, is contained in the information present in the concatenation of all lists, i.e $\belief_t = \{ \openList, \closedList, \closedObsList \}$. The heuristic is a policy $\policy(\belief_t)$ that maps this history to an action $\action_t$, the vertex to expand.

Note that it is more natural to think of this problem as minimizing a one-step-cost than maximizing a reward. Hence when we subsequently refer to this problem instance, we refer to the cost $c(\state, \action) = -\rewardFn{\state}{\action}$ and the cost-to-go $\QFnBel{\policy}{t}(\belief, \action)$. This only results in a change from maximization to minimization.

\subsection{What makes these POMDPs intractable?}
\label{sec:problem:hardness}
A natural question to ask if these problems can be solved by state-of-the-art POMDP solvers such as POMCP~\cite{silver2010monte} or DESPOT~\cite{somani2013despot}. While such solvers are very effective at scaling up and solving large scale POMDPs, there are a few reasons why there are not immediately applicable to our problem. 

Firstly, these methods require a lot of online effort. In the case of search based planning, the effort required to plan in belief space defeats the purpose of a heuristic all together. In the case of informative path planning, the observation space is very large and belief updates would be time consuming. 

Secondly, since both methods employ a particle filter based approach to tracking plausible world maps, they both are susceptible to a realizability problem. Its unlikely that there will be a world map particle that will explain all observations. That being said, the world maps can explain local correlations in observations. For example, when planning indoors the world maps can explain correlations in observations made at intersection of corridors. Hence, we would like to generalize across these local submaps.

%% file: imitation_learning.tex

\section{Imitation of Clairvoyant Oracles}
\label{sec:imitation_learning}

A possible approach is to employ model free Q-learning~\citep{mnih-dqn-2015} by featurizing the history $\belief_t$ and collecting on-policy data. However, given the size of $\beliefSpace$, this may require a large number of samples. Another strategy is to parameterize the policy class and employ policy improvement~\citep{peters2006policy} techniques. However, such techniques when applied to POMDP settings may lead to poor local minima due to poor initialization.
We discussed in Section~\ref{sec:background:il} how imitation learning offers a more effective strategy than reinforcement learning in scenarios where there exist good policies for the original problem, however these policies cannot be executed online (e.g due to computational complexity) hence requiring imitation via an offline training phase. In this section, we extend this principle and show how imitation of \emph{clairvoyant oracles} enables efficient learning of POMDP policies. 

\subsection{Imitation Learning}

We now formally define imitation learning as applied to our setting. Given a policy $\policy$, we define the distribution of histories $P(\belief | \policy)$ induced by it (termed as \emph{roll-in}). Let $\lossFnPolicy{\belief}{\policy}$ be a loss function that captures how well policy $\policy$ imitates an oracle. Our goal is to find a policy $\policyLEARN$ which minimizes the expected loss as follows.

\begin{equation}
\label{eq:imitation_learning}
\policyLEARN = \argmin\limits_{\policy \in \policySet} \expect{ 
\belief \sim P(\belief | \policy)}
{\lossFnPolicy{\belief}{\policy}}
\end{equation} 

This is a non-i.i.d supervised learning problem. \citet{ross2011reduction} propose \FT to train a non-stationary policy (one policy $\policyLEARN_t$ for each timestep), where each policy $\policyLEARN_t$ can be trained on distributions induced by previous policies ($\policyLEARN_1, \dots, \policyLEARN_{t-1}$). While this solves the problem exactly, it is impractical given a different policy is needed for each timestep. For training a single policy, \citet{ross2011reduction} show how such problems can be reduced to no-regret online learning using dataset aggregation (\Dagger). The loss function they consider $\lossFnPolicyDef$ is a mis-classification loss with respect to what the expert demonstrated. \citet{ross2014reinforcement} extend the approach to the reinforcement learning setting where $\lossFnPolicyDef$ is the reward-to-go of an oracle reference policy by aggregating \emph{values} to imitate (\aggrevate).

\subsection{Solving POMDP via Imitation of a Clairvoyant Oracle}

To examine the applicability of imitation learning in the POMDP framework, we compare the loss function (\ref{eq:imitation_learning}) to the action value function (\ref{eq:pomdp:q}).
We see that a good candidate loss function $\lossFnPolicy{\belief}{\policy}$ should incentivize maximization of $\QFnBel{\policy}{T-t+1}(\belief, \policy(\belief))$. 
A suitable approximation of the optimal value function $\QFnBel{\policy^*}{T-t+1}$ that can be computed at train time would suffice. 
However, we cannot resort to oracles that explicitly reasoning about the belief over states $P(\state_t | \belief_t)$, let alone planning in this belief space due to tractability issues. 

In this work, we leverage the fact that for our problem domains, we have access to the true state $\state_t$ at train time. This allows us to define oracles that are \emph{clairvoyant} - that can observe the state at training time and plan actions using this information. 

\begin{definition}[Clairvoyant Oracle]
A clairvoyant oracle $\policyOR(\state)$ is a policy that maps state $\state$ to action $\action$ with an aim to maximize the cumulative reward of the underlying MDP $(\stateSet, \actionSet, \transFnDef, \rewardFnDef, T)$.
\end{definition}

The oracle policy defines an equivalent action value function \emph{defined on the state} as follows 

\begin{equation}
\label{eq:qvaloracle}
\QFn{\policyOR}{t}(\state, \action) = \rewardFn{\state}{\action} + \expect{\state' \sim P(\state' \mid \state, \action)}{
\valueFn{\policyOR}{t-1}(\state')} 
\end{equation}

Our approach is to imitate the oracle during training. This implies that we train a policy $\policyLEARN$ by solving the following optimization problem

\begin{equation}
\label{eq:imitateClairvoyantOracle}
\policyLEARN = \argmax\limits_{\policy \in \policySet} \expect{
\substack{t\sim U(1:T), \\
\state_t, \belief_t \sim P(\state, \belief | \policy, t)}}
{\QFn{\policyOR}{T-t+1}(\state_t, \policy(\belief_t))}
\end{equation}

While we will define training procedures to concretely realize (\ref{eq:imitateClairvoyantOracle}) later in Section~\ref{sec:approach}, we offer some intuition behind this approach. Since the oracle $\policyOR$ knows the state $\state$, it has appropriate information to assign a value to an action $a$. The policy $\policyLEARN$ attempts to imitate this action from the partial information content present in its history $\belief$. Due to this realization error, the policy $\policyLEARN$ visits a different state, updates the history, and queries the oracle for the best action. Hence while the learnt policy can make mistakes in the beginning of an episode, with time it gets better at imitating the oracle. 
\subsection{Analysis using a Hallucinating Oracle}
\label{sec:pomdp_imitate:hallucinating}
The learnt policy imitates a clairvoyant oracle that has access to more information (state $\state$ compared to history $\belief$). This results in a large realizability error
which is due to two terms - firstly the information mismatch between $\state$ and $\belief$, and secondly the expressiveness of feature space. This realizability error can be hard to bound making it difficult to apply the performance guarantee analysis of \cite{ross2014reinforcement}. It is also not desirable to obtain a performance bound with respect to the \emph{clairvoyant oracle} $\valuePol{\policyOR}$.

To alleviate the information mismatch, we take an alternate approach to analyzing the learner by introducing a purely hypothetical construct - a \emph{hallucinating oracle}.

\begin{definition}[Hallucinating Oracle]
\label{def:halluc}
A hallucinating oracle $\policyORBel$ computes the instantaneous posterior distribution over state $P(\state | \belief)$ and returns the expected clairvoyant oracle action value.
\begin{equation}
\label{eq:hallucinating_oracle}
\QFnBel{\policyORBel}{T-t+1}(\belief, \action) =  \expect{\state \sim P(\state | \belief)}{\QFn{\policyOR}{T-t+1}(\state, \action)}
\end{equation}
\end{definition}

We show that by imitating a clairvoyant oracle, the learner effectively imitates the corresponding hallucinating oracle 

\begin{lemma}
\label{lemma:hallucinating}
The \textbf{offline} imitation of \textbf{clairvoyant} oracle (\ref{eq:imitateClairvoyantOracle}) is equivalent to \textbf{online} imitation of a \textbf{hallucinating} oracle as shown

\begin{equation*}
\policyLEARN = \argmax\limits_{\policy \in \policySet} \expect{
\substack{
t\sim U(1:T), \\
\belief_t \sim P(\belief | \policy, t)}}
{\QFnBel{\policyORBel}{T-t+1}(\belief_t, \policy(\belief_t))}
\end{equation*}

\end{lemma}
\begin{proof} Refer to Appendix~\ref{appendix:lemma_hallucinating}.
\end{proof}

Note that a hallucinating oracle uses the same information content as the learnt policy. Hence the realization error is purely due to the expressiveness of the feature space. The empirical risk of imitating the hallucinating oracle will be significantly lower than the risk of imitating the clairvoyant oracle. 

Lemma~\ref{lemma:hallucinating} now allows us to express the performance of the learner with respect to a hallucinating oracle. This brings us to the key question - how good is a hallucinating oracle? Upon examining (\ref{eq:hallucinating_oracle}) we see that this oracle is equivalent to the well known QMDP policy first proposed by \cite{Littman95learningpolicies}. The QMDP policy ignores observations and finds the $Q_{\mathrm{MDP}}(\state, \action)$ values of the underlying MDP. It then estimates the action value by taking an expectation on the current belief over states $P(\state | \belief)$. This estimate amounts to assuming that any uncertainty in the agent's current belief state will be gone after the next action. Thus, the action where long-term reward from all states (weighed by the probability) is largest will be the one chosen.

\cite{Littman95learningpolicies} points out that policies based on this approach are remarkably effective. This has been verified by other works such as \citet{Koval-RSS-14} and \citet{javdani2015shared}. This naturally leads to the question of why we cannot directly apply QMDP to our problem. The QMDP approach requires explicitly sampling from the posterior over states online - a step that we cannot tractably compute as discussed in Section~\ref{sec:problem:hardness}. However, by imitating clairvoyant oracles, we implicitly obtain such a behaviour. 

Imitation of clairvoyant oracles has been shown to be effective in other domains such as receding horizon control via imitating MPC methods that have full information~\citep{kahn2016plato}. \cite{sun2017deeply} show how the partially observable acrobot can be solved by imitation of oracles having full state. \cite{karkus2017qmdp} introduce imitation of QMDP in a deep learning architecture to train POMDP policies end to end. 

The connection with a hallucinating oracle also provides valuable insight into potential failure situations. \cite{Littman95learningpolicies} point out that policies based on this approach will not take actions to gain information. We discuss such situations in Section~\ref{sec:discussion:success_failures}.

%% file: approach.tex

\section{Approach}
\label{sec:approach}

\subsection{Algorithms}
\begin{figure*}[t]
    \centering
    \includegraphics[width=\textwidth]{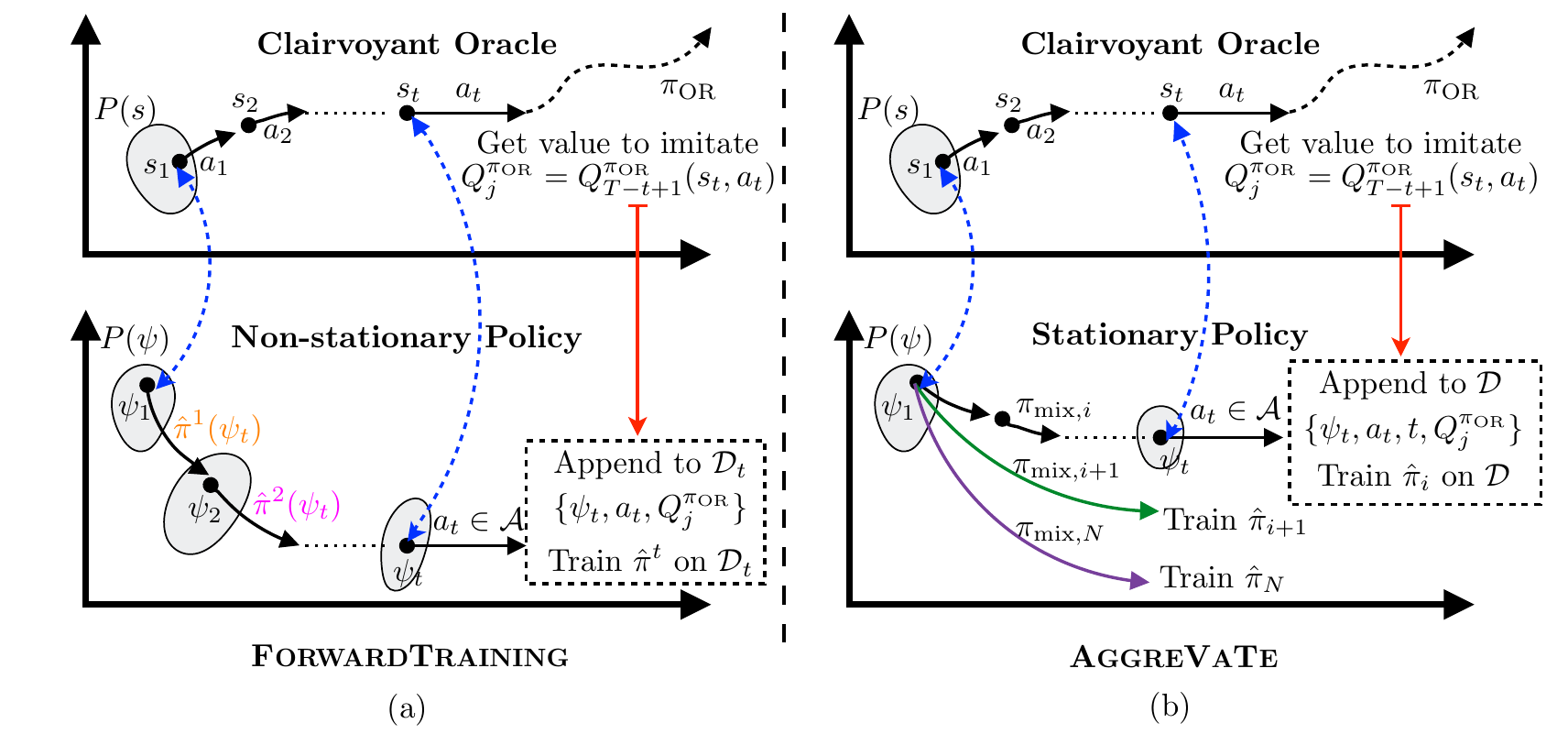}
    \caption{%
    Overview of the two approaches for training policies. 
    (a) \FT is used to train a non-stationary policy, i.e a sequence of policies $\policyLEARN^1, \ldots, \policyLEARN^{T}$ at each time-step. To train a policy at time-step $t$, a state $\state$ is sampled from initial distribution $P(\state)$. The policies  $\policyLEARN^1, \ldots, \policyLEARN^{t-1}$ are then used to roll-in to get $(\state_t, \belief_t)$. The oracle is queried to get $\QFn{\policyOR}{T-t+1}(\state_t, \action_t)$ which is then used to update the dataset and train policy $\policyLEARN^t$.
    (b) \aggrevate is used to train a stationary policy. The training process is iterative where dataset collection is interleaved with learning. At iteration $i$, a mixture policy $\policyMix{i}$ is used to roll-in to get $(\state_t, \belief_t)$. The oracle is queried to get $\QFn{\policyOR}{T-t+1}(\state_t, \action_t)$. The data is then aggregated to the whole dataset which is used to update the entire policy $\policyLEARN^{i}$. 
    \label{fig:ft_aggrevate}}
\end{figure*}%

We introduced imitation learning and its applicability to POMDPs in Section~\ref{sec:imitation_learning}. We now present a set of algorithms to concretely realize the process. The overall idea is as follows - we are training a policy $\policyLEARN(\belief)$ that maps features extracted from the history $\belief$ to an action $\action$. The training objective is to imitate a clairvoyant oracle that has access to the corresponding full state $\state$. In order to define concrete algorithms, we need to reason about two classes of policies - non-stationary and stationary.

\subsubsection{Non-stationary policy}

For the non-stationary case, we have a policy for each timestep $\policyLEARN^1, \ldots, \policyLEARN^{T}$. The motivation for adopting such a policy class is that the problems arising from the non i.i.d distribution immediately disappears. Such a policy class can be trained using the \FT algorithm \cite{ross2011reduction} which sequentially trains each policy on the distribution of features induced from the previous set of policies. Hence the training problem for each policy at timestep $t$ is reduced to supervised learning. 

\begin{algorithm}
\caption{\FT (Non-stationary policy) \label{alg:FT}}
\begin{algorithmic}[1]
\For{$t=1$ \textbf{to} $T$} \label{alg:FT:init}
\State Initialize $\dataset_t \gets \emptyset$.
\For{$j=1$ \textbf{to} $\numDatapoints$}
\State Sample initial state $\state_1$ from dataset $P(\state)$
\State Execute policy $\policyLEARN^1, \ldots, \policyLEARN^{t-1}$ to reach $\pair{\state_t}{\belief_t}$.\label{alg:FT:rollin} 
\State Execute any action $\action_t \in \actionSet$.
\State Collect value to go $\QVal_j^{\policyOR} =  \QFn{\policyOR}{T-t+1}(\state_t, \action_t)$ \label{alg:FT:oracle}
\State $\dataset_t \gets \dataset_t \cup \{\belief_t, \action_t, \QVal_j^{\policyOR}\}$ 
\EndFor
\State Train cost-sensitive classifier $\policyLEARN^t$ on $\dataset_t$
\EndFor
\State \textbf{Return} Set of policies for each time step $\policyLEARN^1, \ldots, \policyLEARN^{T}$ .
\end{algorithmic}
\end{algorithm}

Alg.~\ref{alg:FT} describes the \FT procedure to train the non-stationary policy. The policies are trained in a sequential manner. At each time-step $t$, the previously trained policies $\policyLEARN^1, \ldots, \policyLEARN^{t-1}$ are used to create a dataset of $\belief_t$ by rolling-in (Lines~\ref{alg:FT:init}--\ref{alg:FT:rollin}). For each such datapoint $\psi_t$, there is a corresponding state $\state_t$. A random action $\action_t$ is sampled and the oracle is queried for the cost-to-go $\QFn{\policyOR}{T-t+1}(\state_t, \action_t)$ (Line~\ref{alg:FT:oracle}). This is then added to the dataset $\dataset_t$ which is used to train the policy $\policyLEARN^t$. This is illustrated in Fig.~\ref{fig:ft_aggrevate}.

We can state the following property about the training process
\begin{theorem}
\label{theorem:ft}
\FT has the following guarantee
\begin{equation*}
\begin{aligned}
  \valuePol{\policyLEARN} \geq \valuePol{\policyORBel} -2 T \sqrt{\actionSet \; \errclass} + T\errhor
\end{aligned}
\end{equation*}
where $\errclass$ is the regression error of the learner and $\errhor$ is the local oracle suboptimality.
\end{theorem}
\begin{proof} Refer to Appendix~\ref{appendix:theorem_ft}.
\end{proof}

However, there are several drawbacks to using a non-stationary policy. Firstly, it is impractical to have a different policy for each time-step as it scales with $T$. While this might be a reasonable approach when $T$ is small (e.g. sequence classification problems~\citep{cohen2005stacked}), in our applications $T$ can be fairly large. 
Secondly, and more importantly, each policy operates on data for only that time-step, thus preventing generalizations across timesteps. Each policy sees  only $\frac{\mathcal{D}}{T}$ fraction of the training data. This leads to a high empirical risk. 

\subsubsection{Stationary policy}

A single stationary policy $\policyLEARN$ enjoys the benefit of learning on data across all timesteps. However, the non i.i.d data distribution implies the procedure of data collection and training cannot be decoupled - the learner must be involved in the data collection process. \citet{ross2014reinforcement} show that such policies can be trained by reducing the propblem to a no-regret online learning setting. They present an algorithm, \aggrevate that trains the policy in an interactive fashion where data is collected by a mixture policy of the learner and the oracle, the data is then \emph{aggregated} and the learner is trained on this aggregated data. This process is repeated.

\begin{algorithm}
\caption{\aggrevate (Stationary policy) \label{alg:Agg}}
\begin{algorithmic}[1]
\State Initialize $\dataset \gets \emptyset$, $\policyLEARN_1$ to any policy in $\policySet$ \label{alg:qvalAgg:init}
\For{$i=1$ \textbf{to} $\numLearnIter$}
\State Initialize sub-dataset $\dataset_i \gets \emptyset$\; \label{alg:qvalAgg:initSub}
\State Let roll-in policy be $\policyMix{i} = \mixfrac{i} \policyOR + (1-\mixfrac{i}) \policyLEARN_{i-1}$ \label{alg:qvalAgg:mixPol}
\State Collect $m$ data points as follows:
\For{$j=1$ \textbf{to} $\numDatapoints$}
\State Sample initial state $\state_1$ from dataset $P(\state)$ \label{alg:qvalAgg:sampleWorld}
\State Sample uniformly $t \in \{1,2,\dots,T\}$ \label{alg:qvalAgg:sampleTime}
\State Execute $\policyMix{i}$ up to time $t-1$ to reach $\pair{\state_t}{\belief_t}$ \label{alg:qvalAgg:rollin}
\State Execute any action $\action_t \in \actionSet$ \label{alg:qvalAgg:takeAction}
\State Collect value-to-go $\QVal_j^{\policyOR} =  {\QFn{\policyOR}{T-t+1}(\state_t, \action_t)}$ \label{alg:qvalAgg:collectVal}
\State $\dataset_i \gets \dataset_i \cup \{\belief_t, \action_t, t, \QVal_j^{\policyOR}\}$ \label{alg:qvalAgg:aggrSubData}
\EndFor
\State Aggregate datasets: $\dataset \gets \dataset \bigcup \dataset_i$ \label{alg:qvalAgg:aggrData}
\State Train cost-sensitive classifier $\policyLEARN_{i+1}$ on $\dataset$ \label{alg:qvalAgg:updateLearner}
\EndFor
\State \textbf{Return} best $\policyLEARN_i$ on validation
\end{algorithmic}
\end{algorithm}

Alg.~\ref{alg:Agg} describes the \aggrevate procedure to train the stationary policy. To overcome the non i.i.d distribution issue, the algorithm interleaves data-collection with learning and iteratively trains a set of policies $\seq{\policyLEARN}{\numLearnIter}$. Note that these iterations are not to be confused with time steps - they are simply learning iterations. A policy $\policyLEARN_i$ is valid for all timesteps. At iteration $i$, data is collected by rolling-in with a mixture of the learner and the oracle policy (Lines~\ref{alg:qvalAgg:init}--\ref{alg:qvalAgg:rollin}). The mixing fraction is chosen to be $\mixfrac{i} = (1 - \alpha)^{i-1}$. Mixing implies flipping a coin with bias $\mixfrac{i}$ and executing the oracle if heads comes up. A random action $\action_t$ is sampled and the oracle is queried for the cost-to-go $\QFn{\policyOR}{T-t+1}(\state_t, \action_t)$ (Line~\ref{alg:qvalAgg:collectVal}). 

The key step is to ensure that \emph{data is aggregated}. The motivation for doing so arises from the fact that we want the learner to do well on the distribution it induces. \cite{ross2014reinforcement} show that this can be posed as the mixture of learners $\seq{\policyLEARN}{\numLearnIter}$ doing well on the induced loss sequences $l_i(\policy)$ at every iteration. If we were to treat each iteration as a game in an online adversarial learning setting, this would be equivalent to having bounded regret with respect to the best policy in hindsight on the loss sequence $\seq{l}{\numLearnIter}$. The strategy of dataset aggregation is an instance of follow the leader and hence has bounded regret. Hence, data is appended to the original dataset and used to train an updated learner $\policyLEARN_{i+1}$ (Lines~\ref{alg:qvalAgg:aggrData}--\ref{alg:qvalAgg:updateLearner}).

\aggrevate can be shown to have the following guarantee
\begin{theorem}
\label{theorem:aggrevate}
$N$ iterations of \aggrevate, collecting $m$ regression examples per iteration guarantees that with probability at least $1-\delta$
\begin{equation*}
\begin{aligned}
  \valuePol{\policyLEARN} \geq & \valuePol{\policyORBel} \\
  & - 2 T \sqrt{\abs{\actionSet} \left( \errclass + \errreg + \bigo{\sqrt{\nicefrac{\log \nicefrac{1}{\delta}}{N m}}} \right) } \\
  & - \bigo{\frac{R \; T \log T}{N}} + T\errhor\\
\end{aligned}
\end{equation*}
where $\errclass$ is the empirical regression regret of the best regressor in the regression class on the aggregated dataset, $\errreg$ is the empirical online learning average regret on the sequence of training examples, $R$ is the range of oracle action value and $\errhor$ is the local oracle suboptimality.
\end{theorem}
\begin{proof} Refer to Appendix~\ref{appendix:theorem_aggrevate}.
\end{proof}

\subsection{Application to Informative Path Planning}
\label{sec:approach:ipp}

\begin{figure*}[!htp]
    \centering
    \includegraphics[width=\textwidth]{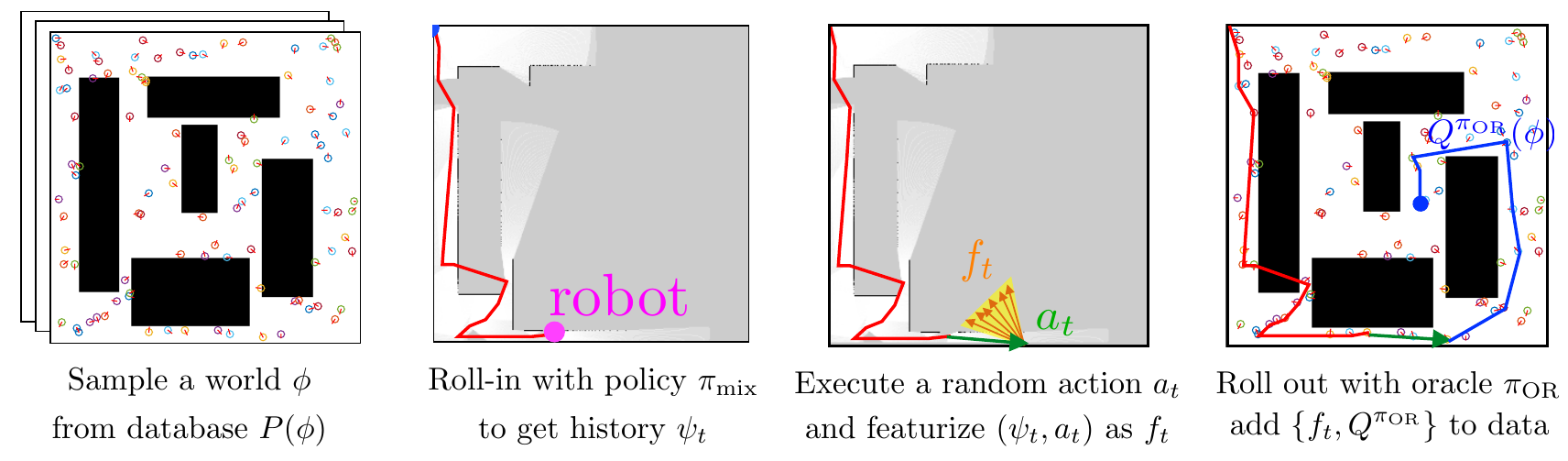}
    \caption{%
    An overview of \algQvalAgg in IPP where a learner $\policyLEARN$ is trained to imitate a clairvoyant oracle $\policyOR$. There are 4 key steps. Step 1: A world map $\world$ is sampled from database representing $P(\world)$. Step 2: A mixture policy $\policyMix{}$ of the learner and oracle is used to roll-in on $\world$ to a timestep $t$ to get history $\belief_t$. Step 3: A random action $\action_t$ is chosen and $(\belief_t, \action_t)$ is featurized as $f_t$. Step 4: A clairvoyant oracle $\policyOR$ is given full access to world map $\world$ to compute the cumulative reward to go $Q^{\policyOR}$. The pair $(f_t, Q^{\policyOR})$ is added to data to update the learner. This process is repeated to train a sequence of learners.}
    \label{fig:algorithm_qvalagg}
\end{figure*}%

We now consider the applicability of Alg.~\ref{alg:FT} and Alg.~\ref{alg:Agg} for learning a policy to plan informative paths. We refer to the mapping of the IPP problem to a POMDP defined in Section~\ref{sec:problem_formulation:ipp_mapping}. We first need to define a clairvoyant oracle in this context. 
Recall that the state $\state_t = \{\vertex_1, \dots, \vertex_t, \world\}$ is the set of nodes visited and the underlying world. A clairvoyant oracle takes a state action pair $(\state_t, \action_t)$ as input and computes a value. Depending on whether we are solving Problem \hiddenunc or \hiddencon, we explore two different kinds of oracles:
\begin{enumerate}
    \item \textit{Clairvoyant One-step-reward}
    \item \textit{Clairvoyant Reward-to-go}
\end{enumerate}

\subsubsection{Solving \hiddenunc by Imitating Clairvoyant One-step-reward}
\label{sec:approach:ipp:one_step_reward}
We first define a Clairvoyant One-step-reward oracle in the IPP framework.

\begin{definition}[Clairvoyant One-step-reward]
\label{def:clair_onestep_rew}
A Clairvoyant One-step-reward returns an action value $\QFn{\policyOR}{t}(\state, \action) = \rewardFn{\state}{\action}$ that considers only the one-step-reward. In the context of \hiddenunc, it uses the world map $\world$, the curent path $\{\vertex_1, \dots, \vertex_t\}$, the next node to visit $\vertex_{t+1} = \action_t$ to compute the value $\QOR(\world, \{\vertex_1, \dots, \vertex_t\}, \vertex_{t+1})$ as the marginal gain in  utility, i.e.
\begin{equation*}
  \marginalGain{\vertex_{t+1} \mid \{\vertex_1, \dots, \vertex_t\}  }{\world}
\end{equation*}
\end{definition}

To motive the use of Clairvoyant One-step-reward, we refer to the discussion on the structure of the Problem \hiddenunc in Section~\ref{sec:background:ipp:problem_hidden}. We assume that the utility function is \emph{adaptive monotone submodular} - it has the property of montonicity and diminishing returns under the belief over world maps. 
 This property implies the following
\begin{enumerate}
\item \emph{Adaptive Monotonicity}: The expected value of the utility can only increase on adding a node, i.e. 
\begin{equation*}
  \expect{\world \sim P(\world | \belief )}{ \marginalGain{\vertex \mid \vertexSet_\belief }{\world} }  \geq 0
\end{equation*}
for all $\vertex \in \vertexSet$, where $\belief = \{ \vertex_i \}_{i=1}^p, \{ \meas_i \}_{i=1}^p$, and $\vertexSet_\belief = \{\vertex_i \}_{i=1}^p$.
\item \emph{Adaptive Submodularity}: The expected gain in adding a node diminshes as more nodes are visited, i.e.
\begin{equation*}
\begin{aligned}
  \expect{\world \sim P(\world | \belief )}{ \marginalGain{\vertex \mid \vertexSet_\belief }{\world} } \geq \\
  \expect{\world \sim P(\world | \belief' )}{ \marginalGain{\vertex \mid \vertexSet_{\belief'} }{\world} } 
\end{aligned}
\end{equation*}
for all $\vertex \in \vertexSet$, where $\belief \subseteq \belief'$ (history $\belief$ is contained in history $\belief'$)
\end{enumerate}

For such functions, \cite{golovin2011adaptive} show that greedily selecting vertices to visit is near-optimal. We use this property to show that the Clairvoyant One-step-reward induces a one-step-oracle which is equivalent to the greedy policy and hence near optimal. This implies the following Lemma

\begin{theorem}
\label{theorem:hidden_unc}
$N$ iterations of \aggrevate with Clairvoyant one-step-reward collecting $m$ regression examples per iteration guarantees that with probability at least $1-\delta$
\begin{equation*}
\begin{aligned}
  \valuePol{\policyLEARN} \geq & \left(1 - \frac{1}{e}\right)\valuePol{\policy^*} \\
  & - 2 T \sqrt{\abs{\actionSet} \left( \errclass + \errreg + \bigo{\sqrt{\nicefrac{\log \nicefrac{1}{\delta}}{N m}}} \right) } \\
  & - \bigo{\frac{R \; T \log T}{N}}\\
\end{aligned}
\end{equation*}
where $\errclass$ is the empirical regression regret of the best regressor in the regression class on the aggregated dataset, $\errreg$ is the empirical online learning average regret on the sequence of training examples, $R$ is the maximum range of one-step-reward.
\end{theorem}
\begin{proof} Refer to Appendix~\ref{appendix:theorem_hiddenunc}.
\end{proof}

We will shown in Section~\ref{sec:res_ipp} that such policies are remarkably effective. An added benefit of imitating the Clairvoyant One-step-reward is that the empirical classification loss $\errclass$ is lower since only the expected one-step-reward of an action needs to be learnt. 

\subsubsection{Solving \hiddencon by Imitating Clairvoyant Reward-to-go}
Unforutunately, Problem \hiddencon does not posses the adaptive-submodular property of \hiddenunc due to the introduction of the travel cost. Hence imitating the one-step-reward is no longer appropriate. We define the Clairvoyant Reward-to-go oracle for this problem class
\begin{definition}[Clairvoyant Reward-to-go]
A Clairvoyant Reward-to-go returns an action value $\QFn{\policyOR}{t}(\state, \action)$ that corresponds to the cumulative reward obtained by executing $\action$ and then following the oracle policy $\policyOR$. In the context of \hiddencon, it uses the world map $\world$, the curent path $\{\vertex_1, \dots, \vertex_t\}$, the next node to visit $\vertex_{t+1} = \action_t$ to solve the problem \knowncon and compute a future sequence of nodes $\{\vertex_{t+2}, \dots, \vertex_T\}$. This provides the value $\QOR(\world, \{\vertex_1, \dots, \vertex_t\}, \vertex_{t+1})$ as the marginal gain
\begin{equation*}
  \marginalGain{\{\vertex_{t+1}, \dots, \vertex_T\} \mid \{\vertex_1, \dots, \vertex_t\}  }{\world}
\end{equation*}
The correspoding oracle policy $\policyOR$ is obtained by following the computed path. 
\end{definition}

Note that solving \knowncon is NP-Hard and even the best approximation algorithms require some computation time. Hence the calls to the oracle must be minimized. 

\subsubsection{Training and Testing Procedure}

We now present concrete algorithms to realize the training procedure. Given the two axes of variation - problem and policy type - we have four possible algorithms 
\begin{enumerate}
    \item \algRewFT: Imitate one-step-reward using non-stationary policy by \FT (Alg.~\ref{alg:FT})
    \item \algQvalFT: Imitate reward-to-go using non-stationary policy by \FT (Alg.~\ref{alg:FT})
    \item \algRewAgg: Imitate one-step-reward using stationary policy by \aggrevate (Alg.~\ref{alg:Agg})
    \item \algQvalAgg: Imitate reward-to-go using stationary policy by \aggrevate (Alg.~\ref{alg:Agg})
\end{enumerate}
Table.~\ref{tab:alg:mapping} shows the algorithm mapping.

 \begin{table}[!htbp]
    \centering
    \caption{Mapping from Problem and Policy type to Algorithm}
    \begin{tabulary}{0.8\textwidth}{L|CC}\toprule
       \diagbox{\bf Policy}{\bf Problem}       &   \hiddenunc         & \hiddencon        \\ \midrule
       Non-stationary policy             &    \algRewFT      & \algQvalFT           \\
       Stationary policy                     &    \algRewAgg     & \algQvalAgg          \\ \bottomrule
    \end{tabulary}
    \label{tab:alg:mapping}
\end{table}

\begin{algorithm}
\caption{\algQvalAgg \label{alg:qvalAgg}}
\begin{algorithmic}[1]
\State Initialize $\dataset \gets \emptyset$, $\policyLEARN_1$ to any policy in $\policySet$ 
\For{$i=1$ \textbf{to} $\numLearnIter$}
\State Initialize sub-dataset $\dataset_i \gets \emptyset$\; 
\State Let roll-in policy be $\policyMix{i} = \mixfrac{i} \policyOR + (1-\mixfrac{i}) \policyLEARN_i$ 
\State Collect $m$ data points as follows:
\For{$j=1$ \textbf{to} $\numDatapoints$}
\State Sample world $\world$ from dataset $P(\state)$ 
\State Sample start node $\vertexStart$ for $P(\vertexStart)$
\State Sample uniformly $t \in \{1,2,\dots,T\}$ 
\State Execute $\policyMix{i}$ up to time $t-1$
\item[]\hspace*{10mm} to get path $\{\vertex_1, \dots, \vertex_t\}$ and history ${\belief_t}$
\State Sample a random action $\action_t \in \actionSet$  
\item[] \hspace*{10mm}as the next vertex to visit $\vertex_{t+1} = \action_t$
\State Invoke Clairvoyant Reward-to-go oracle 
\item[]\hspace*{10mm} to get $\QVal_j^{\policyOR} = \QOR\{\world, \{\vertex_1, \dots, \vertex_t\}, \vertex_{t+1}\}$.
\State $\dataset_i \gets \dataset_i \cup \{\belief_t, \action_t, t, \QVal_j^{\policyOR}\}$ 
\EndFor
\State Aggregate datasets: $\dataset \gets \dataset \bigcup \dataset_i$ 
\State Train cost-sensitive classifier $\policyLEARN_{i+1}$ on $\dataset$
\EndFor
\State \textbf{Return} best $\policyLEARN_i$ on validation
\end{algorithmic}
\end{algorithm}

For completeness, we concretely define the training procedure for \algQvalAgg in Alg.~\ref{alg:qvalAgg}. The procedure for the remaining three algorithms can be inferred from this. The algorithm iteratively trains a sequence of policies $\seq{\policyLEARN}{N}$. At every iteration $i$, the algorithm conducts $m$ episodes. In every episode a different world map $\world$ and start vertex $(\vertexStart)$ is sampled from a database. The roll-in is conducted with a mixture policy $\policyMix{i}$ which blends the learner's current policy, $\policyLEARN_{i-1}$ and the oracle's policy, $\policyOR$ using blending parameter $\mixfrac{i}$. The blending is done in an episodic fashion, with probability $\mixfrac{i}$ the Clairvoyant Reward-to-go oracle is invoked to compute a path which is followed. With probability $1 - \mixfrac{i}$, the learner is invoked for the whole episode. In a given episode, the roll-in is conducted to a timestep $t$ which is uniformly sampled. At the end of the roll-in, we have a path $\{\vertex_1, \dots, \vertex_t\}$ and a history ${\belief_t}$. A random action $\action_t \in \actionSet$ is sampled which defines the next vertex to visit $\vertex_{t+1} = \action_t$. The Clairvoyant Reward-to-go oracle is invoked with the world $\world$ and the path already travelled $\{\vertex_1, \dots, \vertex_t\}, \vertex_{t+1}\}$. It then invokes a solver to \hiddencon to complete the path and return the reward to go $\QVal_j^{\policyOR}$ . This history action pair $(\belief_t, \action_t)$ is projected to a feature space along with label $\QVal_j^{\policyOR}$. The data is aggregated to the dataset which is eventually used to train policy $\policyLEARN_{i+1}$.
Fig.~\ref{fig:algorithm_qvalagg} illustrates this approach. 

\subsection{Application to Search Based Planning}
\begin{figure*}[!htp]
    \centering
    \includegraphics[width=\textwidth]{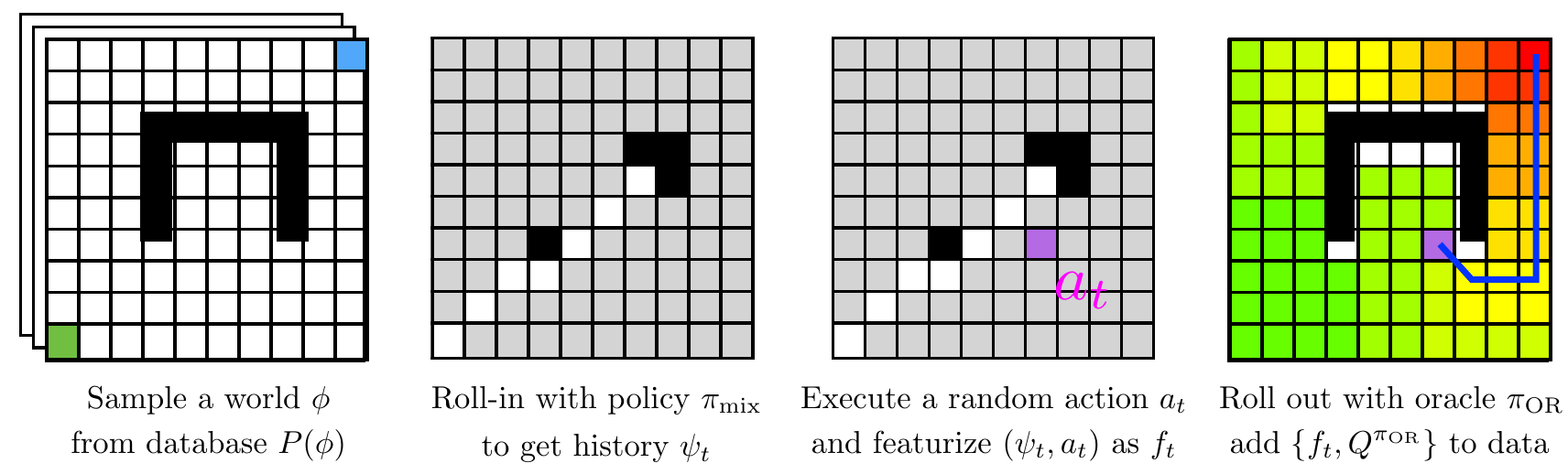}
    \caption{%
    An overview of \algSail in search based planning where a learner $\policyLEARN$ is trained to imitate a clairvoyant oracle $\policyOR$. There are 4 key steps. Step 1: A world map $\world$ is sampled from database representing $P(\world)$ along with start goal pair $\pair{\start}{\goal}$. Step 2: A mixture policy $\policyMix{}$ of the learner and oracle is used to roll-in on $\world$ to a timestep $t$ to get history $\belief_t$ which is the combination of open list, closed list and invalid edges. Step 3: A random vertex $\action_t$ from the open list is chosen and $(\belief_t, \action_t)$ is featurized as $f_t$. Step 4: A clairvoyant oracle $\policyOR$ is given full access to world map $\world$ to compute the cumulative cost to go $Q^{\policyOR}$. The pair $(f_t, Q^{\policyOR})$ is added to data to update the learner. This process is repeated to train a sequence of learners.
    }
    \label{fig:algorithm_sail}
\end{figure*}%

We now consider the applicability of Alg.~\ref{alg:Agg} for heuristic learning in search based planning. Unlike the IPP problem domain, there is no incentive to use a non-stationary policy or imitate Clairvoyant One-step-rewards. Hence we only consider training a stationary policy imitating Clairvoyant Reward-to-go. 

We first need to define a clairvoyant oracle for this problem. Given access to the world map $\world$, the oracle has to solve for the optimal number of expansions to reach the goal. This allows us to define a \emph{clairvoyant oracle planner} that employs a \emph{backward} Dijkstra's algorithm, which given a world $\world$ and a goal vertex $\goal$ plans for the optimal path from every $\vertex \in \vertexSet$ using dynamic programming. 

\begin{definition}[Clairvoyant Oracle Planner]
  \label{def:clairvoyant_oracle}
  Given full access to the state $\state$, which contains the open list $\openList$ and world $\world$, and a goal $\goal$, the oracle planner encodes the cost-to-go from any vertex $\vertex \in \vertexSet$ as the function $\QFn{\policyOR}{t}(\state, \action)$ which implicitly defines an oracle policy, $\policyOR(\state) \; = \; \argminprob{\vertex\in\openList} \; \QFn{\policyOR}{t}(\state, \action)$.
\end{definition}
The clairvoyant oracle planner provides a look-up table $\costToGoOracle\pair{\world}{\vertex}$ for the optimal cost-to-go from any vertex irrespective of the current state of the search. 

A key distinction between this oracle and the one defined for an IPP problem in Section~\ref{sec:approach:ipp} is that we are able to efficiently get the cost-to-go value for all states by dynamic programming - we do not need to repeatedly invoke the oracle. 
We exploit this fact by extracting multiple labels from an episode even though the oracle is invoked only once.
Additionally, this allows us a better roll-in procedure where the oracle and learner are interleaved. 
We adapt the \aggrevate framework to present an algorithm, \emph{Search as Imitation Learning} (\algSail).  

\begin{algorithm}
\caption{\algSail $(P(\world), P(\start, \goal), k)$ \label{alg:sail_alg}}
\begin{algorithmic}[1]
\State Initialize $\dataset \gets \emptyset$, $\policyLEARN_1$ to any policy in $\policySet$  \label{lst:line:}
\For{$i=1$ \textbf{to} $\numLearnIter$}
\State Initialize sub dataset $\dataset_i \gets \emptyset$ 
\State Collect $mk$ data points as follows:
\For{$j=1$ \textbf{to} $m$}
\State Sample world map $\world \sim P(\world)$ 
\State Sample $\pair{\start}{\goal} \sim P(\start, \goal)$
\State Invoke clairvoyant oracle planner 
\item[]\hspace*{10mm} to compute $\QFn{\policyOR}{}(\world, \vertex) \; \forall \; \vertex \in \vertexSet$
\State Sample uniformly $k$ timesteps $\seqset{t}{k}$
\item[]\hspace*{10mm} where each $t_{i} \in \ \set{1, \ldots ,\planTime}$
\State Rollout search with 
\item[]\hspace*{10mm} $\policyMix{i} = \mixfrac{i} \policyOR + (1-\mixfrac{i}) \policyLEARN_i$ 
\State At each $t\in\seqset{t}{k}$ pick a random 
\item[]\hspace*{10mm} action $\action_t$ to get corresponding $\pair{\belief_t}{\vertex}$
\State Query oracle for $\costToGoOracle\pair{\world}{\action_t}$
\State $\dataset_i \gets \dataset_i \cup \{\belief_t, \action_t, t, \costToGoOracle\pair{\world}{\action_t} \}$ 
\EndFor
\State Aggregate datasets: $\dataset \gets \dataset \bigcup \dataset_i$ 
\State Train cost-sensitive classifier $\policyLEARN_{i+1}$ on $\dataset$
\EndFor
\State \textbf{Return} best $\policyLEARN_i$ on validation
\end{algorithmic}
\end{algorithm}

Alg.~\ref{alg:sail_alg}, describes the $\algSail$ framework which iteratively trains a sequence of policies $\seq{\policyLEARN}{N}$. For training the learner, we collect a dataset $\dataset$ as follows - At every iteration \emph{i}, the agent executed \emph{m} different searches (Alg. \ref{alg:search}). For every search, a different world $\world$ and the pair $(\start, \goal)$ is sampled from a database.  The agent then rolls-out a search with a mixture policy $\policyMix{i}$ which blends the learner's current policy, $\policyLEARN_{i}$ and the oracle's policy, $\policyOR$ using blending parameter $\mixfrac{i}$. During the search execution, at every timestep in a set of $k$ uniformly sampled timesteps, we select a random action from the set of feasible actions and collect a datapoint $\{\belief_t, \action_t, t, \costToGoOracle\pair{\world}{\action_t} \}$. The policy $\policyMix{i}$ is rolled out till the end of the episode and all the collected data is aggregated with dataset $\dataset$. At the end of N iterations, the algorithm returns the best performing policy on a set of held-out validation environment or alternatively, a mixture of $\seq{\policyLEARN}{N}$. Fig.~\ref{fig:algorithm_sail} illustrates the \algSail framework.

Note that while the oracle is invoked once per $\world$, we obtain $k$ datapoints - this is critical for speeding up training.
We also note that even though the time complexity of $\select$ is $O \left(|\openList_{t}|\right)$ at timestep $t$, $\algSail$ can have better overall complexity if it can achieve a squared reduction in number of expansions compared to uninformed search as discussed more in Appendix~\ref{appendix:sail_complexity}.

%% file: results_informative_path_planning.tex

\section{Experiments on Informative Path Planning}
\label{sec:res_ipp}

In this section, we extensively evaluate our approach on a set of 2D and 3D informative path planning problems across a spectrum of synthetic and real world environments. We examine a class of informative path planning problem where a robot, equipped with a range limited sensor, possibly constrained by time and fuel resources, is tasked with 3D reconstruction of structures in the world. We choose a variety of environments to highlight the importance of adaptive behaviours for information gathering. 
Our implementation is open sourced for both MATLAB and C++ (\url{https://bitbucket.org/sanjiban/matlab_learning_info_gain}). 

\subsection{Problem Details} 
\label{sec:res_ipp:problem}
We consider both 2D and 3D informative path planning problems. The world map $\world$ is represented as a 2D or 3D binary grid, i.e. a grid cell is either occupied or free. The candidate set of sensing locations $\vertexSet$ is generated by uniformly randomly sampling nodes in the configuration space of the robot. For 2D problems, the configuration space of the robot is $SE(2)$, for 3D it is $SE(3)$. We assume for simplicity that the robot can teleport between any two nodes $\vertex_i$ and $\vertex_j$ and the cost of travel is the 2D/3D euclidean straight-line distance $\costFnDef(\{ \vertex_i, \vertex_j \}, \world) = \norm{\vertex_i - \vertex_j}{2}$. It would be straightforward to incorporate practical constraints such as collision avoidance by only allowing motion between vertices that are known to collision free and computing travel cost to be the arc length distance of a collision free path.

We assume that the robot is equipped with a field-of-vision (FOV) and range limited sensor. When a robot visits a node $\vertex$ in a world map $\world$, the measurement received by the robot, $\meas = \measFn{\vertex}{\world}$, is computed by ray-casting the sensor on the world and obtaining a scan line (2D) or a depth-image (3D). 

The utility function $\utilityFnDef$ is selected to be the fractional coverage function (similar to \cite{isler2016information}) which is defined as follows. Let the robot traverse a path $\Path = \seq{\vertex}{p}$ in a world $\world$. For each node $\vertex_i \in \Path$ we have a corresponding measurement $\meas_i$. Let the coverage map $C_i$ be a binary grid whose cells are $1$ iff the corresponding cell in $\world$ is occupied and $\meas_i$ contains a point in that cell. The total coverage map of a path $\Path$ is a union of all coverage maps $C = \bigcup\limits_{i = 1}^{p} C_i$. Then the utility function is the ratio of the total coverage and the total occcupied cells in the world map, i.e. $\utilityFnDef(\Path, \world) = \frac{\norm{C}{1}}{\norm{\world}{1}}$. 

While we assume the objective of the robot is to `uncover' every cell of the hidden world map, this framework can also allow a more task specific objective. For example, if the objective is to perform surface reconstruction of a specific object (and not of every surface in the world map), the utility function can be modified to only cover gridcells belonging to that object. The quality of an observation can also be included in the utility, i.e. measurements at close range can be weighted more than measurements taken from far away.

The values of total time step $T$ and travel budget $\costBudget$ vary with problem instances and are specified along with the results.

The history of events $\belief_t$ is represented as an occupancy grid $\occGridDef$ where each grid cell $\gridCell \in \occGridDef$ corresponds to an occupancy value $\probOcc{\gridCell} \in [0, 1]$. Every time a new measurement is received, $\occGridDef$ is updated by ray-casting and applying Bayes' rule~\citep{thrun2005probabilistic}. The policy $\policy(\belief_t)$ takes as input the occupancy grid and selects an action $\action_{t+1}$ that corresponds to the next node $\vertex$ to be visited.

\subsection{Baseline: Information Theoretic Heuristics}
\label{sec:res_ipp:baseline}

\citet{isler2016information} propose a set of information theoretic heuristics that quantify the information gain of obtaining a measurement for the task of volumetric reconstruction which include visibility likelihood and the likelihood of seeing new parts of the object. These heuristics are variants of Shannon's entropy where cells are weighted by an importance function. All of the heuristics are myopic, i.e. given the current occupancy grid, each candidate node is evaluated and the best node is selected as the next action. 
We briefly describe these heuristics and ask the reader to refer to \citet{isler2016information} for further details.

To evaluate a node $\vertex$, a set of rays $\raySet(\vertex)$ are cast from the node using the specifications of the sensor model. A ray $\ray \in \raySet$ corresponds to a set of grid cells in the occupancy grid $\occGrid{\ray}$. Given a grid cell $\gridCell$, the probability of it being occupied is $\probOcc{\gridCell}$ and being free is $\probFree{\gridCell}$. This can be used to compute various information gain metrics according to different heuristics. Let $\infoGainFn{}{\gridCell}$ be the information stored in the grid cell $\gridCell$. Then the information gain for a node is given by 
\begin{equation}
  \infoGainVertFn{}{\vertex} = \sum\limits_{\forall \ray \in \raySet(\vertex)} \sum\limits_{\forall \gridCell \in \occGrid{\ray}} \infoGainFn{}{\gridCell}
\end{equation}

Depending on the type of information gain $\infoGainFnDef$, there can be several information gain functions
\begin{enumerate}
\item Average Entropy: $\infoGainVertFn{o}{\vertex}$

This corresponds to the entropy
\begin{equation}
  \infoGainFn{o}{\gridCell} = -\probOcc{\gridCell} \log \probOcc{\gridCell} - \probFree{\gridCell} \log \probFree{\gridCell}
\end{equation}

\item Occlusion Aware Entropy: $\infoGainVertFn{v}{\vertex}$

This corresponds to considering the visibility likelihood of a grid cell
\begin{equation}
  \infoGainFn{v}{\gridCell} = P_v(\gridCell) \infoGainFn{o}{\gridCell}
\end{equation}
where $P_v(\gridCell)$ is the likelihood of the ray $\ray$ leading to the $\gridCell$ being free. 

\item Unobserved Voxel: $\infoGainVertFn{u}{\vertex}$

This corresponds to only considering unknown grid cells
\begin{equation}
  \infoGainFn{u}{\gridCell} = \begin{cases}
1 &\text{if $\gridCell$ is unknown}\\
0 &\text{otherwise}
\end{cases}
\end{equation}

\item Unobserved Entropy: $\infoGainVertFn{k}{\vertex}$

This is the composition of unobserved voxel with occlusion aware entropy
\begin{equation}
  \infoGainFn{k}{\gridCell} = \infoGainFn{u}{\gridCell} \infoGainFn{v}{\gridCell}
\end{equation}

\item Rear Side Voxel: $\infoGainVertFn{b}{\vertex}$

Let $RS$ be the set of \emph{rear-side} grid-cells defined as occluded, unknown gird cells adjacent on the ray to an occupied grid cell. Then
\begin{equation}
  \infoGainFn{b}{\gridCell} = \begin{cases}
1 &\text{if $\gridCell \in RS$}\\
0 &\text{otherwise}
\end{cases}
\end{equation}

\item Rear Side Entropy: $\infoGainVertFn{n}{\vertex}$

This is the composition of rear side voxel with occlusion aware entropy
\begin{equation}
  \infoGainFn{n}{\gridCell} = \infoGainFn{b}{\gridCell} \infoGainFn{v}{\gridCell}
\end{equation}
\end{enumerate}

The heuristics are used in a greedy fashion as follows. Given the robot has already visited nodes $\vertex_1, \dots, \vertex_{t-1}$, it decides to visit node $\vertex_t$ according to the following rule

\begin{equation}
  \vertex_t = \argmaxprob{\vertex_t \in \vertexSet} \frac{ \infoGainVertFn{}{\vertex_t} }{ \sum_{\vertex \in \vertexSet} \infoGainVertFn{}{\vertex} } - 
  \lambda \frac{ \norm{ \vertex_t - \vertex_{t-1}}{2} }{ \sum_{\vertex \in \vertexSet} \norm{ \vertex - \vertex_{t-1}}{2} }
\end{equation}

When applied to the Problem \hiddenunc, we de-activate the penalization and set $\lambda = 0$. 

\subsection{Imitation Learning Details}
\subsubsection{Feature Extraction and Learner}
The policy maps the history $\belief$ to an action $\action$ by learning a function approximation for the action value function $\hat{Q}(\belief, \action)$. The tuple $\left(\action, \belief \right)$ is mapped to a vector of features $\feature =  \bbm \featureIG^T & \featureMot^T \ebm^T$. The first set of features $\featureIG \in \real^6$ are the information gain heuristics defined in Section~\ref{sec:res_ipp:baseline}. These heuristics are computed using the occupancy map corresponding to history $\belief$ and the candidate node corresponding to action $\action$. There are several reasons for using these heuristics as the feature vector. They allow generalization across different instance of the world map. They also allow for fare comparison against the heuristics as baseline  approaches - the learner learns a trade-off between heuristics.

$\featureMot \in \real^7$ encodes the distance already travelled by the robot $(\real^1)$, the relative translation $(\real^3)$ and rotation $(\real^3)$ to visit the candidate node from the current node. These set of features capture the travel cost trade-off for visiting a node. 

We use random forest regression as a function approximator~\citep{liaw2002classification}.

\subsubsection{Dataset Creation}
The 2D world maps are created by randomly distributing geometric objects such as rectangles and circles according to hand design parametric distribution. The 3D world maps are created using the ROS-Gazebo simulator and randomly distributing 3D object meshes. Depending on the environment (such as construction site or office-desk), different collection of objects and parametric distributions are selected. 

For the experiment on a real dataset, we used registered RGBD data collected by \cite{sturm12iros}. The original dataset is a set of registered point cloud along with the measurement pose. This dataset can be used to create the world map $\world$.
The set of poses are used to create a fully connected graph $\vertexSet$. The algorithm is then restricted to choosing a subset of these poses to maximize the utility. Every time the algorithm visits a node $\vertex_i$, the corresponding measurement $\meas_i$ is returned. 
We found that this setup allowed us to easily evaluate information gathering algorithms on real data in a completely decoupled manner from the data collection process. 

This process of dataset creation motivates the applicability of our method in practical settings. Given a new environment, we can envision collecting a dataset open-loop, either via manual operation or via some base exploration policy. We can then learn an efficient policy on this dataset and subsequently used the learnt policy for future operations. The generalization capability of the learner allows performance to be transferred to environments with similar object configurations.

\begin{figure*}[t]
    \centering
    \includegraphics[width=\textwidth]{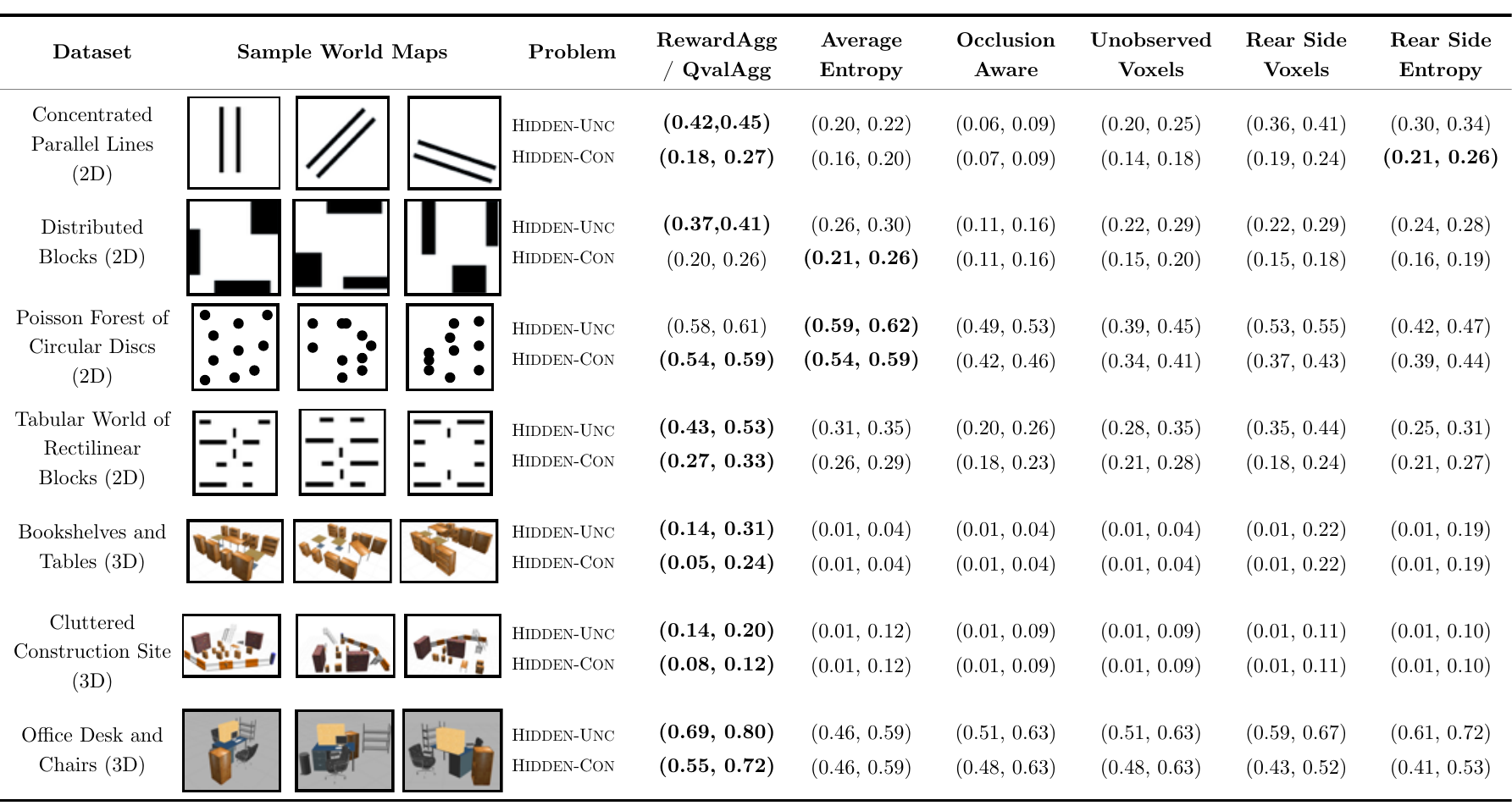}
    \caption{%
    Results for Problems \hiddenunc and \hiddencon on a spectrum of 2D and 3D exploration problems. The train size is $100$ and test size is $10$. Numbers are the confidence bounds (for 95\% CI) of cumulative reward at the final time step. Algorithm with the highest median performance is emphasized in bold.}
    \label{fig:results:extra}
\end{figure*}%

\subsubsection{Clairvoyant Oracle}

For algorithms \algRewAgg and \algRewFT, the clairvoyant oracle is simply the one-step-reward function, i.e. the marginal utility of visiting a node given the history of nodes visited. An important implementation detail is that when using the one-step-reward oracle, the call to the oracle is inexpensive. Hence, instead of sampling a random action and obtaining its value, all actions can be queried. This dramatically improves the convergence due to the increase in data size.

For \algQvalAgg and \algQvalFT, the clairvoyant oracle needs to solve the submodular routing problem (Problem \knowncon). We use the Generalized Cost Benefit (GCB)~\citep{zhang2016submodular} algorithm - an efficient greedy algorithm with bi-criterion approximation guarantees. The core idea of the algorithms is very simple: at iteration $i$ select a node $\vertex_i$ that maximizes the ratio of the marginal gain in utility and the marginal gain in travel cost

\begin{equation}
  \label{eq:gcb}
  \vertex_i = \argmaxprob{\vertex \in \vertexSet} 
  \frac{ \utilityFn{ \vertex_i \cup \{ \vertex_j \}_{j=1}^{i-1} }{\world} - \utilityFn{ \{ \vertex_j \}_{j=1}^{i-1} }{\world} }      
       { \costFn{ \vertex_i \cup \{ \vertex_j \}_{j=1}^{i-1} }{\world}    - \costFn{ \{ \vertex_j \}_{j=1}^{i-1} }{\world}  }
\end{equation}

Once a vertex $\vertex_i$ is selected, a TSP solver is invoked to find the minimum cost route through nodes $\vertex_1, \dots, \vertex_i$ and the vertices are re-ordered accordingly. The process is repeated till the travel budget constraints are met. Note that computing the denominator exactly in (\ref{eq:gcb}) might be expensive since it involves a call to a TSP solver. We can instead approximate it by the distance to the node $\vertex_i$ from the last node in the route $\vertex_{i-1}$.

\begin{figure*}[!htbp]
    \centering
    \includegraphics[width=\textwidth]{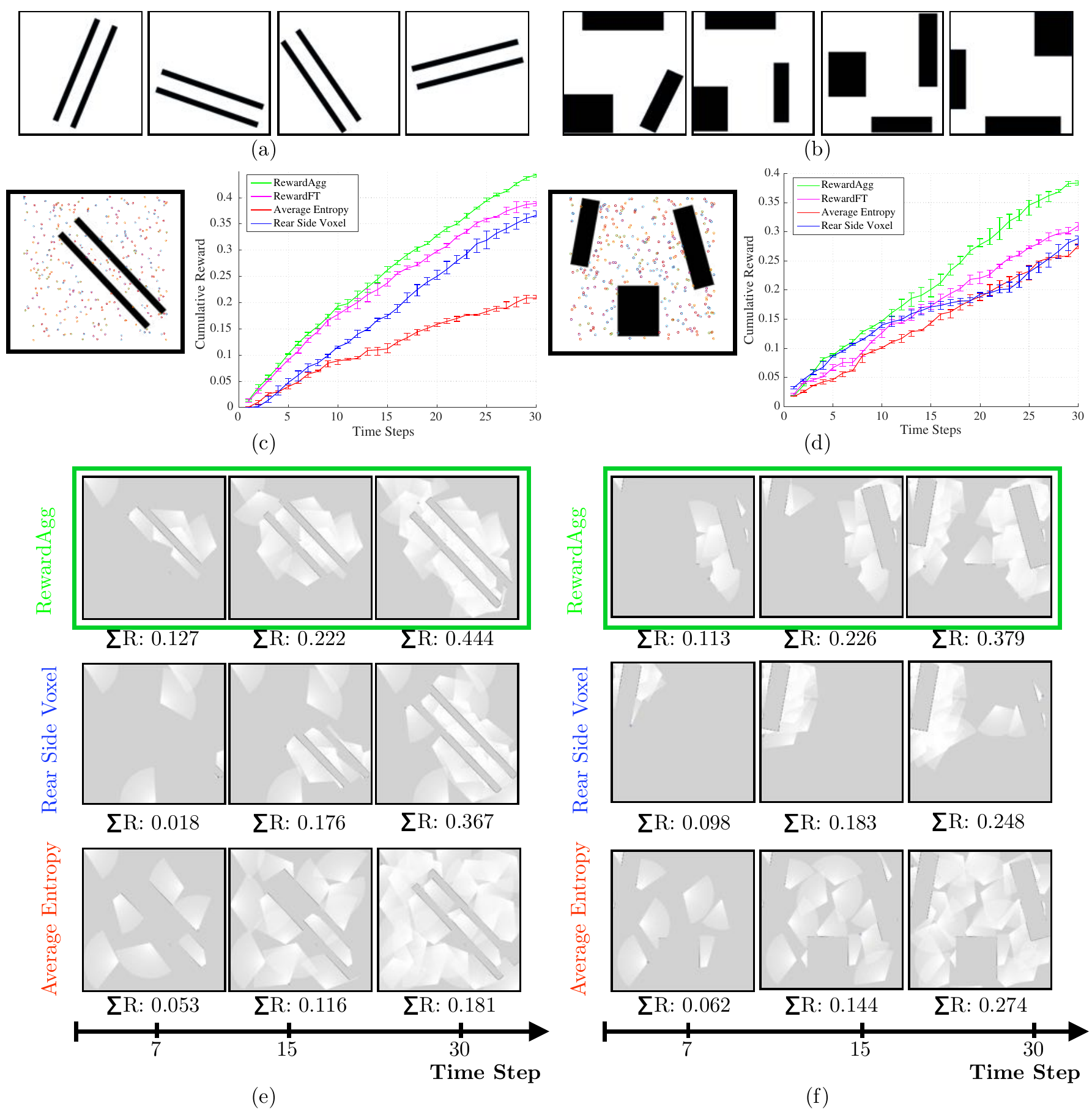}
    \caption{%
    Case study of Problem \hiddenunc using \algRewAgg, \algRewFT and baseline heuristics. Two different datasets of 2D exploration are considered - (a) dataset 1 (parallel lines) and (b) dataset 2 (distributed blocks). Problem details are: $T=30, |\actionSet|=300$, $100$ train and $100$ test maps. A sample test instance is shown along with a plot of cumulative reward with time steps for different policies is shown in (c) and (d). The error bars show $95\%$ confidence intervals. (e) and (f) show snapshots of the execution at time steps $7, 15$ and $30$. 
        \label{fig:results:matlab_unc}}
\end{figure*}%

\begin{figure*}[!htbp]
    \centering
    \includegraphics[width=\textwidth]{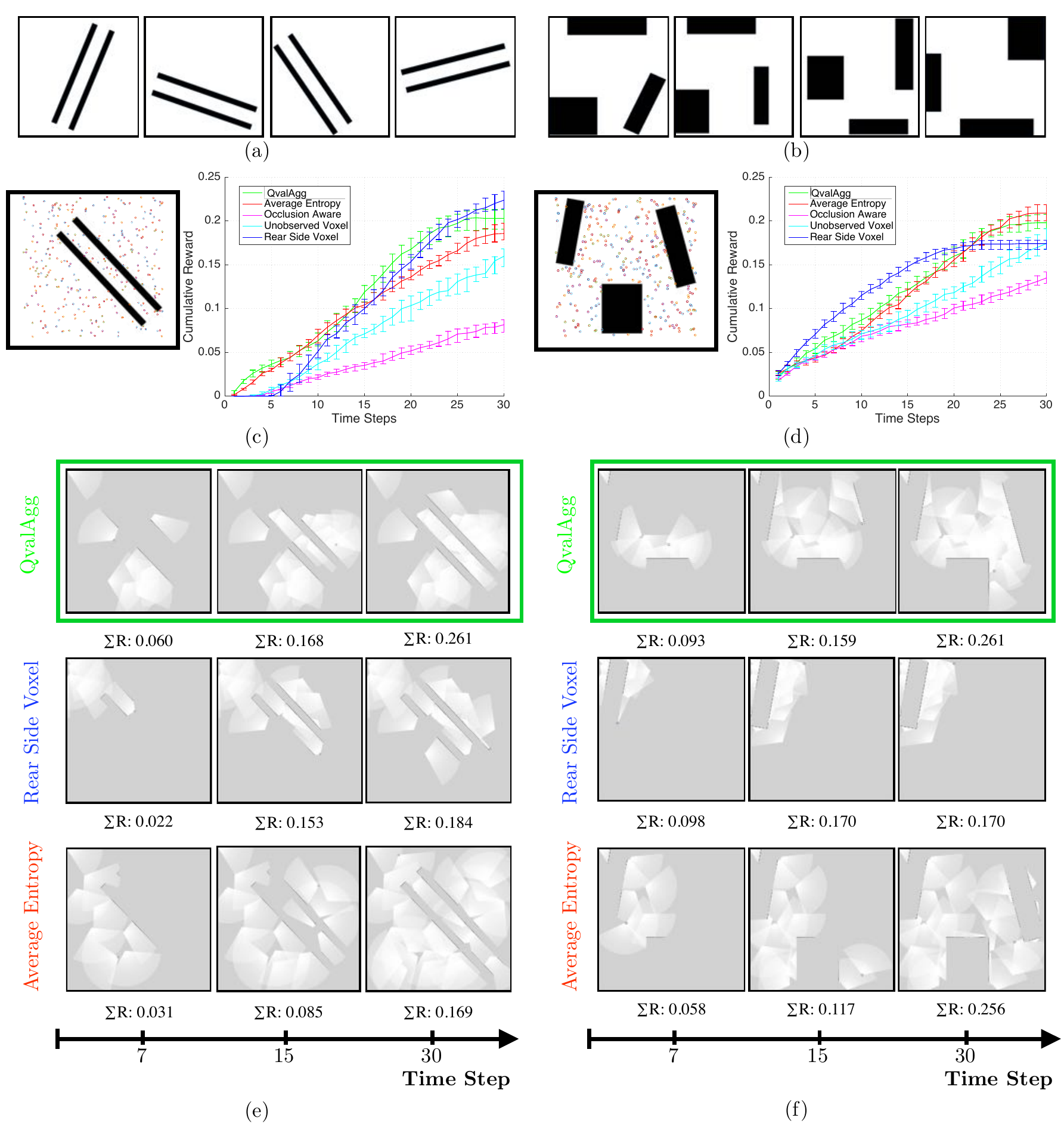}
    \caption{%
    Case study of Problem \hiddenunc using\algQvalAgg with baseline heuristics on a 2D exploration problem on 2 different datasets - dataset 1 (concentrated information) and dataset 2 (distributed information). The problem details are: $T=30, \costBudget=2500, |\actionSet|=300$, $100$ train and $100$ test maps.
    A sample test instance is shown along with a plot of cumulative reward with time steps for different policies is shown in (c) and (d)
    The error bars show $95\%$ confidence intervals
    Snapshots of execution of \algQvalAgg, \RearSideVoxel and \AverageEntropy are shown for (e) dataset 1 and (f) dataset 2. The snapshots show the evidence grid at time steps $7, 15$ and $30$. 
        \label{fig:results:matlab_con}}
\end{figure*}%

\begin{figure*}[t]
    \centering
    \includegraphics[width=\textwidth]{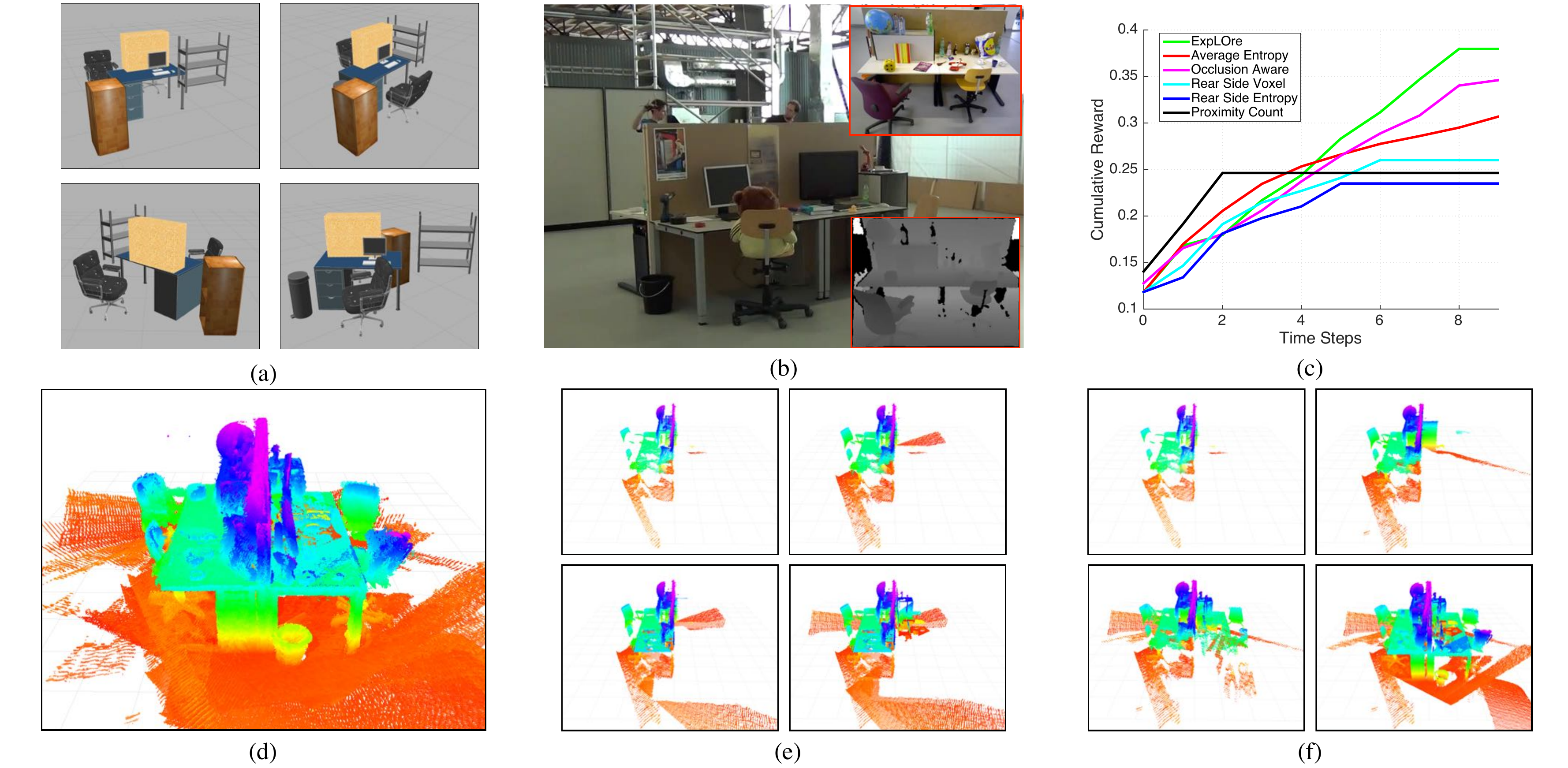}
    \caption{%
        Comparison of \algQvalAgg with baseline heuristics on a 3D exploration problem where training is done on simulated world maps and testing is done on a real dataset of an office workspace. The problem details are: $T=10$, $\costBudget=12$, $|\actionSet|=50$. 
    (a) Samples from $100$ simulated worlds resembling an office workspace created in Gazebo.
    (b) Real dataset collected by \cite{sturm12iros} using a RGBD camera. 
    (c) Plot of cumulative reward with time steps for \algQvalAgg and baseline heuristics on the real dataset.
    (d) The 3D model of the real office workspace formed by cumulating measurements from all poses. 
    (e) Snapshots of execution of \OcclusionAware heuristic at time steps $1,3, 5, 9$.
    (f) Snapshots of execution of \algQvalAgg heuristic at time steps $1,3, 5, 9$. 
    \label{fig:results:cpp}}
\end{figure*}%

\begin{figure*}[t]
    \centering
    \includegraphics[width=\textwidth]{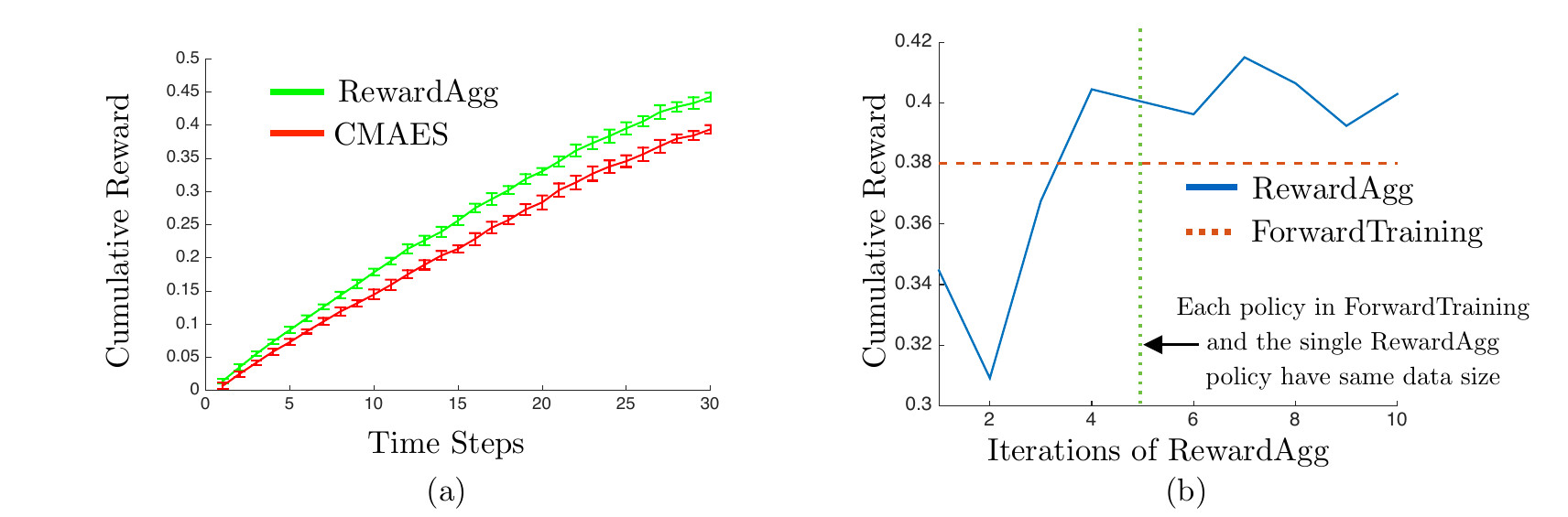}
    \caption{%
    (a) Comparison of \algRewAgg with CEM policy search. Both algorithms are given access to the same amount of data. The final policy from CEM and the best validation policy of \algRewAgg are then executed on a test dataset. \algRewAgg outperforms CEM not only overall but pointwise at each timestep. 
    (b) Comparison of \algRewAgg with \FT. Each policy in \FT is trained with a dataset size of $500$. \algRewAgg is trained with $100$ samples per iteration for 10 iteration. The performance of both policies on test dataset is shown.  \algRewAgg surpasses \FT at the $4^{th}$ iteration and never drops below. At iteration 5 the single policy of \algRewAgg has the same dataset size as each policy of the $10$ policies of \FT. However the single policy still outperforms the nonstationary policy.
    \label{fig:result:agg}}
\end{figure*}%

\subsection{Analysis of Results}
Fig.~\ref{fig:results:extra} shows the utility of all algorithms on various synthetic datasets. The two numbers are lower and upper $95\%$ confidence intervals of the episodic utility of each algorithm. The best performance on each dataset is highlighted. For Problem \hiddenunc, \algRewAgg is employed along with baseline heuristics. For Problem \hiddencon, \algQvalAgg is employed with baseline heuristic augmented with motion penalization. The train size is $100$ and test size is $10$. We present a set of observations to interpret these results.

\begin{observation}
The learnt policy from \algRewAgg / \algQvalAgg has a consistently competitive performance across all datasets.
\end{observation} 

Fig.~\ref{fig:results:extra} shows the performance of all algorithms on a set of 2D and 3D datasets. We see that out of the $10$ datasets, the learners perform better than any heuristic on $8$. On $2$ of the datasets, the \AverageEntropy heuristic outperforms the learner by a small margin. On examining the datasets, we see that the unknown space exploration behaviour of \AverageEntropy results in good performance in environments that either lack spatial correlation or contain objects distributed in the environment.

\begin{observation}
The performance of heuristics vary widely across datasets, however, the performance of the learner is robust.
\end{observation} 

We can see that the relative ranking of \AverageEntropy and \RearSideVoxel interchanges from Dataset 1 to 2. This motivates the need for adaptive policies that assign different utility to unknown cells conditioned on the environment in which the robot is operating. The learner's policy on the other hand adapts to different environments and hence maintains a consistently good performance. Interestingly, it also outperforms the heuristic pointwise across datasets, which is indicative of the fact that the adaptation happens during exploration as well. 

\begin{observation}
The performance margin of \algRewAgg in Problem \hiddenunc as compared to heuristics is much larger than that of \algQvalAgg in Problem \hiddencon
\end{observation} 

This is seen to be especially true in Dataset 1, 2 and 4. As conjectured in Section~\ref{sec:approach:ipp:one_step_reward}, this can be attributed to two reasons. Firstly, the near-optimality guarantee in Theorem~\ref{theorem:hidden_unc} of imitating a Clairvoyant one-step-reward bounds the performance of the learner. Secondly, the empirical regression regret of imitating one step reward values will be much lower than trying to estimate the action values using features from the history $\belief_t$, i.e. it is easier to predict the immediate utility of going to a sensing location than trying to predict the future utility.

\begin{observation}
The performance of \AverageEntropy in the Poisson Forest dataset is at par with the learner.
\end{observation} 

The Poisson Forest dataset is created by sampling circles in the environment from a spatial Poisson distribution where the density of the forest is specified. The lack of spatial correlation, implies it is equally likely to find objects anywhere in the world - an assumption that \AverageEntropy optimizes.  

\subsection{Case study A: Adaptation to Different Environments}
We created a set of 2D exploration problems to gain a better understanding of the learnt policies and baseline heuristics. We did this both for Problem \hiddenunc (Fig.~\ref{fig:results:matlab_unc}) and \hiddencon (Fig.~\ref{fig:results:matlab_con}). The dataset comprises of 2D binary world maps, uniformly distributed nodes and a simulated laser. The problem details are $T=30$ and $|\actionSet|=300$. The cost budget for \hiddencon is $\costBudget = 2500$. The train size is $100$, test size is $100$. \algRewAgg and \algQvalAgg is executed for $10$ iterations.

\subsubsection{Dataset 1: Parallel Lines}
We first examined Problem \hiddenunc.
Fig.~\ref{fig:results:matlab_unc} (a) shows a dataset created by applying random affine transformations to a pair of parallel lines. 
This dataset is representative of information being concentrated in an area in the environment, e.g. powerline inspection. 
Fig.~\ref{fig:results:matlab_unc} (c) shows a comparison of \algRewAgg, \algRewFT with baseline heuristics. While \RearSideVoxel outperforms \AverageEntropy, \algRewAgg outperforms both.
Fig.~\ref{fig:results:matlab_unc} (e) shows progress of each. \AverageEntropy explores the whole world without focusing, \RearSideVoxel exploits early while \algRewAgg trades off exploration and exploitation.

The same trend can be observed in Problem \hiddencon. 
Fig.~\ref{fig:results:matlab_con} (c) shows a comparison of \algQvalAgg with baseline heuristics. The heuristic \RearSideVoxel performs the best, while \algQvalAgg is able to match the heuristic. 
Fig.~\ref{fig:results:matlab_con} (e) shows progress of \algQvalAgg along with two relevant heuristics - \RearSideVoxel and \AverageEntropy. \RearSideVoxel takes small steps focusing on exploiting viewpoints along the already observed area. \AverageEntropy aggressively visits the unexplored area which is mainly free space. \algQvalAgg initially explores the world but on seeing parts of the lines reverts to exploiting the area around it.

\subsubsection{Dataset 2: Distributed Blocks}
We first examined Problem \hiddenunc.
Fig.~\ref{fig:results:matlab_unc} (b) shows a dataset created by randomly distributing rectangular blocks around the periphery of the map.
This dataset is representative of information being distributed around.
Fig.~\ref{fig:results:matlab_unc} (d) shows that \RearSideVoxel saturates early, \AverageEntropy eventually overtaking it while \algRewAgg outperforms all.
Fig.~\ref{fig:results:matlab_unc} (f) shows that \RearSideVoxel gets stuck exploiting an island of information. \AverageEntropy takes broader sweeps of the area thus gaining more information about the world. \algQvalAgg shows a non-trivial behavior exploiting one island before moving to another.

The same trend can be observed in Problem \hiddencon. 
Fig.~\ref{fig:results:matlab_con} (d) shows that the heuristic \AverageEntropy performs the best, while \algQvalAgg is able to match the heuristic. \RearSideVoxel saturates early on and performs worse. 
Fig.~\ref{fig:results:matlab_con} (f) shows a similar trend as Fig.~\ref{fig:results:matlab_unc} (f).

\subsection{Case study B: Train on Synthetic, Test on Real}
To show the practical impact of our framework, we show a scenario where a policy is trained on synthetic data and tested on a real dataset. 
Fig.~\ref{fig:results:cpp} (a) shows some sample worlds created in Gazebo to represent an office desk environment on which \algQvalAgg is trained. 
Fig.~\ref{fig:results:cpp} (b) shows a dataset of an office desk collected by TUM Computer Vision Group \cite{sturm12iros}. The dataset is parsed to create a pair of pose and registered point cloud which can then be used to evaluate different algorithms.
Fig.~\ref{fig:results:cpp} (c) shows that \algQvalAgg outperforms all heuristics. 
Fig.~\ref{fig:results:cpp} (f) shows how \algQvalAgg learns a desk exploring policy by circumnavigating around the desk. This shows the powerful generalization capabilities of the approach. In contrast, the best heuristic \OcclusionAware gets stuck in a local minima(Fig.~\ref{fig:results:cpp} (e))

\subsection{Case study C: Policy Search vs Imitation Learning}
We compared our approach to a baseline approach of policy search. We picked the problem setting \hiddenunc, the dataset `Concentrated Parallel Lines' and the trained policy using \algRewAgg. We created a parametrized policy which was linear on the space of the information gain heuristics. The policy, parameterized by $\theta \in \real^6$, assigns at time $t$ to each vertex $v$, picks the action with the highest score as follows
\begin{equation*}
  \argmaxprob{\vertex_t \in \vertexSet} \theta^T  \infoGainVertFn{}{\vertex_t}
\end{equation*}
We train such a policy using a black-box sample efficient policy search method, Covariance Matrix Adaptation Evolution Strategy (CMAES)~\cite{hansen2016cma}. CMAES is allowed $1000$ roll-outs, the same number of calls to oracle as \algRewAgg (Note that CMAES actually has access to more information as they are full rollouts compared to single reward calls in \algRewAgg). 
Fig.~\ref{fig:result:agg}(a) shows comparison between the final policy trained by CMAES and the best policy on validation trained by \algRewAgg on a held out test dataset. We see that \algRewAgg outperforms CMAES not only on the cumulative reward by also at each time step. This confirms our hypothesis that model free policy improvement is slow to converge on account of sample inefficiency. It should be noted that the CMAES policy outperforms all the baseline heuristics as expected. 

\subsection{Case study D: \FT vs \aggrevate}
We compared the training framework of \FT, which trains a different policy for every time-step with \aggrevate that trains a single policy across all time steps. We wished to examine the following question - `How much data does a the single \aggrevate policy need to be competitive with \FT ?'. We picked the problem setting \hiddenunc and the dataset `Concentrated Parallel Lines'.
We trained \FT where each policy $\policy_t$ is given $500$ datapoints (hence for episode length $T=30$, a total of $15,000$ datapoints are used). We trained \algRewAgg where each iteration has $100$ datapoints, and the the number of iterations is $10$. Hence the \algRewAgg policy matches the same datasize as \FT at iteration $5$.
Fig.~\ref{fig:result:agg}(b) shows a comparison between \FT and \algRewAgg. We see that \algRewAgg outperforms \FT by iteration $4$, following which the performance converges and oscillates at values above \FT. Interestingly, at iteration $5$ \algRewAgg outperforms \FT even though each policy in \FT has access to the same dataset size as \algRewAgg. We conjecture that this might be because of the generalization effect across time-steps - \FT might be over-fitting as it reasons about timesteps individually. 

%% file: results_search_based_planning.tex

\section{Experiments on Search Based Planning}
\label{sec:res_search}

In this section, we extensively evaluate our approach on a set of search based planning problems for 2D planning on synthetic problems and more realistic 4D nonholonomic path planning problems encountered by UAVs flying at various speed regimes. We choose a wide variety of world distributions ranging from simple and intuitive environments, chosen to highlight the importance of exploiting environment structure in motion planning, to complex, heterogenous environments for analyzing scalability and robustness. We also present closed loop results on a UAV flying outdoors at high speeds.

Additionally, we have developed a simple and intuitive Python based planning pipeline to serve as a backend for the Gym environment. The planning environment exposes search as a policy and makes it easy to incorporate standard machine learning libraries~\citep{2016arXiv160502688short, DBLP:journals/corr/AbadiABBCCCDDDG16} with custom planning graphs that requires only environment images as input. We use this planning pipeline to conduct all our experiments. Source code and instructions can be found via our project page at this link: \url{https://goo.gl/YXkQAC} 

\subsection{Problem Details} 
\label{sec:res_search:problem}
We first describe the 2D navigation task. Here, the world map $\world$ is a 2D binary map. The graph $\graph$ is a discrete lattice of size $200\times200$ where each node is connected to the $8$ neighbours. The robot has to plan from bottom-left to top-right of the lattice. Note that while the grid size for these problems are small, the edge evaluation for such a graph could be arbitrarily expensive in practice. For example, consider the problem of planning 2D routes for aircrafts. It is plausible to envision that the lattice resolution is $100m$ and the $200 \times 200$ lattice covers an area of $20km$. Evaluation of each edge of such a lattice requires collision checking with other dynamically moving aircrats, no-fly-zones and risk of flying over urban areas. This implies that a real time traffic control can only search a small fraction of the lattice. 

We now describe a more realistic 4D nonholonomic path planning problem on a state lattice for problems encountered by UAVs. The term \emph{nonholonomic path planning}~\citep{laumond1998guidelines} refers to the fact that certain class of dynamical systems are constrained in the range of feasible motions the robot can execute~\citep{KelNag03}. It is a common practice to approximate UAVs moving at high speeds as curvature constrained systems with unicycle dynamics~\citep{dugar2017smooth, dugar2017kappaite,Choudhury_2014_7588}. We consider the problem of path planning for such systems by planning on a state-lattice~\citep{pivtoraiko2009differentially}. We consider two classes of UAVs : an autonomous helicopter moving at speeds of $30 m/s$ and a quadrotor (DJI M100) flying at $5 m/s$.

The autonomous helicopter has a minimum radius of $50m$ and plans on a state-lattice $\graph$ of resolution $25m$. The average degree of a node is $21$. The distance between start and goal is $600m$. The world $\world$  is represented as a 3D binary grid map and a set of 3D no-fly-zones (represented as polygons with a height range). An edge evaluation requires that every state on an edge is at a clearance distnce from all obstacles. Expansion of each node takes $~1 ms$ on average. The robot is required to plan within a time budget of $500 ms$ thus corresponding to maximum of $500$ expansions. 

The quadrotor has a minimum radius of $12.5m$ and plans on a state-lattice $\graph$ of resolution $12.5m$. The average degree of a node is $9$. The distance between start and goal is $300 m$. The world is represented as a 3D binary grid map and a set of 3D no-fly-zones. Expansion of each node takes $~1 ms$ on average. The robot is required to plan within a time budget of $1000 ms$ thus corresponding to a maximum of $1000$ expansions. 

\subsection{Baseline Approaches For Search Based Planning} 
\label{sec:res_search:baseline}

\subsubsection{Motion Planning Baselines}

For 2D navigation, we compare against greedy best-first search with 2 commonly used heuristics - the euclidean distance ($\greedyEuc$) and the manhattan distance ($\greedyMan$). We also use A* algorithm as a baseline with $\greedyEuc$ heuristic. Additionally, we compare against the MHA* algorithm~\citep{aine2016multi} which has been proven to be an effective way of combining multiple, often unrelated, heuristics providing bounds on solution quality~\citep{phillips2015efficient}. We use a simplified version which expands three different heuristics in a round-robin fashion - $\left[\greedyEuc, \greedyMan, \distObs\right]$, where $\distObs$ is the euclidean distance to closest, \emph{known} obstacle cell in $\closedObsList$.

For 4D nonholonomic planning problems, we use the Dubins distance~\citep{dubins1957curves} as a heuristic.

\subsubsection{Machine Learning Baselines}

We consider two learning baselines (a) Supervised Learning (SL) with data from roll-outs with $\policyOracle$ and (b) Reinforcement Learning using evolutionary strategies (CEM) and Q-Learning (QL) with function approximation. These methods are explained in detail in Appendix~\ref{appendix:ml_baseline_search}. 

\subsection{Imitation Learning Details}
\label{sec:res_search:learning}

\subsubsection{Feature Extraction and Learner}
\label{sec:res_search:learning:feature}

The policy maps the history $\belief$ to an action $\action$ by learning a function approximation for the action value function $\hat{Q}(\belief, \action)$. The tuple $\left(\action, \belief \right)$ is mapped to a vector of features $\featureVec$. Here the history $\belief$ is represented as a concatenation of all lists, i.e $\belief_t = \{ \openList, \closedList, \closedObsList \}$. The action $\action$ is the vertex $\vertex$ to expand.

We now describe the feature extraction for 2D navigation problems. Although technically, the features for a vertex $\vertex$ should depend on the parent edge $\edge$ that leads to the vertex, we ignore this in practice and consider a vertex in isolation to calculate features. 
It is important to note that the features used must be easy to calculate (no high computational burden) and should only require information uncovered by search until that point in time(else it would count as extra expansions). We define the feature vector to be a concatenation of the two vectors i.e, $\featureVec = \left[\featureVec_{S}, \featureVec_{E}\right]$.
\emph{Search Based Features:} $\featureVec_{S}\pair{\vertex}{\belief}$. These features depend on the state of the search only and does not probe the environment
\begin{enumerate}
	\item[-] $(x_{\vertex}, y_{\vertex})$ - location of node in coordinate axis of occupancy map. 
	\item[-] $(x_{\goal}, y_{\goal})$ - location of goal in coordinate axis of occupancy map.
	\item[-] $\gVal$ - cost(number of expansions) of shortest path to start.
	\item[-] $\greedyEuc$ - Euclidean distance to goal.
	\item[-] $\greedyMan$ - Euclidean distance to goal.
	\item[-] $\treeDepth$ - Depth of $\vertex$ in the search tree so far.
\end{enumerate}
\emph{Environment Based Features:} $\featureVec_{E}\pair{\vertex}{\state}$. These features depend upon the environment uncovered so far, more specifically the vertices in $\closedObsList$. \\
\begin{enumerate}
	\item[-] $x_{OBS}, y_{OBS}, d_{OBSX}$ - coordinates and distance of closest node in $\closedObsList$ to $\vertex$
	\item[-] $x_{OBSX}, y_{OBSX}, d_{OBSX}$ - coordinates and distance of closest node in $\closedObsList$ to $\vertex$ in terms of  x-coordinate. 
	\item[-] $x_{OBSY}, y_{OBSY}, d_{OBSY}$ - coordinates and distance of closest node in $\closedObsList$ to $\vertex$ in terms of y-coordinate
\end{enumerate} 
We discuss more about alternate representations and feature extraction ideas in Appendix~\ref{appendix:representation_search}.

For the 4D planning problems, we use a slightly altered feature representation which is $8$-dimensional.
\begin{enumerate}
	\item[-] Normalized Euclidean distance to start. 
	\item[-] Normalized Euclidean distance to goal.
	\item[-] Dot product between start, vertex and goal.
	\item[-] Normalized Eubins distance to start. 
	\item[-] Normalized Eubins distance to goal. 
	\item[-] Normalized heading of vertex.
	\item[-] Normalized distance of vertex from closest obstacle.
	\item[-] Dot product between distance to obstacle and heading of vertex. 
\end{enumerate}
Such a feature representation is chosen as these terms are easy to compute and are informative in estimating the utility of expanding a vertex.

The learner is represented using a feed-forward neural network with two fully connected hidden layers containing [100, 50] units and ReLu activation. The model takes as input a feature vector $\featureVec \in \featureSpace$ for the pair $(\vertex, \state)$ and outputs a scalar cost-to-go estimate. The network is optimized using RMSProp \cite{rmsprop}. A mini-batch size of 64 and a base learning rate of 0.01 is used. The network architecture and hyper-parameters are kept constant across all environments. For experiments with the UAV, we use a random forest regression~\citep{liaw2002classification}.

\begin{figure*}[!htbp]
	\centering
	\includegraphics[width=\textwidth]{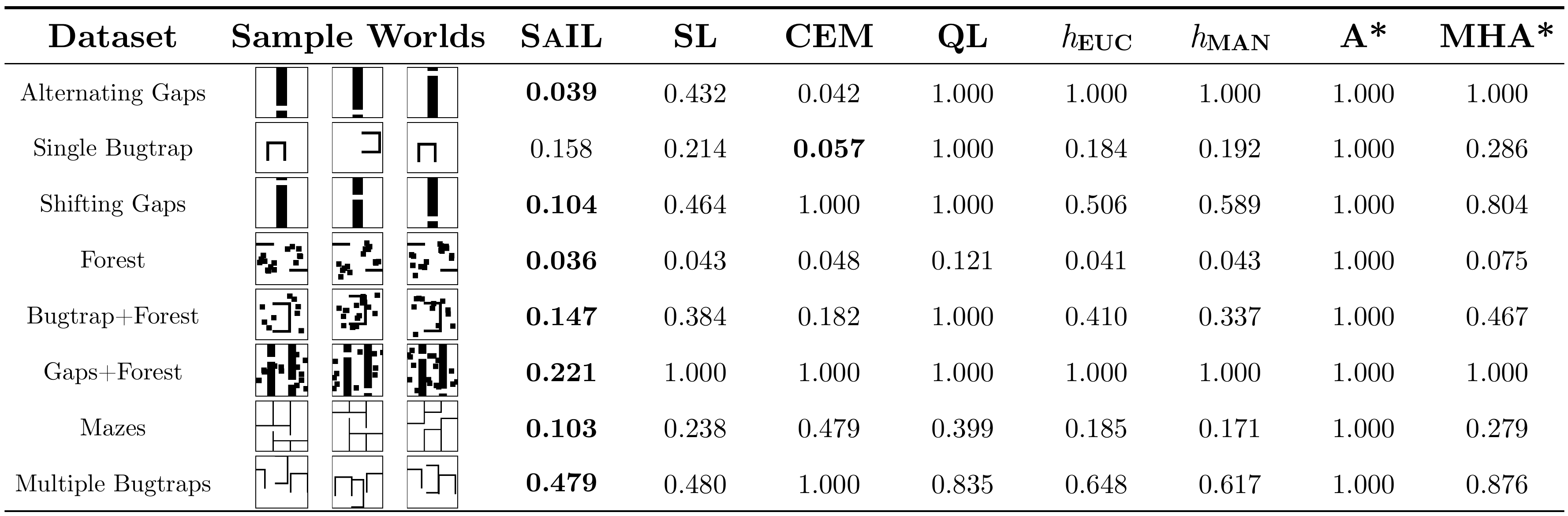}
	
	\caption{\label{fig:results_table} 
		Normalized cost of baselines on different environments (best in bold). 
		The cost corresponds to average expansions on a test set of planning problems normalized between [200, 5000] (max possible: 40000).
		Planning parameters are - map size: $200\times200$,$\planTime_{train}=1100$, $\planTime_{test}=20000$. 
		Data sizes are: train($200$), test($100$), validation($70$). 
		\algName parameters are - $\dataPointsParam: 50, \beta_{0} = 0.7$.
		\algName, CEM and QL are run for $N: 15$ iterations.
		SL uses $m:600$. }
\end{figure*}%

\begin{figure*}[!htbp]
	\centering
	\includegraphics[width=\textwidth]{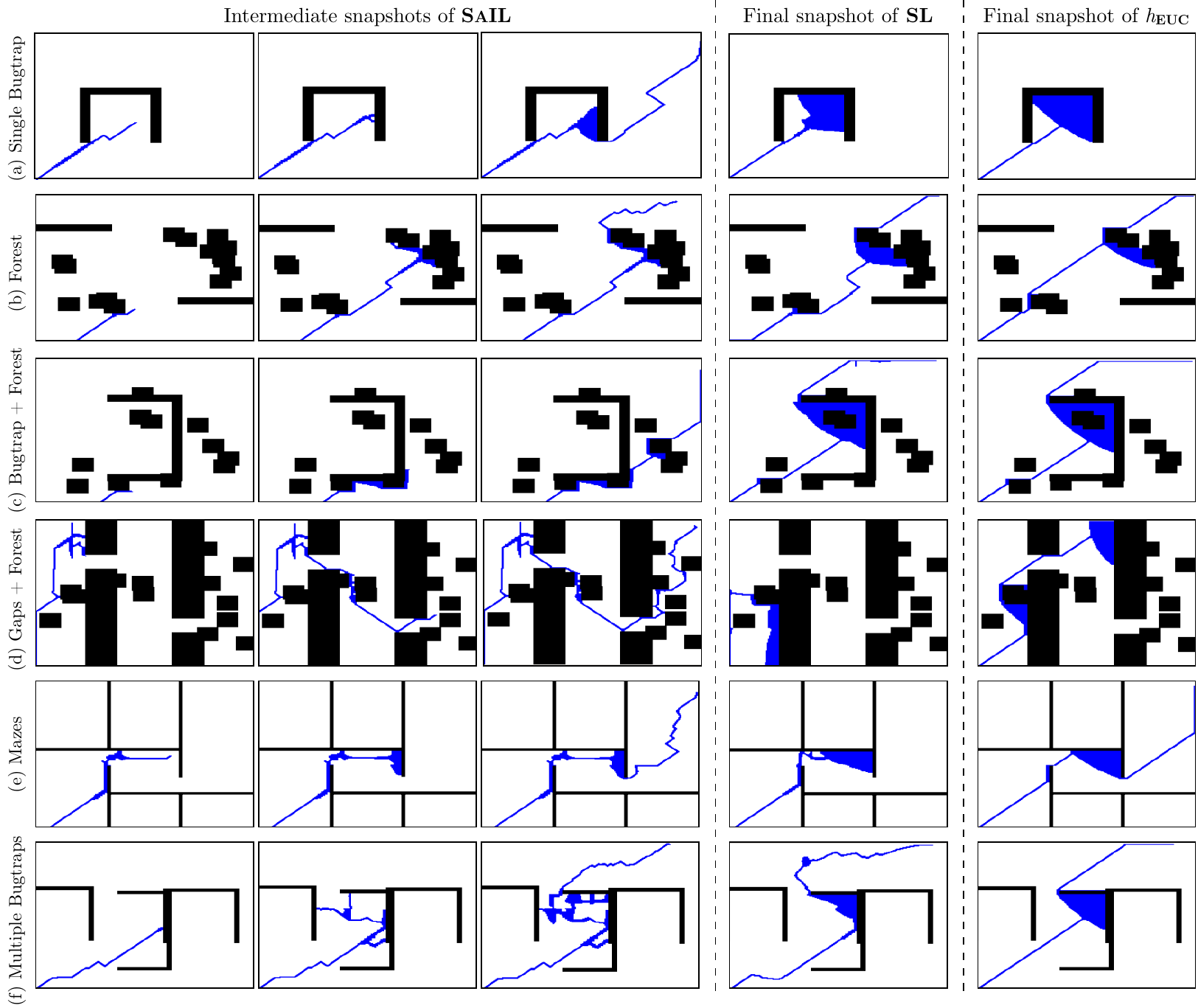}
	\caption{\label{fig:benchmark_results}
	Evolution of search frontier (expanded(blue), invalid(black), unexpanded(white)) of \algSail compared with final snapshot of supervised learning (SL) and $\greedyEuc$ across all environments. \algSail expands far less states.}
\end{figure*}%

\begin{figure*}[!htbp]
	\centering
	\includegraphics[width=\textwidth]{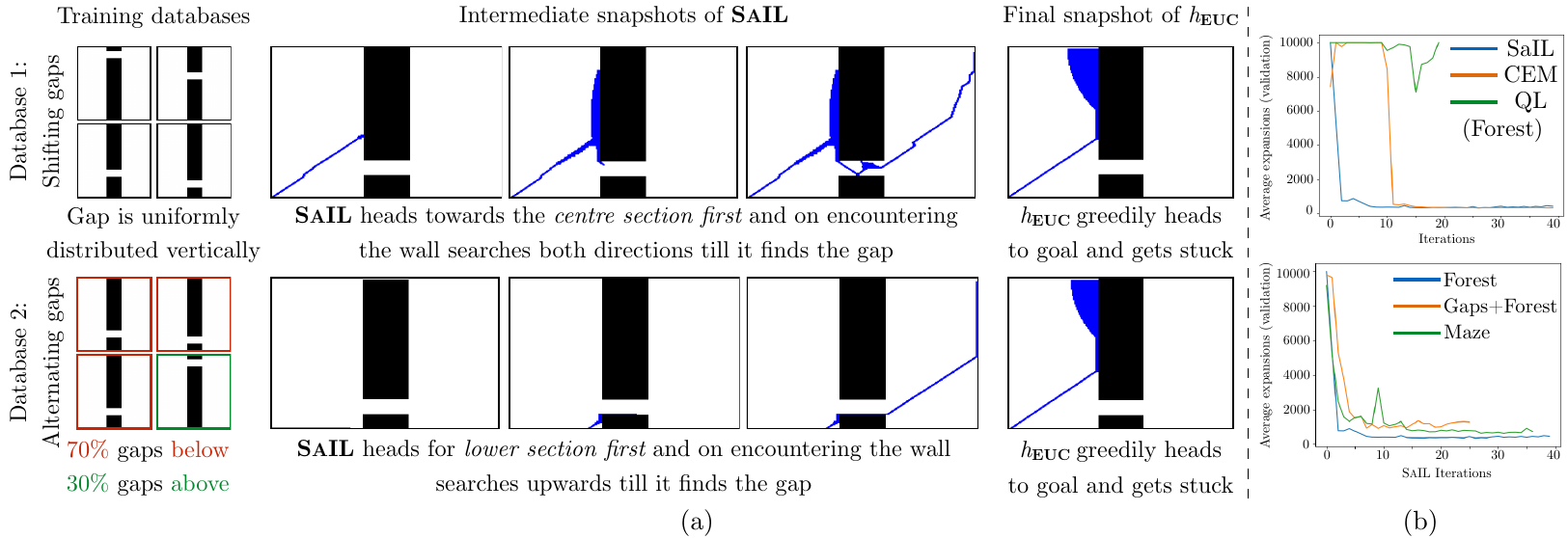}
	\caption{\label{fig:explanatory_result}
	(a) \algSail learns to adapt to different environment distributions by directing search to areas where it expects to find gaps.
	 Note \algSail does not have information about the entire environment, only the explored part. 
	 (b) On the `Forest' dataset, \algSail converges faster that CEM and QL to a good policy. \algSail also converges consistently to a good policy across environments `Gaps', `Gaps+Forest', `Maze.'}
\end{figure*}%

\subsection{Case Study B: Helicopter Path Planning}
\begin{figure*}[!htbp]
	\centering
	\includegraphics[page = 1, width=\textwidth]{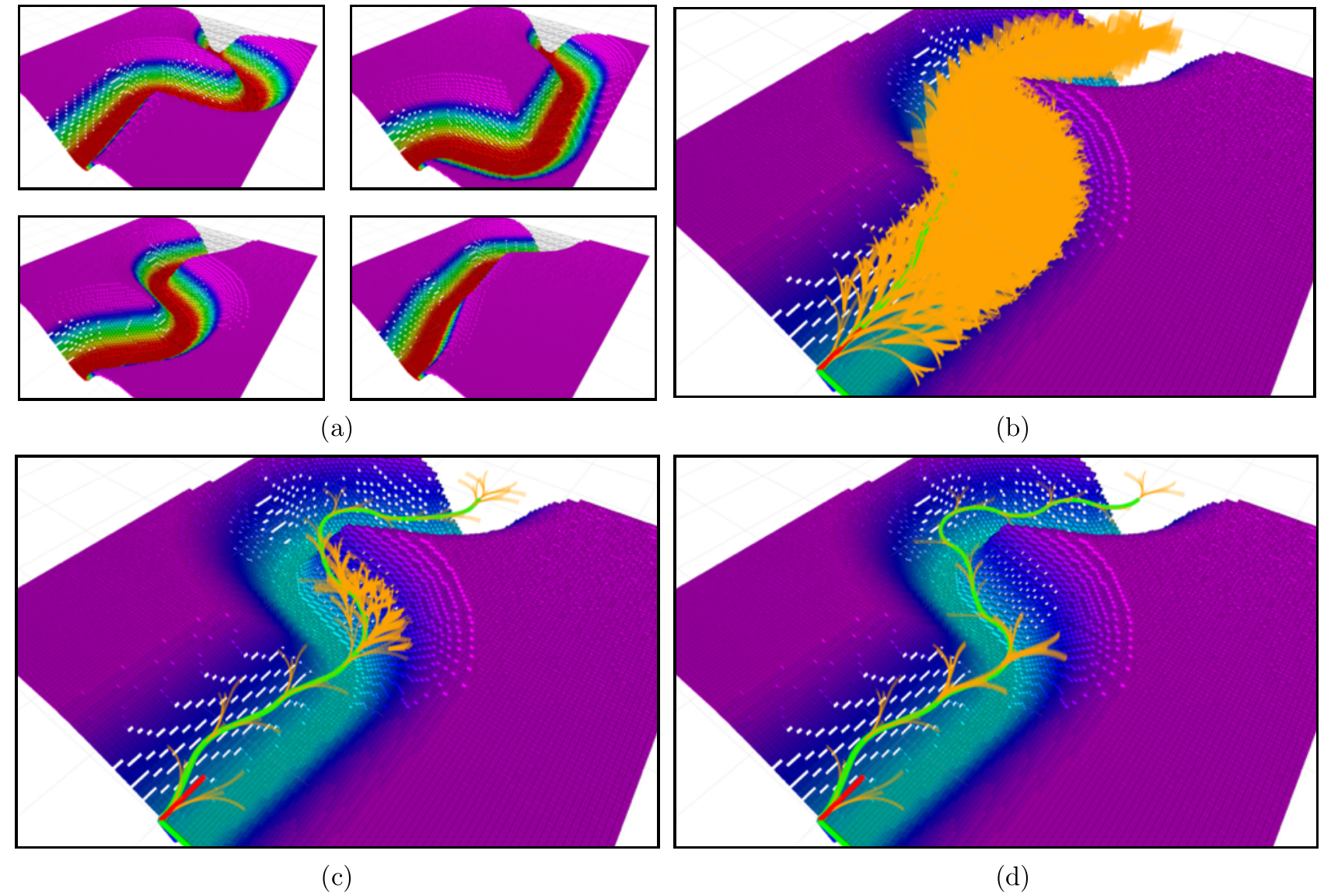}
	\caption{\label{fig:helicopter_results}
	Experiments on path planning for an autonomous helicopter in a canyon environment. The environment is motivated from planning challenges as described in~\cite{Choudhury_2014_7588}. (a) Dataset of canyon-like environments generated by a parametric distribution. (b) The search tree from A* with Dubins distance heuristic on a test environment. The start point is shown by the axes. The expanded edges are shown in yellow. The planned path is shown in green. A* expands $2531$ vertices and takes $7000 ms$. (c) The search tree for greedy search with Dubins distance heuristic. It expands $142$ vertices and takes $500 ms$. Note that most of the wasted expansions are where the tree tries to search through the canyon wall (d) \algSail  expands only $18$ vertices and takes $100 ms$. It hugs the canyon wall till it reaches the goal.
	}
\end{figure*}%

\begin{figure*}[!htbp]
	\centering
	\includegraphics[width=\textwidth]{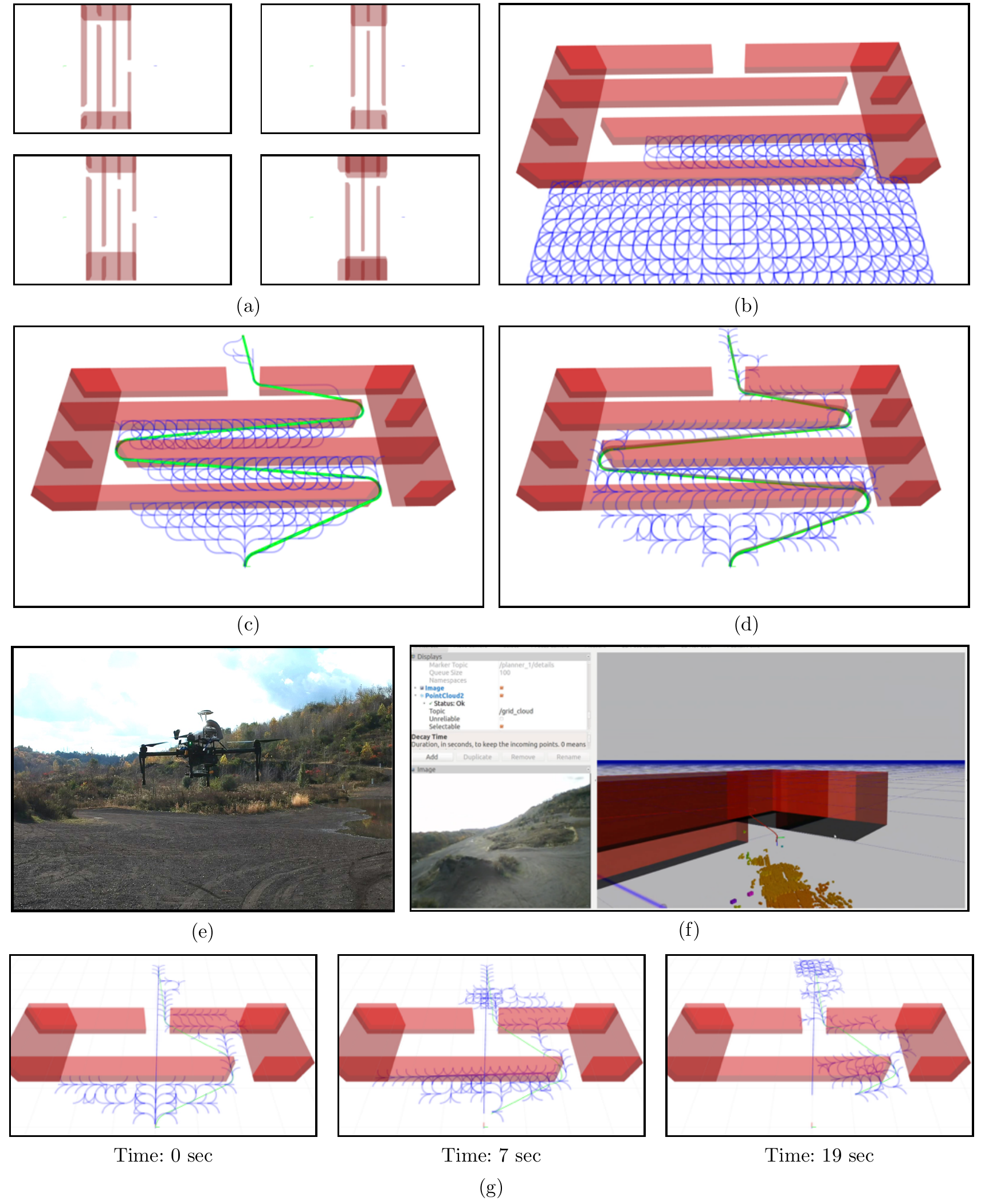}
	\caption{\label{fig:uav_results_sail}
	Experiments on path planning for a real quadrotor flying at high speed $5 m/s$ while avoiding no fly zones that represent a maze like scenario. (a) A dataset of mazes created from a parametric distribution (b) The search graph of A* on the environment. It expands $1910$ states in the $1000 ms$ time budget without finding a path (c) The greedy search with Dubins distance expands $661$ vertices and takes  $400 ms$. The remaining time is used to relax the path shown in green. (d) \algSail outperforms both and finds a path by expanding only $180$ vertices in $120 ms$. (e) The DJI M100 used for our experiments (f) An experiment where \algSail is running onboard the robot. A set of no fly zones is created and the robot has to fly through it. The robot view and onboard imagery is shown (g) A time lapse of the search tree as the robot replans while performing the mission. We can see that the search tree remains sparse through out and \algSail is always able to find a path.
	}
\end{figure*}%

\subsubsection{Dataset Creation}
The 2D world maps are created by randomly distributing geometric objects such as rectangles and circles according to hand design parametric distribution. Each environment class is representative of challenging artifacts in motion planning such as narrow corridors, local minimas, single homotopies. Hetergenous environments are created to show that the heuristic can deal with such problems as well. 

For the experiment with a real robot, a dataset of mazes was created and a real life maze was simulated using no-fly-zones.

\subsubsection{Clairoyant Oracle}
We use the backward Djikstra algorithm as the clairvoyant oracle. It is executed till it expands to all states, or till it reaches a cost-to-go limit. We note that using such an oracle for higher dimensions might be infeasible in higher dimensions and discuss remedies in Section~\ref{sec:discussion}.

\subsubsection{Practical Algorithm Implementation}
Since the size of the action space changes as more states are expanded, the \algSail algorithm requires a forward pass through the model for every action individually unlike the usual practice of using a network that outputs cost-to-go estimate for all actions in one pass as in \cite{mnih-dqn-2015}. This can get computationally demanding as the search progresses ($\bigo{N}$  in actions). Instead, we use a \emph{priority queue} as $\openList$ which sorts vertices in increasing order of the cost-to-go estimates as is usually done in search based motion planning. The vertex on the top of the list is then expanded. We use two priority queues, sorted by the learner and oracle's cost-to-go estimates respectively. This allows us to take actions in $\bigo{1}$ but forces us to freeze the \emph{Q}-value for a vertex to whenever it is inserted in $\openList$. Despite this artificial restriction over the policy class $\policySet$, we are able to learn efficient policies in practice. However, we wish to relax this requirement in future work. We also analyze the time complexity in Appendix~\ref{appendix:sail_complexity}.

\subsection{Analysis of Results}

Fig.~\ref{fig:results_table} shows the normalized evaluation cost of all algorithms on various datasets. 
Snapshots of planning with different heuristics are shown in Fig.~\ref{fig:benchmark_results} and Fig.~\ref{fig:explanatory_result} (a). Convergence of different learning algorithms are shown in Fig.~\ref{fig:explanatory_result} (b). 
We present a set of key observations to summarize these results.

\begin{observation}
\algSail has a consistently competitive performance across all datasets.
\end{observation}

Fig.~\ref{fig:results_table} shows that \algSail learns a better search policy than any other baseline across all but one environments. It maintains performance from homogenous to heterogenous environments. 

\begin{observation}
\algSail has faster convergence than all learning baselines.
\end{observation}

Fig.~\ref{fig:explanatory_result} (b) shows that on the `Forest' dataset, \algSail converges by $6^\text{th}$ iteration, while CEM takes $12$ and QL does not converge. \algSail also converges quickly (by the $8^\text{th}$ iteration) across datasets.

\begin{observation}
\algSail is able to detect and escape local minima.
\end{observation}

A classic case in motion planning is the bugtrap (Fig.~\ref{fig:marquee} (b) ) which traps a greedy search in a local minimum. Fig.~\ref{fig:benchmark_results} (a) and Fig.~\ref{fig:benchmark_results} (f) shows that when trained on such distributions, \algSail is able to detect these artifacts and smartly escape them by exploring in different directions.

\begin{observation}
	\algSail is able to exploit the relative configuration of obstacles and environment structure. 
\end{observation}

In a maze world with rectilinear hallways (Fig.~\ref{fig:benchmark_results} (e)), \algSail learns to quickly find a wall and then concentrate the search along the axes. In Fig.~\ref{fig:benchmark_results} (d), \algSail focuses only on regions where there is a high probability of a gap and skids along obstacles otherwise. 

\subsection{Case Study A: Adaptive behaviour of \algSail}

We take a closer look at the behaviour of \algSail in response to a change in the distribution of worlds that it is being trained on $P(\world)$. Consider the scenario illustrated in Fig.~\ref{fig:explanatory_result} (a). We create two datasets. Both datasets have a wall in the middle of the environment, with a gap in the wall. For dataset $1$, the gap can occur uniformly randomly along the y-axis. For dataset $2$, the gap either occurs with $70\%$ probability at the bottom and $30\%$ probability at the top. 

For dataset $1$, \algSail learns to approach the centre of the environment first and then search along the wall till it finds a gap. This is in response to the fact that the gap can occur anywhere and hence this is a cost efficient strategy. Contrast this to a greedy search that get stuck expanding states near the top of the wall.

For dataset $2$, \algSail learns to approach the bottom of the environment first and then search along the wall. This is in response to the gaps occuring at the bottom of the wall. The greedy search is non responsive to the change in distribution and gets stuck expanding states near the top again. 

An important application of heuristic learning is to speed up high dimension search. An application of particular relevance to us is an autonomous helicopter~\cite{Choudhury_2014_7588}. A class of environment in which the helicopter has to plan in is a canyon like environment. Since the system moves at a speed of $30 m/s$, it has to produce a plan in real-time (within $200 ms$) otherwise it risks reaching states from which collision is inevitable. 

We use \algSail to learn a heuristic that guides search in such environments. We collect a dataset by generating canyons using a parametric distribution as showing in Fig.~\ref{fig:helicopter_results} (a). A lattice, with the specifications described in Section~\ref{sec:res_search:problem} is created. As a baseline, we run A* with Dubins distance as the heuristic on this problem. As shown in Fig.~\ref{fig:helicopter_results} (b), this ends up expanding a large number of vertices ($2531$). This is because the Dubins distance is not the optimal cost to do. The under-estimation of this heuristic results in a large number of vertices being expanded and hence a long planning time ($7000 ms$).

We also run a greedy search using the Dubins distance as a heuristic. We see that for these kind of environments, greedy search performs pretty good - the number of vertices expanded is $142$ and planning time is $500 ms$. However, the greedy search expends search effort trying to search for a tunnel through the canyon. 

\algSail has much better performance than either of these baselines. It is able to learn a heuristic that expands only $18$ vertices with a search time of $100 ms$. The features used by \algSail are minimialistic and are enlisted in Section~\ref{sec:res_search:learning:feature}. Among those features are the Dubins distance to the goal and the direction vector to the nearest obstacle. By examining the search tree produced by \algSail, we observe that it learns a trade-off between following the Dubins distance heuristic and not expanding states that are pointing into the canyon wall (as such states would not result in a feasible path eventually).

\subsection{Case Study C: Quadrotor Path Planning}

We also applied this approach to a real quadrotor which has to navigate in an environment at high speed $5 m/s$ while avoiding no fly zones. No fly zones can result from areas that a UAV cannot fly over because of risks to property or from other vehicles in the area. These no fly zones can be arbitrary in complexity thus creating artifacts such as a maze as shown in Fig~\ref{fig:uav_results_sail}.

We create a dataset of such mazes by means of a parametric distribution as shown in Fig~\ref{fig:uav_results_sail} (a). We give a time budget of $1000 ms$ for planners to solve the problem. A* with Dubins heuristic is unable to solve the problem in the time limit as shown in Fig~\ref{fig:uav_results_sail} (b). This is because the Dubins distance vastly under-estimates the distance to the goal in this environment. A* expands $1910$ states before being terminated. 

Greedy search with Dubins heuristic is able to find a path after $661$ expansions within the time budget (in $400 ms$). The remaining time is spent relaxing the path found. The greedy behaviour is beneficial in this environment because it results in a wall following like behaviour. However the algorithm wastes search effort expanding states perpendicular to the wall which would lead to inevitable collision. 

\algSail outperforms both algorithms by finding a path in $180$ expansions (in $120 ms$). The remaining time is spent relaxing the path. As can be seen for the search graph, it focuses on expanding paths perpendicular to the wall. It learns to not expand vertices that point into the wall since the oracle shows the cost to go of such nodes to be $\infty$. 

We also evaluated \algSail on board a DJI M100 quadrotor equipped with a TX2 computer. We created a synthetic maze with no fly zones and commanded the robot to fly through it (Fig~\ref{fig:uav_results_sail} (e-f) ). \algSail is able to find a path expanding a sparse number of vertices. As the robot follows the path, the algorithm is able to consistently replan and find a path consistently without expanding too many states (Fig~\ref{fig:uav_results_sail} (g) ).

%% file: discussion.tex

\section{Discussion and Future Work}
\label{sec:discussion}

We presented a novel data-driven imitation learning framework to learning planning policies. Our approach trains a policy to imitate a clairvoyant oracle that has full information about the world and can compute optimal planning decisions. We examined two problem domains - informative path planning and search based planning. We evaluated our approach in both these domains and showed that the learnt policy can outperform state-of-the-art approaches. We now discuss a set of relevant questions and directions for future work.

\subsection{When does this framework lead to good policies? What are some failure cases?}
\label{sec:discussion:success_failures}

\begin{figure*}[t]
    \centering
    \includegraphics[width=\textwidth]{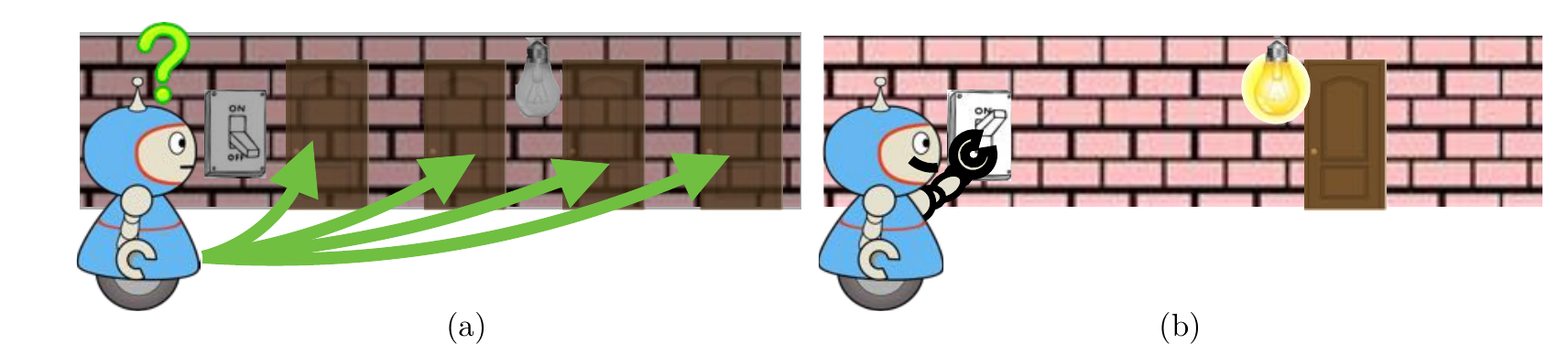}
    \caption{%
    The robot in a dark room problem. The robot is uncertain about the location of a door and the only way to collapse that uncertainty is to pull a light switch. (a) A clairvoyant oracle is not incentivized to flip the switch and hence the robot does not learn to collapse uncertainty (b) The optimal POMDP policy would be to flip the switch and then head for the door.}
    \label{fig:clairvoyant_failure_case}
\end{figure*}%

MDP framework provides an elegant way of posing problems where the complete state of the problem space is known. The value of an action for a given state in an MDP is given by equation \ref{eq:hallucinating_oracle_repeat}.
\begin{equation}
\label{eq:hallucinating_oracle_repeat}
\QFn{\pi}{t}(\state, \action) = \rewardFn{\state}{\action} + \expect{\state' \sim P(\state' \mid \state, \action)}{
\valueFn{\pi}{t-1}(\state')} 
\end{equation}
\begin{equation}
	\valueFn{\pi}{t}(\state) = \sum\limits_{i=t}^{T} \expect{\state_t \sim P(\state_{i}| \policy, i,\state)}{\rewardFn{\state_i}{\policy(\state_i)}}
\end{equation}

The optimal MDP policy maximizes the expected cumulative reward, i.e $\policy^*(\state_t) \in \argmaxprob{\policy \in \policySet} \valueFn{\pi}{t}(\state_t)$.

However there are 2 major challenges that POMDP solvers face-
\begin{itemize}
	\item Computing the expectation over the state space. Since the state space of most of the problems worth solving is large, computing an expectation over such state space needs a large number, making it expensive to evaluate online.
	\item  Keep track of evolving uncertainty about the state space over the planning horizon.  
\end{itemize}

Our approach solves the first challenge through data driven techniques - the MDP solvers are used over sampled MDP problems to train a policy on the expected distribution of problems. The hallucinating oracle is similar in nature to a QMDP algorithm~\citep{Littman95learningpolicies}, an effective approximate solution to POMDPs, which takes the best action on the current posterior. However, while QMDP requires maintaining an explicit posterior, our framework does not. QMDP has been shown to be very successful where explicit information gathering behaviour is not required~\citep{Koval-RSS-14,javdani2015shared} - the belief collapses irrespective of the action. 
Hence this optimization assumes a fixed belief and does not account for evolving belief over time, (which is challenge 2 for POMDP's). This implies there is no motivation for the MDP solver and hence the learnt policy to change the belief.

These kind of methods work quite well in POMDP problems where the required changes in belief can be attained by actions that are rewarding as well. This is very apt in the problem we address - as the set of actions are constrained to candidate nodes in the open list, no single action is very informative. It suffices to expand the best node under the current belief and continue to update the belief as the open list evolves. And there exists no action that is not rewarding while reducing the uncertainty. We note that this is not true for all learning in planning paradigms. For example, when learning to collision check~\citep{choudhury2017active}, a policy that actively reduces uncertainty about the world is effective.  

To illustrate the failure case, we present a simple scenario as shown in Fig~\ref{fig:clairvoyant_failure_case}. We have a `trapped robot' whose task is to escape from a room, i.e. it gets a reward for escaping and penalization for staying in the room. The room is dark, i.e. the robot cannot observe the location of the door. It can performa actions such as moving in the room. It can also perform an action of flicking on the light switch. On performing such an action, it receives an observation containing the location of the door. An optimal POMDP policy would always choose this action, collapse uncertainty about the door location and subsequently head straight for the door. However, imitation of clairvoyant oracles do not provide such behaviours. The oracle, at training, always guides the robot towards the door to maximize reward and is not incentivized to flip the light switch. The policy learns a blind search pattern
which takes a long time to find the door.

For such POMDP problems, one way forward would be to incentivize the oracle at train time to reduce the uncertainty as suggested by the POMDP-lite approach~\cite{chen2016pomdp}. While POMDP-lite quantifies uncertainty reduction as L1-norm of the belief change, this can be hard to compute for the space of world maps. Using approximations to this belief change would be an interesting direction of future work.

\subsection{How can we incorporate solution cost in addition to search effort in this framework?}
\label{sec:discussion:anytime_search}

While our framework ignores the cost of a solution, we note that finding feasible solutions quickly is the core motivation of a number of high dimensional planning problems which have historically resorted to sampling based approaches~\citep{kuffner2000rrt}. 
Hence, one can apply our framework to such problems to produce potentially faster solutions.
We also note that when planning on locally connected lattices for geometric planning problems, minimizing the number of expansions generally leads to near-optimal solutions (unit cost for each valid edge). 

However, if we really cared about near optimal solutions, the framework of Multi-heuristic A* (MHA*)~\citep{aine2016multi} can be easily adopted. In such a framework, any heuristic function~\citep{narayanan2015improved} can be used in tandem with an anchored search which uses an inflated admissible heuristic. Hence we can simply replace our $\search$ function with MHA*.

The bi-objective criteria of solution cost and search effort is best reasoned about in the paradigm of \emph{anytime planning}. In this paradigm, an algorithm traces out the \emph{pareto-frontier}~\citep{choudhury2016pareto} - finds a feasible solution quickly and iteratively improves it. In this paradigm, \algSail trains a heuristic that displays a behaviour we would expect in the first iteration. A direction of future work would be to learn \emph{anytime heuristics} that minimize search effort initially to and solution cost eventually.

\subsection{Can we generalize this framework to sampling based planners?}

The \algSail framework defines $\search$ in a very general way - the underlying implicit graph can also be a tree and the expansion operation can be a local steering operation akin to the framework of EST~\cite{hsu1999placing}. The oracle design is an open question - a plausible oracle is growing a backward tree from the goal and using a k-NN value function approximator. Another paradigm to consider is when the $\expand$ operation is a call to a \emph{sampler}. For example, the framework in Randomized A* (RA*)~\cite{diankov2007randomized} proceeds by selecting a node of the search tree using some criteria and sampling around it. 

Recently, \cite{ichter2017learning} proposed a framework for learning sampling distributions from optimal paths during training by using a conditional variational auto-encoder (CVAE). However, in this framework sampling and planning are decoupled, i.e. the sampling policy learns a good stationary distribution from which samples are generated and provided to the planner. Hence the planner does not adapt during the planning cycle. Such a stationary distribution can be very hard to learn as directly predicting the optimal path requires conditioning on a lot of information about the environment. 

\algSail can be extended to learn sampling policies that address this problem. The CVAE can condition on the state of the search (similar to the feature vector used by \algSail). The labels can be obtained by a backward tree from the goal grown during training. The iterative learning process of \algSail will ensure that the CVAE is trained on the distribution of search state actually encountered rather than simply using the optimal path.

\subsection{Incorporating noise in transition and observation for IPP problems}
\label{sec:discussion:noisy}
The informative path planning problem that we defined in Section~\ref{sec:background:ipp:framework} and subsequently mapped to a POMDP in Section~\ref{sec:problem_formulation:ipp_mapping} consider a deterministic measurement and utility function. This can always be relaxed in an ad-hoc way: the occupancy map used to represent $\belief_t$ is essentially a Bayes' filter and can handle noisy observations, and the policy can also handle motion uncertainty since during the training phase, data collected in the initial stages is from random motions of the learner. 

However, if one is to formally incorporate noise, the mapping needs to be re-examined. The crucial change arises from the fact that the utility function $\utilityFnDef$ is no longer dependent only on the sequence of vertices visited $\{ \vertex_i \}_{i=1}^t$ and the world $\world$. It also depends on the actual observations received $\{ \meas_i \}_{i=1}^t$, i.e. the utility function needs to be redefined to have the following arguments $\utilityFn{\{ \vertex_i, \meas_i \}_{i=1}^t}{\world}$.

To provide a concrete example, we re-examine our application of 3D reconstruction of objects in the environment presented in Section.~\ref{sec:res_ipp:problem}. We had assumed that each vertex $\vertex_i$ in a path $\Path$ is associated with a unique measurement $\meas_i$. The union of all measurements defined the coverage map $C$ which in turn defined the utility. Since this unique measurement assumption is no longer true, the coverage map has to explicitly consider the actual measurements received. This results in a utility function $\utilityFn{\{ \vertex_i, \meas_i \}_{i=1}^t}{\world}$ that depends on measurements as well.

Keeping this important change in mind, we redefine the mapping to POMDP. 
The state is defined to contain all information that is required to define the reward, observation and transition functions. 
Let the state be the set of nodes visited \emph{and measurements received} as well as the underlying world, $\state_t = \{ \{ \vertex_i, \meas_i \}_{i=1}^t, \world\}$.
At the start of an episode, a world is sampled from a prior distribution $\world \sim P(\world)$ along with a graph $\graph \sim P(\graph)$. 
The initial state is assigned by setting $\state_1 = \{ \vertex_1, \meas_1, \world\}$.
Note that the state $\state_t$ is partially observable due to the hidden world map $\world$.

We define the action $\action_t$ to be the next \emph{desired node} to visit. The reward function is now defined as a function of the state $\state_t$ only as the marginal gain in utility on receiving $\{\vertex_t, \meas_t\}$. The marginal gain of the utility function $\utilityFnDef$ is $\marginalGain{\{ \vertex_t, \meas_t \} \mid \{ \vertex_t, \meas_t \}_{i=1}^{t-1}}{\world} = \utilityFn{ \{ \vertex_i, \meas_i \}_{i=1}^t }{\world} - \utilityFn{ \{ \vertex_i, \meas_i \}_{i=1}^{t-1} }{\world}$. 
Additionally, the reward is set to $-\infty$ whenever the cost budget is violated, i.e. $\rewardFn{\state}{\action} = $
\begin{equation}
	\begin{cases}
	\marginalGain{\{ \vertex_t, \meas_t \} \mid \{ \vertex_t, \meas_t \}_{i=1}^{t-1}}{\world} &\text{if $\costFn{\{ \vertex_i\}_{i=1}^t}{\world} \leq \costBudget $} \\
	-\infty & \text{otherwise}
	\end{cases}
\end{equation}

The state transition function, $\transFn{\state}{\action}{\state'}$, is defined by the execution model $P(\vertex_{t+1} | \action_t, \world)$ and the measurement model $P(\meas_{t+1} | \vertex_{t+1}, \world)$. Given state $\state_t$, the observation is now deterministic because it is contained in the state, i.e. $\obs_t = \meas_t$.

%% file: proof_lemma1.tex

\section{Proof of Lemma 1}
\label{appendix:lemma_hallucinating}

\begin{lemma*}
The \textbf{offline} imitation of \textbf{clairvoyant} oracle (\ref{eq:imitateClairvoyantOracle}) is equivalent to sampling \textbf{online} a world from the posterior distribution and executing a \textbf{hallucinating} oracle as shown

\begin{equation*}
\policyLEARN = \argmax\limits_{\policy \in \policySet} \expect{
\substack{
t\sim U(1:T), \\
\belief_t \sim P(\belief | \policy, t)}}
{\QFnBel{\policyORBel}{T-t+1}(\belief_t, \policy(\belief_t))}
\end{equation*}
\end{lemma*}

\begin{proof}
We will define two loss functions on the policy. Let $\lossFnPolicyDef_1 (\policy)$ be the loss function corresponding to clairvoyant oracle, i.e.

\begin{equation}
	\lossFnPolicyDef_1 (\policy) = \expect{\substack{t\sim U(1:T), \\
\state_t, \belief_t \sim P(\state, \belief | \policy, t)}}{ \QFn{\policyOR}{T-t+1}(\state_t, \policy(\belief_t)) }
\end{equation}

Let $\lossFnPolicyDef_2 (\policy)$ be the loss function corresponding to the hallucinating oracle, i.e.
\begin{equation}
\label{eq:loss2}
	\lossFnPolicyDef_2 (\policy) = \expect{
\substack{
t\sim U(1:T), \\
\belief_t \sim P(\belief | \policy, t)}}
{\QFnBel{\policyORBel}{T-t+1}(\belief_t, \policy(\belief_t))}
\end{equation}

Substituting (\ref{eq:hallucinating_oracle}) in (\ref{eq:loss2}) we have

\begin{equation*}
\begin{aligned}
&\expect{
\substack{
t\sim U(1:T), \\
\belief_t \sim P(\belief | \policy, t)}}
{\QFnBel{\policyORBel}{T-t+1}(\belief_t, \policy(\belief_t))} \\
&= \expect{
\substack{
t\sim U(1:T), \\
\belief_t \sim P(\belief | \policy, t)}}
{ \expect{\state_t \sim P(\state_t | \belief_t)}{\QFn{\policyOR}{T-t+1}(\state_t, \policy(\belief_t))}  } \\
&= \expect{t\sim U(1:T)}
{ \sum\limits_{\belief_t, \state_t} P(\belief_t | \policy, t) P(\state_t | \belief_t) \QFn{\policyOR}{T-t+1}(\state_t, \policy(\belief_t))  } \\
&= \expect{t\sim U(1:T)}
{ \sum\limits_{\belief_t, \state_t} P(\state_t, \belief_t | \policy, t) \QFn{\policyOR}{T-t+1}(\state_t, \policy(\belief_t))  } \\
&= \expect{\substack{t\sim U(1:T), \\
\state_t, \belief_t \sim P(\state, \belief | \policy, t)}}{ \QFn{\policyOR}{T-t+1}(\state_t, \policy(\belief_t)) }
\end{aligned}
\end{equation*}

Hence $\lossFnPolicyDef_1 (\policy) = \lossFnPolicyDef_2 (\policy)$.

\end{proof}

%% file: proof_theorem1.tex

\section{Proof of Theorem 1}
\label{appendix:theorem_ft}

We begin with a statement of the \emph{performance difference lemma} that is useful to bound the change in total reward-to-go. This general result bounds the difference in performance of any two policies. 
\begin{lemma}
\label{lemma:pd}
Let $\policy$ and $\policy'$ be any two policies and denote $\VBel_t'$ and $\QBel_t'$ be the t-step value function and action value function of policy $\policy'$ respectively, then:
\begin{equation*}
\begin{aligned}
	&\valuePol{\policy} - \valuePol{\policy'} \\
	& = \sum\limits_{t=1}^T \expect{\belief_t \sim P(\belief | \policy, t)}{ \QFnBel{\policy'}{T-t+1}(\belief_t, \policy(\belief_t)) - \valueFnBel{\policy'}{T-t+1}(\belief_t) }
\end{aligned}
\end{equation*}
\end{lemma}

\begin{proof}
Let $\policy_t$ be the policy that executes $\policy$ in first t timesteps and then switches to $\policy'$ fromt $t+1$ to $T$. We then have $\valuePol{\policy} = \valuePol{\policy_T}$ and $\valuePol{\policy'}=\valuePol{\policy_0}$. Thus:
\begin{equation*}
\begin{aligned}
	&\valuePol{\policy} - \valuePol{\policy'} \\
	& = \sum\limits_{t=1}^T \left[ \valuePol{\policy_t} - \valuePol{\policy_{t-1}} \right] \\
	& = \sum\limits_{t=1}^T \expect{\belief_t \sim P(\belief | \policy, t)}{ \QFnBel{\policy'}{T-t+1}(\belief_t, \policy(\belief_t)) - \valueFnBel{\policy'}{T-t+1}(\belief_t) }
\end{aligned}
\end{equation*}
\end{proof}

We now state the theorem and the proof
\begin{theorem*}
\FT has the following guarantee
\begin{equation*}
\begin{aligned}
  \valuePol{\policyLEARN} \geq \valuePol{\policyORBel} -2 T \sqrt{\actionSet \; \errclass} + T\errhor
\end{aligned}
\end{equation*}
where $\errclass$ is the regression error of the learner, $\errhor$ is the local oracle suboptimality.
\end{theorem*}

\begin{proof}

In \FT, the distribution of history $P(\belief | \policyLEARN, t)$ is generated by the learner directly. Let the cost sensitive classification error $\errcs$ be the expected difference in action value selected by the policy and the best action, $\errcs = $
\begin{equation*}
\begin{aligned}
& \expect{\substack{t\sim U(1:T), \\
\belief_t \sim P(\belief | \policyLEARN, t)}}{\max_{\action \in \actionSet} \QFnBel{\policyORBel}{T-t+1}(\belief_t, \action) -\QFnBel{\policyORBel}{T-t+1}(\belief_t, \policyLEARN(\belief)) }
\end{aligned}
\end{equation*}

We also define the local oracle suboptimality $\errhor$ being the minimum gap between oracle value and the best action value averaged over all time-steps, i.e. $\errhor = $
\begin{equation*}
\expect{t\sim U(1:T)}{\min_{\belief_t} \left( \max_{\action \in \actionSet} \QFnBel{\policyORBel}{T-t+1}(\belief_t, \action) - \valueFnBel{\policyORBel}{T-t+1}(\belief_t) \right)}
\end{equation*}
This can be non-zero when the oracle is sub-optimal with respect to itself at any time-step. This is true in this setting as there is no guarantee that the hallucinating oracle will pick a locally optimal actions with respect to its own value function. This is true even if the clairvoyant oracle was locally optimal as in the case of search based planning.

Applying Lemma \ref{lemma:pd} with $\policy = \policyLEARN$, $\policy' = \policyORBel$, we have 
\begin{equation*}
\begin{aligned}
	&\valuePol{\policyLEARN} - \valuePol{\policyORBel} \\
	& = \sum\limits_{t=1}^T \expect{\belief_t \sim P(\belief | \policyLEARN, t)}{ \QFnBel{\policyORBel}{T-t+1}(\belief_t, \policyLEARN(\belief_t)) - \valueFnBel{\policyORBel}{T-t+1}(\belief_t) } \\
	& = \sum\limits_{t=1}^T \expectS{\belief_t \sim P(\belief | \policyLEARN, t)} [ \QFnBel{\policyORBel}{T-t+1}(\belief_t, \policyLEARN(\belief_t)) \\
	& - \max_{\action \in \actionSet} \QFnBel{\policyORBel}{T-t+1}(\belief_t, \action) ]  \\
	& + \sum\limits_{t=1}^T \expect{\belief_t \sim P(\belief | \policyLEARN, t)}{ \max_{\action \in \actionSet} \QFnBel{\policyORBel}{T-t+1}(\belief_t, \action) - \valueFnBel{\policyORBel}{T-t+1}(\belief_t) } \\
	& \geq -T \errcs + \\
	& \sum\limits_{t=1}^T \expect{\belief_t \sim P(\belief | \policyLEARN, t)}{ \max_{\action \in \actionSet} \QFnBel{\policyORBel}{T-t+1}(\belief_t, \action) - \valueFnBel{\policyORBel}{T-t+1}(\belief_t) } \\
	& \geq -T \errcs + T\errhor
\end{aligned}
\end{equation*}

Hence we have the performance bound
\begin{equation*}
\valuePol{\policyLEARN} \geq \valuePol{\policyORBel} -T \errcs + T\errhor
\end{equation*}

Interestingly, note that if $\errhor \geq \errcs$, we would be guaranteed to do \emph{better} than the hallucinating oracle. 

Since we reduce cost sensitive classification to regression by uniformly sampling actions, we can express $\errcs$ in terms of the regression error $\errclass$ using the reduction bound from (\cite{langford2005sensitive}) 

\begin{equation*}
	\errcs \leq 2 \sqrt{\actionSet \;\errclass}
\end{equation*}

Hence we have the performance bound
\begin{equation*}
\valuePol{\policyLEARN} \geq \valuePol{\policyORBel} -2 T \sqrt{\actionSet \; \errclass} + T\errhor
\end{equation*}

\end{proof}

%% file: proof_theorem2.tex

\section{Proof of Theorem 2}
\label{appendix:theorem_aggrevate}

We follow the analysis of \cite{ross2014reinforcement} with two main difference:
\begin{enumerate}
\item \cite{ross2014reinforcement} examine an MDP on states, we translate that to an MDP on history
\item \cite{ross2014reinforcement} consider one step cost minimization, we consider one step reward maximization
\end{enumerate}

We borrow a couple important Lemmas from \cite{ross2014reinforcement}.

\begin{lemma}
\label{lemma:dist_mismatch}
Let $P$ and $Q$ be any two distributions over $x$, let $f(x)$ be a bounded function with range $r$. We then have
\begin{equation*}
	\abs{ \expect{x \sim P}{f(x)} - \expect{x \sim Q}{f(x)} } \leq \frac{r}{2} \norm{P-Q}{1}	
\end{equation*}
\end{lemma}

\begin{lemma}
\label{lemma:mix_divergence}
Let $P(\belief | \policyMix{i})$ be the distribution of history encountered by the mixture policy over all time steps and $P(\belief | \policyLEARN_i)$ be the distribution encountered by the learner. We have
\begin{equation*}
	\norm{ P(\belief | \policyMix{i}) - P(\belief | \policyLEARN_i) }{1} \leq 2 \min (1, T \mixfrac{i})
\end{equation*}
\end{lemma}

We now state the theorem we wish to prove
\begin{theorem*}
$N$ iterations of \aggrevate, collecting $m$ regression examples per iteration guarantees that with probability at least $1-\delta$
\begin{equation*}
\begin{aligned}
	\valuePol{\policyLEARN} \geq & \valuePol{\policyORBel} \\
	& - 2 T \sqrt{\abs{\actionSet} \left( \errclass + \errreg + \bigo{\sqrt{\nicefrac{\log \nicefrac{1}{\delta}}{N m}}} \right) } \\
	& - \bigo{\frac{R \; T \log T}{N}} + T\errhor\\
\end{aligned}
\end{equation*}
where $\errclass$ is the empirical regression regret of the best regressor in the regression class on the aggregated dataset, $\errreg$ is the empirical online learning average regret on the sequence of training examples, $R$ is the range of oracle action value and $\errhor$ is the local oracle suboptimality.
\end{theorem*}

\begin{proof}

We first define the local oracle suboptimality $\errhor$ as in Appendix~\ref{appendix:theorem_ft}
\begin{equation*}
\errhor =  \expect{t\sim U(1:T)}{\min_{\belief_t} \left( \max_{\action \in \actionSet} \QFnBel{\policyORBel}{T-t+1}(\belief_t, \action) - \valueFnBel{\policyORBel}{T-t+1}(\belief_t) \right)}
\end{equation*}

We define the average cost sensitive classification error $\errcs$
\begin{equation*}
\begin{aligned}
& \errcs = \frac{1}{N} \sum_{i=1} ^ N\\
& \expect{\substack{t\sim U(1:T), \\
\belief_t \sim P(\belief | \policyMix{i}, t)}}{\max_{\action \in \actionSet} \QFnBel{\policyORBel}{T-t+1}(\belief_t, \action) -\QFnBel{\policyORBel}{T-t+1}(\belief_t, \policyLEARN(\belief)) }
\end{aligned}
\end{equation*}

Applying the performance difference lemma in Lemma~\ref{lemma:pd} with $\policy = \policyLEARN_i$, $\policy' = \policyORBel$, we have 

\begin{equation*}
\begin{aligned}
	&\valuePol{\policyLEARN_i} - \valuePol{\policyORBel} \\
	& = \sum\limits_{t=1}^T \expect{\belief_t \sim P(\belief | \policyLEARN_i, t)}{ \QFnBel{\policyORBel}{T-t+1}(\belief_t, \policyLEARN_i(\belief_t)) - \valueFnBel{\policyORBel}{T-t+1}(\belief_t) } \\
	& = \sum\limits_{t=1}^T \expectS{\belief_t \sim P(\belief | \policyLEARN_i, t)} [ \QFnBel{\policyORBel}{T-t+1}(\belief_t, \policyLEARN_i(\belief_t))  \\
	& - \max_{\action \in \actionSet} \QFnBel{\policyORBel}{T-t+1}(\belief_t, \action) ] \\
	& + \sum\limits_{t=1}^T \expect{\belief_t \sim P(\belief | \policyLEARN_i, t)}{ \max_{\action \in \actionSet} \QFnBel{\policyORBel}{T-t+1}(\belief_t, \action) - \valueFnBel{\policyORBel}{T-t+1}(\belief_t) } \\
	& \geq \sum\limits_{t=1}^T \expectS{\belief_t \sim P(\belief | \policyLEARN_i, t)} [ \QFnBel{\policyORBel}{T-t+1}(\belief_t, \policyLEARN_i(\belief_t)) \\
	& - \max_{\action \in \actionSet} \QFnBel{\policyORBel}{T-t+1}(\belief_t, \action) ] \\
	& + T\errhor\\
\end{aligned}
\end{equation*}

We define the range $R$ of the maximum difference between the best and worst action value function of the oracle.

\begin{equation*}
R = \max_{t, \belief_t} \abs{ \max_{\action \in \actionSet} \QFnBel{\policyORBel}{T-t+1}(\belief_t, \action)  - \min_{\action \in \actionSet} \QFnBel{\policyORBel}{T-t+1}(\belief_t, \action)   }
\end{equation*}

We can then apply Lemma~\ref{lemma:dist_mismatch} and Lemma~\ref{lemma:mix_divergence} with $P(\belief | \policyLEARN_i, t)$ and $P(\belief | \policyMix{i}, t)$ to get 
\begin{equation*}
\begin{aligned}
	&\valuePol{\policyLEARN_i} - \valuePol{\policyORBel} \\
	& \geq \sum\limits_{t=1}^T \expectS{\belief_t \sim P(\belief | \policyMix{i}, t)}[ \QFnBel{\policyORBel}{T-t+1}(\belief_t, \policyLEARN_i(\belief_t)) \\ 
	& - \max_{\action \in \actionSet} \QFnBel{\policyORBel}{T-t+1}(\belief_t, \action) ] \\
	& - \frac{R}{2} \sum\limits_{t=1}^T \norm{P(\belief | \policyMix{i}, t) - P(\belief | \policyLEARN_i, t)}{1} + T\errhor\\
	& \geq \sum\limits_{t=1}^T \expectS{\belief_t \sim P(\belief | \policyMix{i}, t)}[ \QFnBel{\policyORBel}{T-t+1}(\belief_t, \policyLEARN_i(\belief_t)) \\
	& - \max_{\action \in \actionSet} \QFnBel{\policyORBel}{T-t+1}(\belief_t, \action) ] \\
	& - R \;T \;\min (1, T \mixfrac{i}) + T\errhor\\
\end{aligned}
\end{equation*}

If we now wish to bound the performance of the average learner over $N$ iterations
\begin{equation*}
\begin{aligned}
	&\valuePol{\policyLEARNAvg} - \valuePol{\policyORBel} \\
	& = \frac{1}{N} \sum_{i=1}^N \abs{ \valuePol{\policyLEARN_i} - \valuePol{\policyORBel} } \\
	& \geq - T \errcs  -  \frac{R \; T}{N} \sum_{i=1}^N \min (1, T \mixfrac{i}) + T\errhor\\
	& \geq - T \errcs  -   \frac{R \; T}{N} \frac{ \log(T) + 2}{\alpha}  + T\errhor\\
\end{aligned}
\end{equation*}

where the last inequality follows from \cite{ross2014reinforcement} after setting $\mixfrac{i} = (1 - \alpha)^{i-1}$.

To bound $\errcs$, we need to define two terms: $\errclass$, the empirical regression regret of the best regressor in the regression class on the aggregated dataset, and $\errreg$ the empirical online learning average regret on the sequence of training examples. We then use the following result from \cite{ross2014reinforcement} 
\begin{equation*}
\errcs \leq 2 \sqrt{\abs{\actionSet} \left( \errclass + \errreg + \bigo{\sqrt{\nicefrac{\log \nicefrac{1}{\delta}}{N m}}} \right) }
\end{equation*}
with probability $1-\delta$.

Also note that the performance of the best policy in the sequence $\policyLEARN$ is better than the average learner, i.e. $\valuePol{\policyLEARN} \geq \valuePol{\policyLEARNAvg}$.

This results in the following bound for \aggrevate with probability $1-\delta$
\begin{equation*}
\begin{aligned}
	\valuePol{\policyLEARN} \geq & \valuePol{\policyORBel} \\
	& - 2 T \sqrt{\abs{\actionSet} \left( \errclass + \errreg + \bigo{\sqrt{\nicefrac{\log \nicefrac{1}{\delta}}{N m}}} \right) } \\
	& - \bigo{\frac{R \; T \log T}{N}} + T\errhor\\
\end{aligned}
\end{equation*}

\end{proof}

%% file: proof_theorem3.tex

\section{Proof of Theorem 3}
\label{appendix:theorem_hiddenunc}

\begin{theorem*}
$N$ iterations of \aggrevate with Clairvoyant one-step-reward collecting $m$ regression examples per iteration guarantees that with probability at least $1-\delta$
\begin{equation*}
\begin{aligned}
  \valuePol{\policyLEARN} \geq & \left(1 - \frac{1}{e}\right)\valuePol{\policy^*} \\
  & - 2 T \sqrt{\abs{\actionSet} \left( \errclass + \errreg + \bigo{\sqrt{\nicefrac{\log \nicefrac{1}{\delta}}{N m}}} \right) } \\
  & - \bigo{\frac{R \; T \log T}{N}}\\
\end{aligned}
\end{equation*}
where $\errclass$ is the empirical regression regret of the best regressor in the regression class on the aggregated dataset, $\errreg$ is the empirical online learning average regret on the sequence of training examples, $R$ is the maximum range of one-step-reward.
\end{theorem*}
\begin{proof}
We use an important result from \cite{golovin2011adaptive} about the near-optimality properties of greedy maximization of an adaptive montonone and adaptive submodular set function. We define a greedy policy 

The greedy algorithm selects a node to visit that has the highest expected marginal gain under the conditional distribution of world maps given the history. If the history of vertices visited and measurements received are where $\belief = \{ \vertex_i \}_{i=1}^t, \{ \meas_i \}_{i=1}^t$, the greedy algorithm $\policyGR(\belief_t)$ selects node to visit $\vertex_{t+1}$ with the highest expected marginal gain
\begin{equation}
\label{eq:adaptive_greedy}
\vertex_{t+1} = \argmaxprob{\vertex \in \vertexSet} \expect{\world \sim P(\world | \belief_t)}{\marginalGain{\vertex | \{ \vertex_i \}_{i=1}^t }{\world}}
\end{equation}

\cite{golovin2011adaptive} show that the greedy algorithm $\policyGR$  has the following guarantee
\begin{lemma}
\label{lemma:adapt_sub}
If $\utilityFnDef$ is adaptive monotone and adaptive submodular with respect to $P(\world)$ and $\policyGR$ is a greedy policy, then for all policies $\policy^*$ we have
\begin{equation*}
	\expect{\world \sim P(\world)}{\utilityFn{\policyGR}{\world}} \geq \left( 1 - \frac{1}{e} \right) \expect{\world \sim P(\world)}{\utilityFn{\policy^*}{\world}} 
\end{equation*}
\end{lemma}

We note that for the Clairvoyant one-step-reward oracle is defined such that 
\begin{equation}
\begin{aligned}
&\policyORBel(\belief_t) = \argmaxprob{\action \in \actionSet} \QFnBel{\policyORBel}{T-t+1}(\belief_t, \action) \\
&	= \argmaxprob{\vertex_{t+1} \in \vertexSet} \expect{\world \sim P(\world | \belief)}{  \marginalGain{\vertex_{t+1} | \{ \vertex_i \}_{i=1}^t }{\world} }
\end{aligned}
\end{equation}
where the second inequality uses Definition~\ref{def:clair_onestep_rew} along with Definition~\ref{def:halluc}. Hence $\policyORBel = \policyGR$. Also the local suboptimality is $\errhor = 0$ since the oracle selects actions that maximize one-step-reward. Finally the range $R$ is that of the one step reward. 

Hence applying these terms along with Lemma~\ref{lemma:adapt_sub} in Theorem~\ref{theorem:aggrevate}, we have 
\begin{equation*}
\begin{aligned}
  \valuePol{\policyLEARN} \geq & \left(1 - \frac{1}{e}\right)\valuePol{\policy^*} \\
  & - 2 T \sqrt{\abs{\actionSet} \left( \errclass + \errreg + \bigo{\sqrt{\nicefrac{\log \nicefrac{1}{\delta}}{N m}}} \right) } \\
  & - \bigo{\frac{R \; T \log T}{N}}\\
\end{aligned}
\end{equation*}
\end{proof}

%% file: ml_baseline_search.tex
\section{Machine Learning Baselines for Search Based Planning}
\label{appendix:ml_baseline_search}

\subsection{Supervised Learning (Behavior Cloning)}
The supervised learning algorithm is identical to \algSail with the key difference that roll-outs are made with $\policyOR$ and not $\policyMix{i}$. This is equivalent to setting the mixing parameter $\mixParam = 1$ across all environments. For completeness, we present the algorithm below in Alg.~\ref{alg:supervised_learning}
\begin{algorithm}
\caption{Supervised Learning $(P(\world), P(\start, \goal))$ \label{alg:supervised_learning}}
\begin{algorithmic}[1]
\State Initialize $\dataset \leftarrow \emptyset$ 
\State Collect datapoints as follows:
\For{$i = 1, \ldots, m$}
\State Initialize sub-dataset $\dataset_{i} \leftarrow \emptyset$
\State Sample $\world \sim P(\world)$ 
\State Sample $\pair{\start}{\goal} \sim P(\start)$
\State Invoke clairvoyant oracle planner 
\item[]\hspace*{5mm} to compute $\costToGoOracle\pair{\vertex}{\world} \forall \vertex \in \vertexSet$
\State Rollout a new search with $\policyOracle$
\State At each timestep $t$ pick a random action $\action_t$ 
\item[]\hspace*{5mm} to get corresponding $\pair{\vertex}{\state_t}$
\State Query oracle for $\costToGoOracle\pair{\vertex}{\world}$ 
\State $\dataset_i \gets \dataset_{i} \cup \left< \vertex, \state_t, \costToGoOracle\pair{\vertex}{\world} \right>$
\State Continue roll-out with $\policyOR$ till end of episode.
\EndFor
\State Append to c.s classification data $\dataset \leftarrow \dataset \cup \dataset_{i}$ 
\State Train on $\dataset$ to get $\policyLEARN$
\State \textbf{Return} $\policyLEARN$
\end{algorithmic}
\end{algorithm}
We use $m = 600$ for all the environments. The network architecture and hyper-parameters used are the same as \algSail.

\subsection{Q Learning with Function Approximation}
We use an episodic implementation of the Q-learning algorithm which collects data in an iteration-wise manner similar to \algSail. The learner is trained on the aggregated dataset across all iterations by regressing to the TD-error. The aggregated dataset $\dataset$ effectively acts as an experience replay buffer to which helps in stabilizing learning when using neural network function approximation as has been suggested in recent work~\cite{mnih-dqn-2015}. However,we do not use a target network or any other extensions over the original qlearning algorithm in our baselines~\cite{DBLP:journals/corr/HasseltGS15, DBLP:journals/corr/SchaulQAS15}. We also use only a single observation to take decisions and not a history length of past $h$ observations for a fair comparison with \algSail which also uses a single observation. Alg.~\ref{alg:q_learning} describes the training procedure for the Q-learning baseline.
\begin{algorithm}
\caption{Q-Learning $(P(\world), P(\start, \goal), k)$ \label{alg:q_learning}}
\begin{algorithmic}[1]
\State	Initialize $\dataset \leftarrow \emptyset,\; \policyLEARN_{1}$ to any policy in $\policySet$ 
\For{$i = 1, \ldots, N$}
\State Initialize sub-dataset $\dataset_{i} \leftarrow \emptyset$ 
\State Let mixture policy be 
\item[]\hspace*{5mm} $\policyMix{i} = \epsilon  \text{-greedy on} \;  \policyLEARN_{i} \; \text{with} \; \epsilon_{i} $ 
\State Collect \emph{mk} datapoints as follows:
\For{$j = 1,\ldots,m$} 
\State Sample $\world \sim P(\world)$  
\State Sample $\pair{\start}{\goal} \sim P(\start)$
\State Sample uniformly \emph{k} timesteps $\seqset{t}{k}$ 
\item[]\hspace*{10mm} where each $t_{i} \in \ \set{1, \ldots ,\planTime}$
\State Rollout a new search with $\policyMix{i}$
\State At each $t\in\seqset{t}{k}$, 
\item[]\hspace*{10mm} $\dataset_i \gets \dataset_{i} \cup \left< \vertex, \state_t, C, \vertex_{t+1} \right>$ 
\State Continue roll-out with $\policyMix{i}$ till end of episode.
\EndFor
\State Append to dataset $\dataset \leftarrow \dataset \cup \dataset_{i}$
\State Train learner by minimizing T.D error  on $\dataset$ 
\item[]\hspace*{5mm} to get $\policyLEARN_{i+1}$
\EndFor
\State \textbf{Return} Best $\policyLEARN$ on validation
\end{algorithmic}
\end{algorithm}
$C$ is the one step cost which is 1 for every expansion till goal is added to the open list. We use $k = 100$ and $\epsilon_{0} = 0.9$. Epsilon is decayed after every iteration in an exponential manner. Network architecture and params are kept the same as \algSail.

\subsection{Cross Entropy Method (C.E.M)}
We use C.E.M as a derivative free optimization method for training~\citep{goodfellow2016deep}. At each iteration of the algorithm we sample $\mathrm{batch_{size}} = 40$ set of parameters from a Gaussian Distribution. Each parameter is used to roll-out a policy on 5 environments each and the total cost is collected. The total cost (number of expansions) is used as the fitness function and the 
the best performing, $n_{elite} = 20\%$ of the parameters are selected. These elite parameters are then used to create a new Gaussian distribution (using sample mean and standard deviation) for the next iteration. At the end of all iterations, the best performing policy on a set of held-out states is returned. For this baseline, we use a simpler neural network architecture with one hidden layer of 100 units and ReLu activation.

%% file: representations_search.tex
\section{Representation for Search Based Planning}
\label{appendix:representation_search}

In order to overcome the changing sizes of the observation and action spaces in our setting, we use insight from motion planning literature and represent an entire search state in terms of closest nodes in $\openList$ to a set of pre-defined \emph{attractor states} and \emph{attractor paths}. Attractor states are manually defined states that can be thought of as landmarks trying to pull the search cloud in different directions. Such states can be useful in pulling the search out of local minima such as a bugtrap or they could be strategic orientations of the robot or an object the robot is trying to manipulate that lead to faster solutions~\citep{aine2016multi}. Attractor paths on the other hand are solutions to a small subset of environments from the training dataset. In many episodic tasks, where the structure of the environment does not change drastically between planning iterations, such \emph{path-reuse} can be very useful in finding solutions faster~\citep{phillips2012graphs}. The planning algorithm is built into the environment, and the agent only receives as an observation the nodes in the open list closest to each attractor paths/states. At each iteration then, the action that the agent performs is to select a node from the observation to expand.  

Although this is a generic framework that can be applied to many different problems, we chose not to use it for this work. The reason for this choice was that in this paper, our aim was to build the foundation for learning graph search heuristics as sequential decision making problem and clearly demonstrate the efficacy of the imitation learning paradigm in this domain. We found that using attractor paths/states would distract from the effectiveness of $\algSail$ and also make learning easier for other baseline methods.  

In our final experiments, we instead featurize every pair $\pair{\vertex}{\state}$ using simple information based on the search tree and the environment uncovered up until that point. 

%% file: time_complexity_sail.tex

\section{Analyzing the time complexity of \algSail}
\label{appendix:sail_complexity}

The computational bottleneck in \algSail is the $\select$ function which requires estimating the Q-value for every node in the open list $\openList$. Contrast this with something like Dijkstra's algorithm which selects a node to expand in very little time, but wastes a lot of computation in excessively expanding nodes and evaluating edges. In order to analyze the usefulness of \algSail in terms of computational gains, we make the following assumptions. Firstly, we assume that the computational cost of calculating the Q-value of a single node(including feature calculation and forward pass through function approximator) is equal to the computational cost of $\expand$ function (involves checking all edges coming out of a node for collision and calculating edge costs). This is in reality a very conservative approximation as in many high-dimensional planning problems, collision checking is way more computationally demanding as it requires expensive geometric intersection computations. We also ignore the computational cost of re-ordering the priority-queue whenever a node is popped which means Dijkstra's algorithm can select a node to expand in $O(1)$. Given a graph with cardinality $k$, we obtain the following time complexities for algSail and Dijkstra's algorithm.
	
\subsection*{\algSail test time complexity:} Assume \algSail did $A$ expansions before it found a solution starting from an empty $\openList$. Also assume that states are never removed from $\openList$. 
	Total $\select$ complexity: $O\left(k + 2k + 3k + \ldots + Ak\right) = O\left(kA^{2}\right)$
	Total expansion complexity: $O(\sum_{i=1}^{A} k) = O(Ak)$ 
	From this we get total complexity of \algSail to be $O\left(kA^{2}\right)$. 
\subsection*{Dijkstra's test time complexity:} Assume Dijkstra's algorithm does $B$ expansions before finding a path. As mentioned earlier, we assume that priority-queue reordering can be achieved in constant time. 
	Total $\select$ complexity: $O(1)$
	Total expansion complexity: $O\left(kB\right)$
	From this we obtain total complexity for Dijkstra's algorithm to be $O\left(kB\right)$.
	
	From the above analysis, for \algSail to have lesser overall computational complexity than uninformed search we require the following condition to be satisfied:
	\begin{equation}
	A^{2} < B
	\end{equation}
	Thus, \algSail must obtain a squared reduction in total number of expansions for it to be computationally better than uninformed search. We argue that this strengthens the case for using \algSail in higher dimensional search graphs as in uninformed search expands a very large number of nodes as the total number of graph nodes increases. 

%% file: main.bbl
\begin{thebibliography}{131}
\providecommand{\natexlab}[1]{#1}
\providecommand{\url}[1]{\texttt{#1}}
\expandafter\ifx\csname urlstyle\endcsname\relax
  \providecommand{\doi}[1]{doi: #1}\else
  \providecommand{\doi}{doi: \begingroup \urlstyle{rm}\Url}\fi

\bibitem[Abadi et~al.(2016)Abadi, Agarwal, Barham, Brevdo, Chen, Citro,
  Corrado, Davis, Dean, Devin, Ghemawat, Goodfellow, Harp, Irving, Isard, Jia,
  J{\'{o}}zefowicz, Kaiser, Kudlur, Levenberg, Man{\'{e}}, Monga, Moore,
  Murray, Olah, Schuster, Shlens, Steiner, Sutskever, Talwar, Tucker,
  Vanhoucke, Vasudevan, Vi{\'{e}}gas, Vinyals, Warden, Wattenberg, Wicke, Yu,
  and Zheng]{DBLP:journals/corr/AbadiABBCCCDDDG16}
Mart{\'{\i}}n Abadi, Ashish Agarwal, Paul Barham, Eugene Brevdo, Zhifeng Chen,
  Craig Citro, Gregory~S. Corrado, Andy Davis, Jeffrey Dean, Matthieu Devin,
  Sanjay Ghemawat, Ian~J. Goodfellow, Andrew Harp, Geoffrey Irving, Michael
  Isard, Yangqing Jia, Rafal J{\'{o}}zefowicz, Lukasz Kaiser, Manjunath Kudlur,
  Josh Levenberg, Dan Man{\'{e}}, Rajat Monga, Sherry Moore, Derek~Gordon
  Murray, Chris Olah, Mike Schuster, Jonathon Shlens, Benoit Steiner, Ilya
  Sutskever, Kunal Talwar, Paul~A. Tucker, Vincent Vanhoucke, Vijay Vasudevan,
  Fernanda~B. Vi{\'{e}}gas, Oriol Vinyals, Pete Warden, Martin Wattenberg,
  Martin Wicke, Yuan Yu, and Xiaoqiang Zheng.
\newblock Tensorflow: Large-scale machine learning on heterogeneous distributed
  systems.
\newblock \emph{CoRR}, abs/1603.04467, 2016.
\newblock URL \url{http://arxiv.org/abs/1603.04467}.

\bibitem[Abbeel and Ng(2004)]{abbeel2004apprenticeship}
Pieter Abbeel and Andrew~Y Ng.
\newblock Apprenticeship learning via inverse reinforcement learning.
\newblock In \emph{Proceedings of the twenty-first international conference on
  Machine learning}, page~1. ACM, 2004.

\bibitem[Aine et~al.(2016)Aine, Swaminathan, Narayanan, Hwang, and
  Likhachev]{aine2016multi}
Sandip Aine, Siddharth Swaminathan, Venkatraman Narayanan, Victor Hwang, and
  Maxim Likhachev.
\newblock Multi-heuristic a.
\newblock \emph{The International Journal of Robotics Research}, 2016.

\bibitem[Arfaee et~al.(2011)Arfaee, Zilles, and Holte]{arfaee2011learning}
Shahab~Jabbari Arfaee, Sandra Zilles, and Robert~C Holte.
\newblock Learning heuristic functions for large state spaces.
\newblock \emph{Artificial Intelligence}, 2011.

\bibitem[Arora and Scherer(2017)]{arora2017rapidly}
Sankalp Arora and Sebastian Scherer.
\newblock Randomized algorithm for informative path planning with budget
  constraints.
\newblock In \emph{ICRA}, 2017.

\bibitem[Arulkumaran et~al.(2017)Arulkumaran, Deisenroth, Brundage, and
  Bharath]{arulkumaran2017brief}
Kai Arulkumaran, Marc~Peter Deisenroth, Miles Brundage, and Anil~Anthony
  Bharath.
\newblock A brief survey of deep reinforcement learning.
\newblock \emph{arXiv preprint arXiv:1708.05866}, 2017.

\bibitem[Asmuth and Littman(2011)]{asmuth2011approaching}
John Asmuth and Michael~L Littman.
\newblock Approaching bayes-optimalilty using monte-carlo tree search.
\newblock In \emph{Proc. 21st Int. Conf. Automat. Plan. Sched., Freiburg,
  Germany}, 2011.

\bibitem[Bhardwaj et~al.(2017)Bhardwaj, Choudhury, and
  Scherer]{bhardwaj2017heuristic}
Mohak Bhardwaj, Sanjiban Choudhury, and Sebastian Scherer.
\newblock Learning heuristic search via imitation.
\newblock In \emph{CoRL}, 2017.

\bibitem[Boots et~al.(2011)Boots, Siddiqi, and Gordon]{boots2011closing}
Byron Boots, Sajid~M Siddiqi, and Geoffrey~J Gordon.
\newblock Closing the learning-planning loop with predictive state
  representations.
\newblock \emph{The International Journal of Robotics Research}, 30\penalty0
  (7):\penalty0 954--966, 2011.

\bibitem[Brafman and Tennenholtz(2002)]{brafman2002r}
Ronen~I Brafman and Moshe Tennenholtz.
\newblock R-max-a general polynomial time algorithm for near-optimal
  reinforcement learning.
\newblock \emph{Journal of Machine Learning Research}, 3\penalty0
  (Oct):\penalty0 213--231, 2002.

\bibitem[Canny(1988)]{canny1988complexity}
John Canny.
\newblock \emph{The complexity of robot motion planning}.
\newblock MIT press, 1988.

\bibitem[Chang et~al.(2015)Chang, Krishnamurthy, Agarwal, Daume, and
  Langford]{chang2015learning}
Kai-Wei Chang, Akshay Krishnamurthy, Alekh Agarwal, Hal Daume, and John
  Langford.
\newblock Learning to search better than your teacher.
\newblock In \emph{ICML}, 2015.

\bibitem[Charrow et~al.(2015)Charrow, Kahn, Patil, Liu, Goldberg, Abbeel,
  Michael, and Kumar]{Charrow-RSS-15}
Benjamin Charrow, Gregory Kahn, Sachin Patil, Sikang Liu, Ken Goldberg, Pieter
  Abbeel, Nathan Michael, and Vijay Kumar.
\newblock Information-theoretic planning with trajectory optimization for dense
  3d mapping.
\newblock In \emph{RSS}, 2015.

\bibitem[Chekuri and Pal(2005)]{chekuri2005recursive}
Chandra Chekuri and Martin Pal.
\newblock A recursive greedy algorithm for walks in directed graphs.
\newblock In \emph{FOCS}, 2005.

\bibitem[Chen et~al.(2016{\natexlab{a}})Chen, Frazzoli, Hsu, and
  Lee]{chen2016pomdp}
Min Chen, Emilio Frazzoli, David Hsu, and Wee~Sun Lee.
\newblock Pomdp-lite for robust robot planning under uncertainty.
\newblock \emph{arXiv:1602.04875}, 2016{\natexlab{a}}.

\bibitem[Chen et~al.(2015)Chen, Javdani, Karbasi, Bagnell, Srinivasa, and
  Krause]{AAAI159841}
Yuxin Chen, Shervin Javdani, Amin Karbasi, J.~Bagnell, Siddhartha Srinivasa,
  and Andreas Krause.
\newblock Submodular surrogates for value of information.
\newblock In \emph{AAAI}, 2015.

\bibitem[Chen et~al.(2016{\natexlab{b}})Chen, Hassani, and
  Krause]{DBLP:journals/corr/ChenHK16a}
Yuxin Chen, S.~Hamed Hassani, and Andreas Krause.
\newblock Near-optimal bayesian active learning with correlated and noisy
  tests.
\newblock \emph{CoRR}, abs/1605.07334, 2016{\natexlab{b}}.

\bibitem[Choudhury et~al.(2014)Choudhury, Arora, and
  Scherer]{Choudhury_2014_7588}
Sanjiban Choudhury, Sankalp Arora, and Sebastian Scherer.
\newblock The planner ensemble and trajectory executive: A high performance
  motion planning system with guaranteed safety.
\newblock In \emph{AHS 70th Annual Forum, Montreal, Quebec, Canada}, 2014.

\bibitem[Choudhury et~al.(2017{\natexlab{a}})Choudhury, Javdani, Srinivasa, and
  Scherer]{choudhury2017active}
Sanjiban Choudhury, Shervin Javdani, Siddhartha Srinivasa, and Sebastian
  Scherer.
\newblock Near-optimal edge evaluation in explicit generalized binomial graphs.
\newblock In \emph{NIPS}, 2017{\natexlab{a}}.

\bibitem[Choudhury et~al.(2017{\natexlab{b}})Choudhury, Kapoor, Ranade, and
  Dey]{choudhury2016learning}
Sanjiban Choudhury, Ashish Kapoor, Gireeja Ranade, and Debadeepta Dey.
\newblock Learning to gather information via imitation.
\newblock In \emph{ICRA}, 2017{\natexlab{b}}.

\bibitem[Choudhury et~al.(2017{\natexlab{c}})Choudhury, Kapoor, Ranade, and
  Dey]{choudhury2017adaptive}
Sanjiban Choudhury, Ashish Kapoor, Gireeja Ranade, and Debadeepta Dey.
\newblock Adaptive information gathering via imitation learning.
\newblock In \emph{RSS}, 2017{\natexlab{c}}.

\bibitem[Choudhury et~al.(2016)Choudhury, Dellin, and
  Srinivasa]{choudhury2016pareto}
Shushman Choudhury, Christopher~M Dellin, and Siddhartha~S Srinivasa.
\newblock Pareto-optimal search over configuration space beliefs for anytime
  motion planning.
\newblock In \emph{IROS}, 2016.

\bibitem[Cohen and Carvalho(2005)]{cohen2005stacked}
William~W Cohen and Vitor~R Carvalho.
\newblock Stacked sequential learning.
\newblock In \emph{International Joint Conference on Artificial Intelligence}.
  Citeseer, 2005.

\bibitem[Cover et~al.(2013)Cover, Choudhury, Scherer, and
  Singh]{cover2013sparse}
Hugh Cover, Sanjiban Choudhury, Sebastian Scherer, and Sanjiv Singh.
\newblock Sparse tangential network (spartan): Motion planning for micro aerial
  vehicles.
\newblock In \emph{ICRA}. IEEE, 2013.

\bibitem[Daum{\'e} et~al.(2009)Daum{\'e}, Langford, and Marcu]{daume2009search}
Hal Daum{\'e}, John Langford, and Daniel Marcu.
\newblock Search-based structured prediction.
\newblock \emph{Machine learning}, 75\penalty0 (3):\penalty0 297--325, 2009.

\bibitem[Dellin and Srinivasa(2016)]{dellin2016unifying}
Christopher~M Dellin and Siddhartha~S Srinivasa.
\newblock A unifying formalism for shortest path problems with expensive edge
  evaluations via lazy best-first search over paths with edge selectors.
\newblock In \emph{ICAPS}, 2016.

\bibitem[Dellin et~al.(2016)Dellin, Strabala, Haynes, Stager, and
  Srinivasa]{dellin2016guided}
Christopher~M Dellin, Kyle Strabala, G~Clark Haynes, David Stager, and
  Siddhartha~S Srinivasa.
\newblock Guided manipulation planning at the darpa robotics challenge trials.
\newblock In \emph{Experimental Robotics}, 2016.

\bibitem[Diankov and Kuffner(2007)]{diankov2007randomized}
Rosen Diankov and James Kuffner.
\newblock Randomized statistical path planning.
\newblock In \emph{2007 IEEE/RSJ International Conference on Intelligent Robots
  and Systems}, pages 1--6. IEEE, 2007.

\bibitem[Dolgov et~al.(2008)Dolgov, Thrun, Montemerlo, and
  Diebel]{dolgov2008practical}
Dmitri Dolgov, Sebastian Thrun, Michael Montemerlo, and James Diebel.
\newblock Practical search techniques in path planning for autonomous driving.
\newblock \emph{AAAI}, 2008.

\bibitem[Duan et~al.(2016)Duan, Chen, Houthooft, Schulman, and
  Abbeel]{duan2016benchmarking}
Yan Duan, Xi~Chen, Rein Houthooft, John Schulman, and Pieter Abbeel.
\newblock Benchmarking deep reinforcement learning for continuous control.
\newblock In \emph{International Conference on Machine Learning}, pages
  1329--1338, 2016.

\bibitem[Dubins(1957)]{dubins1957curves}
Lester~E Dubins.
\newblock On curves of minimal length with a constraint on average curvature,
  and with prescribed initial and terminal positions and tangents.
\newblock 1957.

\bibitem[Dugar et~al.(2017{\natexlab{a}})Dugar, Choudhury, and
  Scherer]{dugar2017kappaite}
Vishal Dugar, Sanjiban Choudhury, and Sebastian Scherer.
\newblock A $\kappa$ite in the wind: Smooth trajectory optimization in a moving
  reference frame.
\newblock In \emph{Robotics and Automation (ICRA), 2017 IEEE International
  Conference on}, pages 109--116. IEEE, 2017{\natexlab{a}}.

\bibitem[Dugar et~al.(2017{\natexlab{b}})Dugar, Choudhury, and
  Scherer]{dugar2017smooth}
Vishal Dugar, Sanjiban Choudhury, and Sebastian Scherer.
\newblock Smooth trajectory optimization in wind: First results on a full-scale
  helicopter.
\newblock In \emph{AHS International 73rd Annual Forum, Forth Worth, Texas,
  USA}, volume~1, 2017{\natexlab{b}}.

\bibitem[Finn et~al.(2016)Finn, Levine, and Abbeel]{finn2016guided}
Chelsea Finn, Sergey Levine, and Pieter Abbeel.
\newblock Guided cost learning: Deep inverse optimal control via policy
  optimization.
\newblock In \emph{International Conference on Machine Learning}, pages 49--58,
  2016.

\bibitem[Garrett et~al.()Garrett, Kaelbling, and
  Lozano-P{\'e}rez]{garrettlearning}
Caelan~Reed Garrett, Leslie~Pack Kaelbling, and Tom{\'a}s Lozano-P{\'e}rez.
\newblock Learning to rank for synthesizing planning heuristics.

\bibitem[Golovin and Krause(2011)]{golovin2011adaptive}
Daniel Golovin and Andreas Krause.
\newblock Adaptive submodularity: Theory and applications in active learning
  and stochastic optimization.
\newblock \emph{JAIR}, 2011.

\bibitem[Golovin et~al.(2010)Golovin, Krause, and Ray]{golovin2010near}
Daniel Golovin, Andreas Krause, and Debajyoti Ray.
\newblock Near-optimal bayesian active learning with noisy observations.
\newblock In \emph{NIPS}, 2010.

\bibitem[Goodfellow et~al.(2016)Goodfellow, Bengio, and
  Courville]{goodfellow2016deep}
Ian Goodfellow, Yoshua Bengio, and Aaron Courville.
\newblock \emph{Deep learning}.
\newblock MIT press, 2016.

\bibitem[Gupta et~al.(2010)Gupta, Nagarajan, and Ravi]{gupta2010approximation}
Anupam Gupta, Viswanath Nagarajan, and R~Ravi.
\newblock Approximation algorithms for optimal decision trees and adaptive tsp
  problems.
\newblock In \emph{International Colloquium on Automata, Languages, and
  Programming}, 2010.

\bibitem[Gupta et~al.(2017)Gupta, Davidson, Levine, Sukthankar, and
  Malik]{gupta2017cmp}
Saurabh Gupta, James Davidson, Sergey Levine, Rahul Sukthankar, and Jitendra
  Malik.
\newblock Cognitive mapping and planning for visual navigation.
\newblock In \emph{CVPR}, 2017.

\bibitem[Hansen(2016)]{hansen2016cma}
Nikolaus Hansen.
\newblock The cma evolution strategy: A tutorial.
\newblock \emph{arXiv preprint arXiv:1604.00772}, 2016.

\bibitem[Hausknecht and Stone(2015)]{hausknecht2015deep}
Matthew Hausknecht and Peter Stone.
\newblock Deep recurrent q-learning for partially observable mdps.
\newblock In \emph{2015 AAAI Fall Symposium Series}, 2015.

\bibitem[Heng et~al.(2015)Heng, Gotovos, Krause, and
  Pollefeys]{heng2015efficient}
Lionel Heng, Alkis Gotovos, Andreas Krause, and Marc Pollefeys.
\newblock Efficient visual exploration and coverage with a micro aerial vehicle
  in unknown environments.
\newblock In \emph{ICRA}, 2015.

\bibitem[Hoffmann and Nebel(2001)]{hoffmann2001ff}
J{\"o}rg Hoffmann and Bernhard Nebel.
\newblock The ff planning system: Fast plan generation through heuristic
  search.
\newblock \emph{Journal of Artificial Intelligence Research}, 2001.

\bibitem[Hollinger and Sukhatme(2013)]{hollinger2013sampling}
Geoffrey~A Hollinger and Gaurav~S Sukhatme.
\newblock Sampling-based motion planning for robotic information gathering.
\newblock In \emph{RSS}, 2013.

\bibitem[Hollinger et~al.(2012)Hollinger, Englot, Hover, Mitra, and
  Sukhatme]{hollinger2012active}
Geoffrey~A Hollinger, Brendan Englot, Franz~S Hover, Urbashi Mitra, and
  Gaurav~S Sukhatme.
\newblock Active planning for underwater inspection and the benefit of
  adaptivity.
\newblock \emph{IJRR}, 2012.

\bibitem[Hollinger et~al.(2017)Hollinger, Mitra, and
  Sukhatme]{hollinger2011active}
Geoffrey~A Hollinger, Urbashi Mitra, and Gaurav~S Sukhatme.
\newblock Active classification: Theory and application to underwater
  inspection.
\newblock In \emph{Robotics Research}, pages 95--110. Springer, 2017.

\bibitem[Hsu et~al.(1999)Hsu, Latcombe, and Sorkin]{hsu1999placing}
David Hsu, J-C Latcombe, and Stephen Sorkin.
\newblock Placing a robot manipulator amid obstacles for optimized execution.
\newblock In \emph{Assembly and Task Planning, 1999.(ISATP'99) Proceedings of
  the 1999 IEEE International Symposium on}, pages 280--285. IEEE, 1999.

\bibitem[Ichter et~al.(2017)Ichter, Harrison, and Pavone]{ichter2017learning}
Brian Ichter, James Harrison, and Marco Pavone.
\newblock Learning sampling distributions for robot motion planning.
\newblock \emph{arXiv preprint arXiv:1709.05448}, 2017.

\bibitem[Isler et~al.(2016)Isler, Sabzevari, Delmerico, and
  Scaramuzza]{isler2016information}
Stefan Isler, Reza Sabzevari, Jeffrey Delmerico, and Davide Scaramuzza.
\newblock An information gain formulation for active volumetric 3d
  reconstruction.
\newblock In \emph{ICRA}, 2016.

\bibitem[Iyer and Bilmes(2013)]{iyer2013submodular}
Rishabh~K Iyer and Jeff~A Bilmes.
\newblock Submodular optimization with submodular cover and submodular knapsack
  constraints.
\newblock In \emph{NIPS}, 2013.

\bibitem[Javdani et~al.(2013)Javdani, Klingensmith, Bagnell, Pollard, and
  Srinivasa]{Javdani_2013_7419}
Shervin Javdani, Matthew Klingensmith, J.~Andrew Bagnell, Nancy Pollard, and
  Siddhartha Srinivasa.
\newblock Efficient touch based localization through submodularity.
\newblock In \emph{ICRA}, 2013.

\bibitem[Javdani et~al.(2014)Javdani, Chen, Karbasi, Krause, Bagnell, and
  Srinivasa]{Javdani_2014_7555}
Shervin Javdani, Yuxin Chen, Amin Karbasi, Andreas Krause, J.~Andrew Bagnell,
  and Siddhartha Srinivasa.
\newblock Near optimal bayesian active learning for decision making.
\newblock In \emph{AISTATS}, 2014.

\bibitem[Javdani et~al.(2015)Javdani, Srinivasa, and
  Bagnell]{javdani2015shared}
Shervin Javdani, Siddhartha~S Srinivasa, and J~Andrew Bagnell.
\newblock Shared autonomy via hindsight optimization.
\newblock In \emph{RSS}, 2015.

\bibitem[Jim{\'e}nez et~al.(2012)Jim{\'e}nez, De~La~Rosa, Fern{\'a}ndez,
  Fern{\'a}ndez, and Borrajo]{jimenez2012review}
Sergio Jim{\'e}nez, Tom{\'a}s De~La~Rosa, Susana Fern{\'a}ndez, Fernando
  Fern{\'a}ndez, and Daniel Borrajo.
\newblock A review of machine learning for automated planning.
\newblock \emph{The Knowledge Engineering Review}, 2012.

\bibitem[Kaelbling et~al.(1998)Kaelbling, Littman, and
  Cassandra]{kaelbling1998planning}
Leslie~Pack Kaelbling, Michael~L Littman, and Anthony~R Cassandra.
\newblock Planning and acting in partially observable stochastic domains.
\newblock \emph{Artificial Intelligence}, 1998.

\bibitem[Kahn et~al.(2017)Kahn, Zhang, Levine, and Abbeel]{kahn2016plato}
Gregory Kahn, Tianhao Zhang, Sergey Levine, and Pieter Abbeel.
\newblock Plato: Policy learning using adaptive trajectory optimization.
\newblock In \emph{ICRA}, 2017.

\bibitem[Karaman and Frazzoli(2011)]{karaman2011sampling}
Sertac Karaman and Emilio Frazzoli.
\newblock Sampling-based algorithms for optimal motion planning.
\newblock \emph{The International Journal of Robotics Research}, 30\penalty0
  (7):\penalty0 846--894, 2011.

\bibitem[Karkus et~al.(2017)Karkus, Hsu, and Lee]{karkus2017qmdp}
Peter Karkus, David Hsu, and Wee~Sun Lee.
\newblock Qmdp-net: Deep learning for planning under partial observability.
\newblock \emph{arXiv preprint arXiv:1703.06692}, 2017.

\bibitem[Kearns et~al.(2000)Kearns, Mansour, and Ng]{kearns2000approximate}
Michael~J Kearns, Yishay Mansour, and Andrew~Y Ng.
\newblock Approximate planning in large pomdps via reusable trajectories.
\newblock In \emph{Advances in Neural Information Processing Systems}, pages
  1001--1007, 2000.

\bibitem[Kelly and Nagy(2003)]{KelNag03}
A.~Kelly and B.~Nagy.
\newblock Reactive nonholonomic trajectory generation via parametric optimal
  control.
\newblock \emph{International Journal of Robotics Research}, 22\penalty0
  (7-8):\penalty0 583--601, 2003.

\bibitem[Kober et~al.(2013)Kober, Bagnell, and Peters]{kober2013reinforcement}
Jens Kober, J~Andrew Bagnell, and Jan Peters.
\newblock Reinforcement learning in robotics: A survey.
\newblock \emph{The International Journal of Robotics Research}, 32\penalty0
  (11):\penalty0 1238--1274, 2013.

\bibitem[Koval et~al.(2014)Koval, Pollard, and Srinivasa]{Koval-RSS-14}
Michael Koval, Nancy Pollard, and Siddhartha Srinivasa.
\newblock Pre- and post-contact policy decomposition for planar contact
  manipulation under uncertainty.
\newblock In \emph{RSS}, 2014.

\bibitem[Krause and Golovin(2012)]{krause2012submodular}
Andreas Krause and Daniel Golovin.
\newblock Submodular function maximization.
\newblock \emph{Tractability: Practical Approaches to Hard Problems}, 2012.

\bibitem[Krause and Guestrin(2007)]{Krause:2007:NOS:1619797.1619913}
Andreas Krause and Carlos Guestrin.
\newblock Near-optimal observation selection using submodular functions.
\newblock In \emph{AAAI}, 2007.

\bibitem[Krause et~al.(2008)Krause, Leskovec, Guestrin, VanBriesen, and
  Faloutsos]{krause2008efficient}
Andreas Krause, Jure Leskovec, Carlos Guestrin, Jeanne VanBriesen, and Christos
  Faloutsos.
\newblock Efficient sensor placement optimization for securing large water
  distribution networks.
\newblock \emph{Journal of Water Resources Planning and Management}, 2008.

\bibitem[Kuffner and LaValle(2000)]{kuffner2000rrt}
James~J Kuffner and Steven~M LaValle.
\newblock {RRT-Connect: An efficient approach to single-query path planning}.
\newblock In \emph{ICRA}, 2000.

\bibitem[Kurniawati et~al.(2008)Kurniawati, Hsu, and Lee]{kurniawati2008sarsop}
Hanna Kurniawati, David Hsu, and Wee~Sun Lee.
\newblock Sarsop: Efficient point-based pomdp planning by approximating
  optimally reachable belief spaces.
\newblock In \emph{RSS}, 2008.

\bibitem[Langford and Beygelzimer(2005)]{langford2005sensitive}
John Langford and Alina Beygelzimer.
\newblock Sensitive error correcting output codes.
\newblock In \emph{COLT}, volume 3559, pages 158--172. Springer, 2005.

\bibitem[Laumond et~al.(1998)Laumond, Sekhavat, and
  Lamiraux]{laumond1998guidelines}
Jean-Paul Laumond, S~Sekhavat, and F~Lamiraux.
\newblock Guidelines in nonholonomic motion planning for mobile robots.
\newblock In \emph{Robot motion planning and control}, pages 1--53. Springer,
  1998.

\bibitem[LaValle(2006)]{Lav06}
S.~M. LaValle.
\newblock \emph{Planning Algorithms}.
\newblock Cambridge University Press, Cambridge, U.K., 2006.

\bibitem[LaValle and Kuffner(2001)]{lavalle2001randomized}
Steven~M LaValle and James~J Kuffner.
\newblock Randomized kinodynamic planning.
\newblock \emph{IJRR}, 2001.

\bibitem[Levine and Koltun(2013)]{levine2013guided}
Sergey Levine and Vladlen Koltun.
\newblock Guided policy search.
\newblock In \emph{ICML}, 2013.

\bibitem[Li et~al.(2009)Li, Liao, and Carin]{li2009multi}
Hui Li, Xuejun Liao, and Lawrence Carin.
\newblock Multi-task reinforcement learning in partially observable stochastic
  environments.
\newblock \emph{Journal of Machine Learning Research}, 2009.

\bibitem[Li et~al.(2016)Li, Monroe, Ritter, Galley, Gao, and
  Jurafsky]{li2016deep}
Jiwei Li, Will Monroe, Alan Ritter, Michel Galley, Jianfeng Gao, and Dan
  Jurafsky.
\newblock Deep reinforcement learning for dialogue generation.
\newblock \emph{arXiv preprint arXiv:1606.01541}, 2016.

\bibitem[Liaw and Wiener()]{liaw2002classification}
Andy Liaw and Matthew Wiener.
\newblock Classification and regression by randomforest.

\bibitem[Likhachev and Ferguson(2009)]{likhachev2009planning}
Maxim Likhachev and Dave Ferguson.
\newblock Planning long dynamically feasible maneuvers for autonomous vehicles.
\newblock \emph{The International Journal of Robotics Research}, 2009.

\bibitem[Lim et~al.(2015)Lim, Hsu, and Lee]{NIPS2015_6005}
Zhan~Wei Lim, David Hsu, and Wee~Sun Lee.
\newblock Adaptive stochastic optimization: From sets to paths.
\newblock In \emph{NIPS}. 2015.

\bibitem[Lim et~al.(2016)Lim, Hsu, and Lee]{lim2016adaptive}
Zhan~Wei Lim, David Hsu, and Wee~Sun Lee.
\newblock Adaptive informative path planning in metric spaces.
\newblock \emph{IJRR}, 2016.

\bibitem[Littman and Sutton(2002)]{littman2002predictive}
Michael~L Littman and Richard~S Sutton.
\newblock Predictive representations of state.
\newblock In \emph{Advances in neural information processing systems}, pages
  1555--1561, 2002.

\bibitem[Littman et~al.(1995)Littman, Cassandra, and
  Kaelbling]{Littman95learningpolicies}
Michael~L. Littman, Anthony~R. Cassandra, and Leslie~Pack Kaelbling.
\newblock Learning policies for partially observable environments: Scaling up.
\newblock In \emph{ICML}, 1995.

\bibitem[Liu et~al.(2013)Liu, Liao, and Carin]{liu2013online}
Miao Liu, Xuejun Liao, and Lawrence Carin.
\newblock Online expectation maximization for reinforcement learning in pomdps.
\newblock In \emph{IJCAI}, 2013.

\bibitem[Madani et~al.(2003)Madani, Hanks, and
  Condon]{madani2003undecidability}
Omid Madani, Steve Hanks, and Anne Condon.
\newblock On the undecidability of probabilistic planning and related
  stochastic optimization problems.
\newblock \emph{Artificial Intelligence}, 147\penalty0 (1-2):\penalty0 5--34,
  2003.

\bibitem[McAllester and Singh(1999)]{mcallester1999approximate}
David~A McAllester and Satinder Singh.
\newblock Approximate planning for factored pomdps using belief state
  simplification.
\newblock In \emph{Proceedings of the Fifteenth conference on Uncertainty in
  artificial intelligence}, pages 409--416. Morgan Kaufmann Publishers Inc.,
  1999.

\bibitem[Mnih et~al.(2015)Mnih, Kavukcuoglu, Silver, Rusu, Veness, Bellemare,
  Graves, Riedmiller, Fidjeland, Ostrovski, Petersen, Beattie, Sadik,
  Antonoglou, King, Kumaran, Wierstra, Legg, and Hassabis]{mnih-dqn-2015}
Volodymyr Mnih, Koray Kavukcuoglu, David Silver, Andrei~A. Rusu, Joel Veness,
  Marc~G. Bellemare, Alex Graves, Martin Riedmiller, Andreas~K. Fidjeland,
  Georg Ostrovski, Stig Petersen, Charles Beattie, Amir Sadik, Ioannis
  Antonoglou, Helen King, Dharshan Kumaran, Daan Wierstra, Shane Legg, and
  Demis Hassabis.
\newblock Human-level control through deep reinforcement learning.
\newblock \emph{Nature}, 2015.

\bibitem[Narayanan et~al.(2015)Narayanan, Aine, and
  Likhachev]{narayanan2015improved}
Venkatraman Narayanan, Sandip Aine, and Maxim Likhachev.
\newblock Improved multi-heuristic a* for searching with uncalibrated
  heuristics.
\newblock In \emph{Eighth Annual Symposium on Combinatorial Search}, 2015.

\bibitem[Nelson and Michael(2015)]{nelson2015information}
Erik Nelson and Nathan Michael.
\newblock Information-theoretic occupancy grid compression for high-speed
  information-based exploration.
\newblock In \emph{Intelligent Robots and Systems (IROS), 2015 IEEE/RSJ
  International Conference on}, pages 4976--4982. IEEE, 2015.

\bibitem[Ng and Jordan(2000)]{ng2000pegasus}
Andrew~Y Ng and Michael Jordan.
\newblock Pegasus: A policy search method for large mdps and pomdps.
\newblock In \emph{Proceedings of the Sixteenth conference on Uncertainty in
  artificial intelligence}, pages 406--415. Morgan Kaufmann Publishers Inc.,
  2000.

\bibitem[Paden et~al.(2017)Paden, Varricchio, and
  Frazzoli]{paden2017verification}
Brian Paden, Valerio Varricchio, and Emilio Frazzoli.
\newblock Verification and synthesis of admissible heuristics for kinodynamic
  motion planning.
\newblock \emph{IEEE Robotics and Automation Letters}, 2\penalty0 (2):\penalty0
  648--655, 2017.

\bibitem[Papadimitriou and Tsitsiklis(1987)]{papadimitriou1987complexity}
Christos~H Papadimitriou and John~N Tsitsiklis.
\newblock The complexity of markov decision processes.
\newblock \emph{Mathematics of operations research}, 12\penalty0 (3):\penalty0
  441--450, 1987.

\bibitem[Pearl(1984)]{pearl1984heuristics}
Judea Pearl.
\newblock Heuristics: intelligent search strategies for computer problem
  solving.
\newblock 1984.

\bibitem[Peters and Schaal(2006)]{peters2006policy}
Jan Peters and Stefan Schaal.
\newblock Policy gradient methods for robotics.
\newblock In \emph{IROS}, 2006.

\bibitem[Phillips et~al.(2012)Phillips, Cohen, Chitta, and
  Likhachev]{phillips2012graphs}
Mike Phillips, Benjamin~J Cohen, Sachin Chitta, and Maxim Likhachev.
\newblock E-graphs: Bootstrapping planning with experience graphs.
\newblock 2012.

\bibitem[Phillips et~al.(2015)Phillips, Narayanan, Aine, and
  Likhachev]{phillips2015efficient}
Mike Phillips, Venkatraman Narayanan, Sandip Aine, and Maxim Likhachev.
\newblock Efficient search with an ensemble of heuristics.
\newblock In \emph{IJCAI}, 2015.

\bibitem[Pivtoraiko et~al.(2009)Pivtoraiko, Knepper, and
  Kelly]{pivtoraiko2009differentially}
Mihail Pivtoraiko, Ross~A Knepper, and Alonzo Kelly.
\newblock Differentially constrained mobile robot motion planning in state
  lattices.
\newblock \emph{Journal of Field Robotics}, 26\penalty0 (3):\penalty0 308--333,
  2009.

\bibitem[Pohl(1970)]{pohl1970first}
Ira Pohl.
\newblock First results on the effect of error in heuristic search.
\newblock \emph{Machine Intelligence}, 1970.

\bibitem[Ranzato et~al.(2015)Ranzato, Chopra, Auli, and
  Zaremba]{ranzato2015sequence}
Marc'Aurelio Ranzato, Sumit Chopra, Michael Auli, and Wojciech Zaremba.
\newblock Sequence level training with recurrent neural networks.
\newblock \emph{arXiv preprint arXiv:1511.06732}, 2015.

\bibitem[Ratliff et~al.(2009)Ratliff, Silver, and Bagnell]{ratliff2009learning}
Nathan~D Ratliff, David Silver, and J~Andrew Bagnell.
\newblock Learning to search: Functional gradient techniques for imitation
  learning.
\newblock \emph{Autonomous Robots}, 27\penalty0 (1):\penalty0 25--53, 2009.

\bibitem[Ross and Bagnell(2014)]{ross2014reinforcement}
Stephane Ross and J~Andrew Bagnell.
\newblock Reinforcement and imitation learning via interactive no-regret
  learning.
\newblock \emph{arXiv}, 2014.

\bibitem[Ross et~al.(2008)Ross, Pineau, Paquet, and Chaib-Draa]{ross2008online}
St{\'e}phane Ross, Joelle Pineau, S{\'e}bastien Paquet, and Brahim Chaib-Draa.
\newblock Online planning algorithms for pomdps.
\newblock \emph{JAIR}, 2008.

\bibitem[Ross et~al.(2011)Ross, Gordon, and Bagnell]{ross2011reduction}
St{\'e}phane Ross, Geoffrey~J Gordon, and Drew Bagnell.
\newblock A reduction of imitation learning and structured prediction to
  no-regret online learning.
\newblock In \emph{AISTATS}, volume~1, page~6, 2011.

\bibitem[Schaul et~al.(2015)Schaul, Quan, Antonoglou, and
  Silver]{DBLP:journals/corr/SchaulQAS15}
Tom Schaul, John Quan, Ioannis Antonoglou, and David Silver.
\newblock Prioritized experience replay.
\newblock \emph{CoRR}, abs/1511.05952, 2015.
\newblock URL \url{http://arxiv.org/abs/1511.05952}.

\bibitem[Silver and Veness(2010)]{silver2010monte}
David Silver and Joel Veness.
\newblock Monte-carlo planning in large pomdps.
\newblock In \emph{NIPS}, 2010.

\bibitem[Silver et~al.(2016)Silver, Huang, Maddison, Guez, Sifre, Van
  Den~Driessche, Schrittwieser, Antonoglou, Panneershelvam, Lanctot,
  et~al.]{silver2016mastering}
David Silver, Aja Huang, Chris~J Maddison, Arthur Guez, Laurent Sifre, George
  Van Den~Driessche, Julian Schrittwieser, Ioannis Antonoglou, Veda
  Panneershelvam, Marc Lanctot, et~al.
\newblock Mastering the game of go with deep neural networks and tree search.
\newblock \emph{Nature}, 529\penalty0 (7587):\penalty0 484--489, 2016.

\bibitem[Singh et~al.(2007)Singh, Krause, Guestrin, Kaiser, and
  Batalin]{singh2007efficient}
Amarjeet Singh, Andreas Krause, Carlos Guestrin, William Kaiser, and Maxim
  Batalin.
\newblock Efficient planning of informative paths for multiple robots.
\newblock In \emph{IJCAI}, 2007.

\bibitem[Singh et~al.(2009)Singh, Krause, and
  Kaiser]{Singh:2009:NAI:1661445.1661741}
Amarjeet Singh, Andreas Krause, and William~J. Kaiser.
\newblock Nonmyopic adaptive informative path planning for multiple robots.
\newblock In \emph{IJCAI}, 2009.

\bibitem[Smith and Simmons(2012)]{smith2012point}
Trey Smith and Reid Simmons.
\newblock Point-based pomdp algorithms: Improved analysis and implementation.
\newblock \emph{arXiv:1207.1412}, 2012.

\bibitem[Somani et~al.(2013)Somani, Ye, Hsu, and Lee]{somani2013despot}
Adhiraj Somani, Nan Ye, David Hsu, and Wee~Sun Lee.
\newblock Despot: Online pomdp planning with regularization.
\newblock In \emph{NIPS}, 2013.

\bibitem[Spaan and Vlassis(2005)]{spaan2005perseus}
Matthijs~TJ Spaan and Nikos Vlassis.
\newblock Perseus: Randomized point-based value iteration for pomdps.
\newblock \emph{Journal of artificial intelligence research}, 24:\penalty0
  195--220, 2005.

\bibitem[Sturm et~al.(2012)Sturm, Engelhard, Endres, Burgard, and
  Cremers]{sturm12iros}
J.~Sturm, N.~Engelhard, F.~Endres, W.~Burgard, and D.~Cremers.
\newblock A benchmark for the evaluation of rgb-d slam systems.
\newblock In \emph{IROS}, 2012.

\bibitem[Sun et~al.(2017)Sun, Venkatraman, Gordon, Boots, and
  Bagnell]{sun2017deeply}
Wen Sun, Arun Venkatraman, Geoffrey~J Gordon, Byron Boots, and J~Andrew
  Bagnell.
\newblock Deeply aggrevated: Differentiable imitation learning for sequential
  prediction.
\newblock In \emph{ICML}, 2017.

\bibitem[Sutton and Barto(1998)]{sutton1998reinforcement}
Richard~S Sutton and Andrew~G Barto.
\newblock \emph{Reinforcement learning: An introduction}, volume~1.
\newblock MIT press Cambridge, 1998.

\bibitem[Tamar et~al.(2016)Tamar, Thomas, Zhang, Levine, and
  Abbeel]{tamar2016hindsight}
Aviv Tamar, Garrett Thomas, Tianhao Zhang, Sergey Levine, and Pieter Abbeel.
\newblock Learning from the hindsight plan--episodic mpc improvement.
\newblock \emph{arXiv preprint arXiv:1609.09001}, 2016.

\bibitem[Thayer et~al.(2011)Thayer, Dionne, and Ruml]{thayer2011learning}
Jordan~Tyler Thayer, Austin~J Dionne, and Wheeler Ruml.
\newblock Learning inadmissible heuristics during search.
\newblock In \emph{ICAPS}, 2011.

\bibitem[{Theano Development Team}(2016)]{2016arXiv160502688short}
{Theano Development Team}.
\newblock {Theano: A {Python} framework for fast computation of mathematical
  expressions}.
\newblock \emph{arXiv e-prints}, abs/1605.02688, May 2016.
\newblock URL \url{http://arxiv.org/abs/1605.02688}.

\bibitem[Thrun et~al.(2005)Thrun, Burgard, and Fox]{thrun2005probabilistic}
Sebastian Thrun, Wolfram Burgard, and Dieter Fox.
\newblock \emph{Probabilistic robotics}.
\newblock MIT press, 2005.

\bibitem[Tielman and Hinton(2012)]{rmsprop}
T.~Tielman and G.~Hinton.
\newblock {Lecture 6.5---RmsProp: Divide the gradient by a running average of
  its recent magnitude}.
\newblock COURSERA: Neural Networks for Machine Learning, 2012.

\bibitem[{\'u}s~Virseda et~al.(2013){\'u}s~Virseda, Borrajo, and
  Alc{\'a}zar]{us2013learning}
Jes {\'u}s~Virseda, Daniel Borrajo, and Vidal Alc{\'a}zar.
\newblock Learning heuristic functions for cost-based planning.
\newblock \emph{Planning and Learning}, 2013.

\bibitem[van Hasselt et~al.(2015)van Hasselt, Guez, and
  Silver]{DBLP:journals/corr/HasseltGS15}
Hado van Hasselt, Arthur Guez, and David Silver.
\newblock Deep reinforcement learning with double q-learning.
\newblock \emph{CoRR}, abs/1509.06461, 2015.
\newblock URL \url{http://arxiv.org/abs/1509.06461}.

\bibitem[Venkatraman et~al.(2014)Venkatraman, Boots, Hebert, and
  Bagnell]{venkatraman2014data}
Arun Venkatraman, Byron Boots, Martial Hebert, and J~Andrew Bagnell.
\newblock Data as demonstrator with applications to system identification.
\newblock In \emph{ALR Workshop, NIPS}, 2014.

\bibitem[Wang et~al.(2016)Wang, Schaul, Hessel, Hasselt, Lanctot, and
  Freitas]{wang2016dueling}
Ziyu Wang, Tom Schaul, Matteo Hessel, Hado Hasselt, Marc Lanctot, and Nando
  Freitas.
\newblock Dueling network architectures for deep reinforcement learning.
\newblock In \emph{ICML}, 2016.

\bibitem[Watkins and Dayan(1992)]{watkins1992q}
Christopher~JCH Watkins and Peter Dayan.
\newblock Q-learning.
\newblock \emph{Machine learning}, 8\penalty0 (3-4):\penalty0 279--292, 1992.

\bibitem[Wilt and Ruml(2015)]{wilt2015building}
Christopher~Makoto Wilt and Wheeler Ruml.
\newblock Building a heuristic for greedy search.
\newblock In \emph{Eighth Annual Symposium on Combinatorial Search}, 2015.

\bibitem[Xu et~al.(2007)Xu, Fern, and Yoon]{xu2007discriminative}
Yuehua Xu, Alan Fern, and Sung~Wook Yoon.
\newblock Discriminative learning of beam-search heuristics for planning.
\newblock In \emph{IJCAI}, 2007.

\bibitem[Xu et~al.(2009)Xu, Fern, and Yoon]{xu2009learning}
Yuehua Xu, Alan Fern, and Sungwook Yoon.
\newblock Learning linear ranking functions for beam search with application to
  planning.
\newblock \emph{Journal of Machine Learning Research}, 2009.

\bibitem[Xu et~al.(2010)Xu, Fern, and Yoon]{xu2010iterative}
Yuehua Xu, Alan Fern, and Sung~Wook Yoon.
\newblock Iterative learning of weighted rule sets for greedy search.
\newblock In \emph{ICAPS}, 2010.

\bibitem[Yoon et~al.(2006)Yoon, Fern, and Givan]{yoon2006learning}
Sung~Wook Yoon, Alan Fern, and Robert Givan.
\newblock Learning heuristic functions from relaxed plans.
\newblock In \emph{ICAPS}, 2006.

\bibitem[Yu et~al.(2014)Yu, Schwager, and Rus]{yu2014correlated}
Jingjin Yu, Mac Schwager, and Daniela Rus.
\newblock Correlated orienteering problem and its application to informative
  path planning for persistent monitoring tasks.
\newblock In \emph{IROS}, 2014.

\bibitem[Zhang and Vorobeychik(2016)]{zhang2016submodular}
Haifeng Zhang and Yevgeniy Vorobeychik.
\newblock Submodular optimization with routing constraints, 2016.

\bibitem[Zhang et~al.(2016)Zhang, Kahn, Levine, and Abbeel]{zhang2016mpcgps}
T.~Zhang, G.~Kahn, S.~Levine, and P.~Abbeel.
\newblock Learning deep control policies for autonomous aerial vehicles with
  mpc-guided policy search.
\newblock In \emph{ICRA}, 2016.

\bibitem[Ziebart et~al.(2008)Ziebart, Maas, Bagnell, and
  Dey]{ziebart2008maximum}
Brian~D Ziebart, Andrew~L Maas, J~Andrew Bagnell, and Anind~K Dey.
\newblock Maximum entropy inverse reinforcement learning.
\newblock In \emph{AAAI}, volume~8, pages 1433--1438. Chicago, IL, USA, 2008.

\end{thebibliography}
